\documentclass[11pt]{article}

\usepackage[top=1in, bottom=1in, left=1in, right=1in]{geometry}
\usepackage[numbers]{natbib}
\usepackage{mathtools}
\usepackage{amsmath}
\usepackage{amsthm}
\usepackage{amssymb}
\usepackage{algorithm}
\usepackage{algpseudocode}
\usepackage{dsfont}  
\usepackage{color}
\usepackage{xcolor}
\usepackage{multirow}   
\usepackage{enumerate}  
\usepackage{soul}       
\usepackage{diagbox}    
\usepackage{geometry}  
\usepackage{graphicx}
\usepackage{hyperref}
\hypersetup{
	colorlinks=true,
	linkcolor=blue,
	filecolor=magenta,      
	urlcolor=cyan,
}

\usepackage{fancyhdr}
\usepackage{subcaption}  
\usepackage[justification=centering]{caption}  
\usepackage{url}

\newcommand{\abs}[1]{\left\lvert#1\right\rvert}  
\newcommand{\norm}[1]{\left\lVert#1\right\rVert}  
\newcommand{\angles}[1]{\left\langle#1\right\rangle}  
\newcommand{\ceil}[1]{\left\lceil#1\right\rceil}  

\newcommand{\calA}{\mathcal{A}}
	
\newcommand{\calC}{\mathcal{C}}

\newcommand{\calF}{\mathcal{F}}

\newcommand{\calH}{\mathcal{H}}
\newcommand{\calI}{\mathcal{I}}

\newcommand{\calN}{\mathcal{N}}
 
\newcommand{\calP}{\mathcal{P}}

\newcommand{\calR}{\mathcal{R}}
\newcommand{\calS}{\mathcal{S}}
\newcommand{\calT}{\mathcal{T}}

\newcommand{\calY}{\mathcal{Y}}

\newcommand{\E}{\mathbb{E}}   
\newcommand{\sN}{\mathbb{N}}  
\newcommand{\sR}{\mathbb{R}}  
\newcommand{\sP}{\mathbb{P}}  

\newcommand{\Bernoulli}{\mathrm{Bernoulli}} 

\newcommand{\tr}{\mathrm{tr}}           

\newcommand{\dd}{\mathrm{d}}

\newcommand{\KL}{\mathrm{KL}}
\newcommand{\reg}{\mathrm{reg}}  

\newcommand{\supp}{\mathrm{supp}}

\newcommand{\StochasticBandit}{\mathrm{StochasticBandit}} 
\newcommand{\BernoulliBandit}{\mathrm{BernoulliBandit}}  
\newcommand{\PTS}{\mathrm{PTS}} 

\newcommand{\Block}{\mathrm{Block}}    

\newcommand{\indicator}{\mathds{1}}

\def\XXint#1#2#3{{\setbox0=\hbox{$#1{#2#3}{\int}$ }
		\vcenter{\hbox{$#2#3$ }}\kern-.6\wd0}}

\newcommand{\approptoinn}[2]{\mathrel{\vcenter{
			\offinterlineskip\halign{\hfil$##$\cr
				#1\propto\cr\noalign{\kern2pt}#1\sim\cr\noalign{\kern-2pt}}}}}
\newcommand{\appropto}{\mathpalette\approptoinn\relax}

\newtheorem{proposition}{Proposition}
\newtheorem{corollary}[proposition]{Corollary}
\newtheorem{lemma}[proposition]{Lemma}
\newtheorem{conjecture}[proposition]{Conjecture}

\theoremstyle{definition}
\newtheorem{definition}{Definition}

\theoremstyle{definition}
\newtheorem{assumption}{Assumption}

\theoremstyle{definition}
\newtheorem{example}{Example}

\theoremstyle{remark}

\theoremstyle{remark}
\newtheorem*{remark}{Remark}

\begin{document}
	
\title{Regenerative Particle Thompson Sampling\footnote{Parts of this work have been published as two papers in \emph{Proceedings of the 57th Annual Conference on Information Sciences and Systems (CISS)}, Baltimore, MD, USA, March, 2023, titled \href{https://ieeexplore.ieee.org/document/10089653}{Particle Thompson Sampling with Static Particles} and \href{https://ieeexplore.ieee.org/document/10089647}{Improving Particle Thompson Sampling with Regenerative Particles.}}}
\author{Zeyu Zhou\footnote{Department of Radiology, Mayo Clinic, Rochester, Minnesota, USA. Email: \href{mailto:zeyuzhou91@gmail.com}{zeyuzhou91@gmail.com}}, Bruce Hajek\footnote{Department of Electrical and Computer Engineering and the Coordinated Science Laboratory, University of Illinois at Urbana-Champaign, Email:  \href{mailto:b-hajek@illinois.edu}{b-hajek@illinois.edu}}, Nakjung Choi\footnote{Network System and Security Research, Nokia Bell Labs, Murray Hill, New Jersey, USA. Email: \href{mailto:nakjung.choi@nokia-bell-labs.com}{nakjung.choi@nokia-bell-labs.com}}, Anwar Walid\footnote{Amazon, New York, USA. Email: \href{mailto:acmanwar@acm.org}{acmanwar@acm.org}}} 
\date{\today}
\maketitle

\begin{abstract}
This paper proposes regenerative particle Thompson sampling (RPTS), a flexible variation of Thompson sampling. Thompson sampling itself is a Bayesian heuristic for solving stochastic bandit problems, but it is hard to implement in practice due to the intractability of maintaining a continuous posterior distribution. Particle Thompson sampling (PTS) is an approximation of Thompson sampling obtained by simply replacing the continuous distribution by a discrete distribution supported at a set of weighted static particles. We observe that in PTS, the weights of all but a few fit particles converge to zero. RPTS is based on the heuristic: delete the decaying unfit particles and regenerate new particles in the vicinity of fit surviving particles. Empirical evidence shows uniform improvement from PTS to RPTS and flexibility and efficacy of RPTS across a set of representative bandit problems, including an application to 5G network slicing.
\end{abstract}

\section{Introduction} \label{sec:introduction}

A bandit problem is a sequential decision problem that elegantly captures the fundamental trade-off between the exploitation of actions with high rewards in the past and the exploration of actions that may produce higher rewards in the future. \emph{Thompson sampling (TS)} is a Bayesian heuristic for solving bandit problems with an assumption that the rewards are generated according to a given distribution with a fixed unknown parameter. TS maintains a posterior distribution on the parameter and selects an action according to the posterior probability that the action is optimal. The biggest advantage of TS is its ability to automatically handle setups with a complex information structure, where knowing the performance of one action may inform properties about other actions. Also, it has strong empirical performance \cite{ChapelleLi2011}. Theoretical performance guarantees of TS have been established for some bandit problems \cite{KaufmannKordaMunos2012, AgrawalGoyal2012_BernoulliBandits, AgrawalGoyal2013_ContextualBandits, GopalanMannorMansour2014}. However, efficient updating, storing, and sampling from the posterior distribution in TS are only feasible for some special cases (e.g. conjugate distributions). For general bandit problems, one has to resort to various approximations, most of which are complicated and have restrictive assumptions. 

\emph{Particle Thompson sampling} (PTS) is an approximation of TS based on the following idea: replace the continuous posterior distribution by a discrete distribution supported at a set of weighted static particles. Updating the posterior distribution then becomes updating the particles' weights by Bayes formula, followed by normalization. PTS is flexible: it applies to very general bandit setups. Also, PTS is very easy to implement. However, it may seem on the surface that the crude approximation may bring down the performance of TS significantly, because the set of particles in PTS is finite and static and may not contain the actual parameter. Intuitively, the performance of PTS can be improved by using more particles. However, that comes with an increasing computational cost. 

The main contributions of this paper:

\begin{itemize}
    \item We provide an analysis of PTS for general bandit problems, without assuming that the set of particles contains the hidden system parameter. The main result is a drift-based sample-path necessary condition on the surviving particles, illuminating the phenomenon that fit particles survive and unfit particles decay. 
    \item We propose an algorithm, \emph{regenerative particle Thompson sampling} (RPTS), to improve PTS. The heuristic is: periodically replace the decaying unfit particles in PTS with new generated particles in the vicinity of the survivors. Empirical results show that RPTS algorithms outperform PTS uniformly for a set of representative bandit problems. RPTS is very flexible and easy to implement. 
    \item We show an application of PTS and RPTS to network slicing, a 5G communication network problem, and demonstrate their efficacy through simulation.
\end{itemize}

The remainder of this paper is organized as follows. Section \ref{sec:related_work} lists some related work. Section \ref{sec:setup_and_preliminaries} introduces the general setup and notation of stochastic bandit problems and PTS. Section \ref{sec:analysis_of_PTS} provides a sample-path analysis of PTS. Section \ref{sec:RPTS} introduces RPTS and presents some simulation results. Section \ref{sec:application} shows an application of PTS and RPTS to network slicing. Section  \ref{sec:conclusions} concludes the paper and mentions some potential future work. 

\section{Related Work} \label{sec:related_work}

See \cite{BubeckCesaBianchi2012_BanditMonograph} and \cite{LattimoreSzepesvari_BanditAlgorithms} for a survey and recent developments in bandit problems. 

Upper-confidence-bound (UCB) algorithms \cite{AuerCesaBianchiFischer2002UCB, GarivierOlivier2011KL-UCB} have certain theoretical guarantees for some simple bandit models. KL-UCB \cite{GarivierOlivier2011KL-UCB} even meets a lower bound on regret established in \cite{LaiRobbins1985}. Empirically, UCB algorithms are not very competitive in the non-asymptotic regime due to their inefficient exploration and inability to take advantage of the problem structure for complex bandit problems. 

Reward-biased maximum likelihood estimation (RBMLE) \cite{LiuHsiehBhattacharyaKumar2019MAB, HungHsiehLiuKumar2020LinearBandits} reduces to an indexed policy like UCB and performs well compared to state-of-art algorithms. But for many problems in which the actions give information about the parameter in complicated ways, there is no efficient implementation of RBMLE.  

Thompson sampling (TS) \cite{Thompson1935} has strong empirical performance \cite{ChapelleLi2011} and can handle rather general and complex stochastic bandit problems \cite{GopalanMannorMansour2014, RussoVanRoyKazerouniOsbandWen2017ThompsonSamplingTutorial}. Note that there are certain problems for which TS does not work well \cite{RussoVanRoyKazerouniOsbandWen2017ThompsonSamplingTutorial} and it is still an active area of research to identify such problems and design algorithms to solve them. 

TS can be implemented efficiently in setups where a conjugate prior exists for the reward distribution. In cases where a conjugate prior is not available, one need to resort to approximations of TS, such as Gibbs sampling, Laplace approximation, Langevin Monte Carlo, and boostrapping \cite{RussoVanRoyKazerouniOsbandWen2017ThompsonSamplingTutorial}. These approximations are either complicated, or rely on restrictive assumptions. 

\cite{LuVanRoy2017} proposes ensemble sampling, which is related to the idea of PTS because it aims to maintain a set of particles (called ``models" in the paper) independently and identically sampled from the posterior distribution  in order to approximate TS. Particles in ensemble sampling are unweighted. A major restriction of the algorithm is that it requires Gaussian noise in the observation. Also, except in special setups, updating the particles in ensemble sampling requires solving an optimization problem that accounts for all the data from the start to the current time. 

To the best of our knowledge, the term \emph{particle Thompson sampling} first  appeared in \cite{KawaleBuiKvetonTranThanhChawla2015}, where the authors apply PTS as an efficient approximation of TS to solve a matrix-factorization recommendation problem. Note that in their work, the particles are not static, but are incrementally re-sampled at each step through an MCMC-kernel. The re-sampling method relies heavily on the specific problem structure. It is not clear how it can be generalized to other bandit problems. 

\cite{GopalanMannorMansour2014} analyzes TS for general stochastic bandit problems. The main result is that with high probability the number of plays of non-optimal actions is upper bounded by $B+C \log T$, where $B, C$ are problem-dependent constants and $T$ is the time horizon. For technical tractability, the paper assumes the prior distribution of the parameter is supported over a finite (possibly huge) set instead of a continuum. Therefore, TS in the paper is tantamount to PTS, with the finite prior support set equivalent to a set of particles. The result of the paper relies on a realizability assumption (called ``grain of truth" in the paper): the finite support set of the prior includes the true system parameter. However, for PTS when the true parameter exists in a continuum, the realizability assumption is unreasonable. In fact, without the realizability assumption, PTS may be inconsistent, i.e., the running average regret may not converge to zero. In this paper, PTS is analyzed without the realizability assumption. The analysis is inspired by \cite{GopalanMannorMansour2014}, especially on how KL divergence comes into play in the measurement of the fitness of particles. 

\section{Setup and Preliminaries} \label{sec:setup_and_preliminaries}

\subsection{Stochastic bandit problem} \label{subsec:setup_and_preliminaries|stochastic_bandit_problem}

A \emph{stochastic bandit problem} contains the following elements: an action set $\calA$, an observation space $\calY$, a parameter space $\Theta$, a known observation model $P_\theta(\cdot |a)$ and a reward function $R: \calY \rightarrow \sR$.  Consider a player who acts at steps $t = 1,2,\cdots$. At step $t$, the player takes an action $A_t \in \calA$, then observes $Y_t \in \calY$ according to the observation model $P_{\theta^*}(\cdot|A_t)$ for some fixed and unknown $\theta^* \in \Theta$, independent of past observations. The observation $Y_t$ then incurs a reward $R_t = R(Y_t)$. The goal of the player is to maximize the cumulative reward. For notational convenience, we denote an instance of the stochastic bandit problem by $\StochasticBandit(\calA, \calY, \Theta, P_{\theta}(\cdot |a), R, \theta^*)$. \footnote{The problem can be made more general by adding contexts. Let $\calC$ be a context set. The observation model becomes $P_\theta(\cdot | a, c)$. At each step of the game, the game player receives an arbitrary context $c_t \in \calC$ before taking action $A_t$. The observation $Y_t$ follows distribution $P_{\theta^*}(\cdot|A_t, c_t)$. This is known as the contextual stochastic bandit model, for which PTS still works. The reason we do not use this more general model here is that we want to emphasize the key word \emph{stochastic}, not contextual.} Let $\calH_t = (A_1, \cdots, A_t, Y_1, \cdots, Y_t)$ denote the history of actions and observations up to time $t$. An algorithm is a (possibly randomized) mapping from $\calH_{t-1}$ to $\calA$, for each step $t$. The performance of an algorithm is measured by \emph{regret}. Let $a^* \triangleq \arg \max_{a \in \calA} \E_{\theta^*}[R(Y) | a]$ denote the optimal action that maximizes the mean reward, assuming complete knowledge about $\theta^*$. Let $R^* \triangleq \E_{\theta^*} \left[R(Y) | a^*\right]$ denote the maximum expected reward. The regret of an algorithm that selects $A_t$ at time $t$ is $\reg_t \triangleq R^* - \E_{\theta^*}\left[R(Y) | A_t \right]$, the difference between the expected reward of an optimal action and the action selected by the algorithm. The cumulative regret and running average regret up to time $t$ are $\sum_{\tau=1}^t \reg_\tau$ and $\frac{1}{t} \sum_{\tau=1}^t \reg_\tau$, respectively.

\begin{example}[Bernoulli bandit]
	\label{example:bernoulli_bandit}
	
	Let $K$ be a positive integer. A Bernoulli bandit problem depicts a player who picks an arm indexed by $a \in \{1, \cdots, K \}$ at each step, which generates a reward of either 0 or 1 according to a Bernoulli distribution parameterized by $\theta^*_a \in [0,1]$, fixed and unknown. This is a stochastic bandit problem with $\calA = \left\{1, 2, \cdots, K\right\}$, $\calY = \{0,1\}$, $\Theta = [0,1]^K$, $P_\theta(\cdot|a) \sim \Bernoulli(\theta_a)$, and $R(y) = y$. This is a bandit problem with separable actions -- the observation distribution for each action is parametrized by a corresponding coordinate of $\theta^*.$
\end{example}

\begin{example}[Max-Bernoulli bandit]
	\label{example:max_bernoulli_bandit}
	Let $K, M$ be positive integers with $K \geq 2$ and $M < K$. A max-Bernoulli bandit problem is similar to the Bernoulli bandit, with arms indexed by $\{1,\cdots,K\}$ and each arm is associated with a Bernoulli distribution with a fixed and unknown parameter $\theta^*_a$. The difference is that, in a max-Bernoulli bandit problem, the player picks $M$ different arms at each step instead of one. The reward is the maximum of the $M$ binary values generated by the $M$ selected arms. This problem can be formulated as a stochastic bandit problem with $\Theta = [0,1]^K$, $\calA = \binom{[K]}{M} = \left\{S \subset [K]: |S| = M \right\}$, $\calY = \{0,1\}$. Given $a = (a_1, \cdots, a_M) \in \calA$, observe $Y = \max_{m \in [M]} X_m$, where $X_m \sim \Bernoulli(\theta^*_{a_m})$. That is, the observation model is $P_\theta(\cdot | a) \sim \Bernoulli\left(1 - \prod_{m \in M} (1-\theta_{a_m}) \right)$. The reward function is $R(y) = y$. Actions in the max-Bernoulli bandit problem are not separable. The number of actions, $\binom{K}{M}$, can be much larger than $K$, the dimension of the parameter space. The problem is considered in \cite{GopalanMannorMansour2014}.
\end{example}

\begin{example}[Linear bandit]
	\label{example:linear_bandit}
	A linear bandit problem has two parameters: a positive integer $K$ and $\sigma_W^2 > 0$. It is a stochastic bandit problem with $\Theta = \sR^K$, $\calA = \calS^{K-1} = \{x \in \sR^K: \norm{x}_2 = 1 \}$, the surface of a unit sphere in $\sR^K$, $\calY = \sR$ and $R(y) = y$. Given an action $a \in \calA$, we observe $Y = \angles{\theta^*, a} + W$, where $\theta^* \in \Theta^K$ is fixed and unknown and $W \sim \calN(0, \sigma_W^2)$ is some Gaussian noise. That is, the observation model is $P_\theta(\cdot | a) \sim \calN(\angles{\theta, a}, \sigma_W^2)$. The problem is named ``linear" because the expected reward in each round is an unknown linear function of the action taken. This is an example of a bandit problem in which the dimension of the parameter space is finite, but the number of actions is infinite. 
\end{example}

\subsection{Particle Thompson sampling (PTS)} \label{subsec:setup_and_preliminaries|PTS}

Thompson sampling (TS) is the algorithm for solving stochastic bandit problems, shown in Algorithm \ref{alg:TS_stochastic_bandit}. 

\begin{algorithm}
	\caption{Thompson sampling (TS)}
	\label{alg:TS_stochastic_bandit}
	\textbf{Inputs}: $\calA, \calY, \Theta, P_{\theta}(\cdot | a), R, \theta^*$ \\
	\textbf{Initialization}: prior $\pi_0$ over $\Theta$
	\begin{algorithmic}[1]
		\For{$t = 1,2,\cdots$} 
		\State Sample $\theta_t \sim \pi_{t-1}$
		\State Play $A_t \gets \arg \max_{a \in \calA} \E_{\theta_t} \left[R(Y) | A_t = a\right]$
		\State Observe $Y_t \sim P_{\theta^*}(\cdot | A_t)$
		\State $\text{Update} \; \pi_t$: $\pi_t(\theta) = \frac{P_\theta(Y_t | A_t) \pi_{t-1}(\theta)}{\int_\Theta P_\theta(Y_t | A_t) \pi_{t-1}(\theta) \, \dd \theta} \quad \forall \theta \in \Theta$.
		\EndFor
	\end{algorithmic}
\end{algorithm}

TS is often difficult to implement in practice because $\pi_t$ may not have a closed form. Even if a closed form can be obtained, it is not clear how it can be efficiently stored and be sampled from. The idea of particle Thompson sampling (PTS) (Algorithm \ref{alg:PTS_stochastic_bandit}) is to approximate $\pi_t$ by a discrete distribution $w_t = (w_{t,1}, \cdots, w_{t,N})$ supported on a finite set of fixed particles $\calP_N = \left\{\theta^{(1)}, \cdots, \theta^{(N)} \right\} \subset \Theta$, where $N$ is the number of particles. 

\begin{algorithm}
	\caption{Particle Thompson sampling (PTS)}
	\label{alg:PTS_stochastic_bandit}
	\textbf{Inputs}: $\calA, \calY, \Theta, P_{\theta}(\cdot | a), R, \theta^*, \calP_N$ \\
	\textbf{Initialization}: $w_0 \gets \left(\frac{1}{N}, \cdots, \frac{1}{N}  \right)$ 
	\begin{algorithmic}[1]
		\For{$t = 1,2,\cdots$} 
		\State Generate $\theta_t$ from $\calP_N$ according to weights $w_{t-1}$
		\State Play $A_t \gets \arg \max_{a \in \calA} \E_{\theta_t} \left[R(Y) | A_t = a\right]$
		\State Observe $Y_t \sim P_{\theta^*}(\cdot | A_t)$
		    \For{$i \in \left\{1,2,\cdots, N\right\}$}
		    \State $\widetilde{w}_{t,i} = w_{t-1,i} \; P_{\theta^{(i)}}(Y_t|A_t)$
		    \EndFor
		\State $w_t \gets \text{normalize} \; \widetilde{w}_t$
		\EndFor
	\end{algorithmic}
\end{algorithm}

In practice, one can use a pre-determined set of points $\calP_N$ in $\Theta$, or randomly generate some points from $\Theta$. $\widetilde{w}_{t,i}$ is the unnormalized weight of particle $i$ at time $t$. Step 6 can be alternatively implemented by $\widetilde{w}_{t,i} = \widetilde{w}_{t-1,i} P_{\theta^{(i)}}(Y_t | A_t)$,
with the initialization $\widetilde{w}_0 = w_0$, because it yields the same normalized vectors $w_t$. PTS is very flexible because it does not require any structure on the observation model $P_\theta(\cdot | a)$, as long as the model is given. Steps 5-7 in Algorithm \ref{alg:PTS_stochastic_bandit} are easy to implement: they require only multiplication and normalization. For notational convenience, we denote an instance of particle Thompson sampling with particle set $\calP_N$ by $\PTS(\calP_N)$.

\section{A Sample-Path Analysis of PTS} \label{sec:analysis_of_PTS}

We provide an analysis of PTS in this section. The main result is a sample-path necessary condition for surviving particles based on drift information. 

Notation: Let $I_t \in [N]$ be the index of the particle chosen at time $t$. Thus, $I_t \sim w_{t-1}$. Let $A_t \in \calA$ be the arm chosen at time $t$. Let $A: \Theta \rightarrow \calA$ be the function mapping from a particle to the corresponding optimal arm, defined by $A(\theta) = \arg\max_{a \in \calA} \E_\theta[R(Y)|a]$. If there are multiple maximizers, let $A(\theta)$ be one of them selected deterministically. With a slight abuse of notation, we sometimes abbreviate $A(\theta^{(i)})$ by $A(i)$. So $A_t = A(I_t)$. For any $x \in \sR^N$, define $\supp(x) \triangleq \left\{i \in [N]: x_i \neq 0\right\}$ and $\arg \max x \triangleq \left\{i \in [N]: x_i = \max_{j \in [N]} x_j \right\}$. 

Recall from Algorithm \ref{alg:PTS_stochastic_bandit} that the unnormalized weights of the particles evolve by the equation $\widetilde{w}_{t,i} = \widetilde{w}_{t-1,i} P_{\theta^{(i)}}(Y_t | A_t)$, where $Y_t \sim P_{\theta^*}(\cdot | A_t)$. 

\begin{definition}(Drift matrix)
	\label{definition:general_stochastic_bandit|drift_matrix|drift_matrix}
	For a given $\StochasticBandit(\calA, \Theta, \calY, P_\theta(\cdot|a), R, \theta^*)$ problem and a set of particles $\calP_N \subset \Theta$, the \emph{drift matrix} $D$ is a $N \times N$ matrix, where 
	\begin{equation*}
	\begin{aligned}
		D_{ij} &\triangleq \E\left[\ln \widetilde{w}_{t,j} - \ln \widetilde{w}_{t-1,j} | I_t = i\right] = \E[\ln P_{\theta^{(j)}}(Y_t | A_t)|I_t = i] = \E_{Y \sim P_{\theta^*}(\cdot | A(i)) }\left[\ln P_{\theta^{(j)}}(Y | A(i)) \right] \, ,
	\end{aligned}
	\end{equation*}
	for $i, j \in [N]$. In words, $D_{ij}$ is the (exponential) drift of particle $j$ when particle $i$ is chosen. 
\end{definition}

The following properties of $D$ are readily verified: 1) Entries in $D$ are non-positive; 2) $D$ is independent of time, fundamentally because $\{\widetilde{w}_t\}$ is a time-homogeneous Markov process;  3) Row $i_1$ and row $i_2$ of $D$ are the same if $A(i_1) = A(i_2)$. Therefore $D$ can have at most $\abs{\calA}$ distinct rows. In what follows we consider drift matrices $D$ and $D'$ to be equivalent if each row in $D'$ is equal to the corresponding row of $D$ up to an additive constant.  Therefore, $D$ remains in the same equivalence class if for each $i$ the constant $-\E \left[\ln P_{\theta^*}(Y | A(i))\right]$ is added to row $i.$  Therefore, a representative choice of $D$ is the following:
\begin{equation*}
    \begin{aligned}
        D_{ij} &\overset{\text{equivalent}}{=} -\E_{Y \sim P_{\theta^*}(\cdot | A(i))} \left[\ln \frac{P_{\theta^*}(Y | A(i))}{P_{\theta^{(j)}}(Y | A(i))} \right] = -\KL\left(P_{\theta^*}(\cdot | A(i)) \; \big|\big| \; P_{\theta^{(j)}} (\cdot | A(i)) \right) .
    \end{aligned}
\end{equation*}

Here $D_{ij}$ is the negative of KL divergence between distributions $P_{\theta^*}(\cdot | A(i))$ and $P_{\theta^{(j)}}(\cdot | A(i)).$  In this sense, the $i$th row of $D$ gives the relative fitness of the particles for action $A(i)$, and the $j^{th}$ column of $D$ gives the fitness of particle $j$ for action $A(i)$ varying over all $i.$

We need the following two assumptions before the main result.

\begin{assumption}[Sample path assumptions]
	\label{assumption:PTS_for_general_stochastic_bandit|sample_path_asymptotic_particle_behavior|sample_path_assumptions}

	Consider the problem $\StochasticBandit(\calA, \Theta, \calY, P_\theta(\cdot|a), R, \theta^*)$ and suppose
	$\PTS(\calP_N)$ is run for a set of $N$ particles $\calP_N \subset \Theta$. Assume that the sample path satisfies the following: there exists a non-empty set $S \subset [N]$ that satisfies\footnote{There are two additional technical assumptions on sample-path, which are put in appendix Section \ref{appendix:analysis_of_PTS} to save space.}
	\begin{itemize}
		\item[(a)] (Non-zero decaying rate gap) For any $i \not\in S$ and $j \in S$,  $\limsup_{t \rightarrow \infty} \frac{1}{t} \left( \ln \widetilde{w}_{t,i} -  \ln \widetilde{w}_{t,j} \right) < 0$, and
		\item[(b)] (Existence of survivor limiting distribution) $G_t = \left(\ln \widetilde{w}_{t,i} - \ln \widetilde{w}_{t,j}: i, j \in S\right) \in \sR^{|S| \times |S|}$ has a limiting empirical distribution $\mu_{G}$. In other words, for any bounded continuous function $h$ on $\sR^{|S| \times |S|}$, $\frac{1}{t} \sum_{\tau=0}^t h(G_\tau) \rightarrow \E_{\mu_G}[h]$.
	\end{itemize}
\end{assumption}

The set $S$ can be thought of as the set of surviving particles. Assumption \ref{assumption:PTS_for_general_stochastic_bandit|sample_path_asymptotic_particle_behavior|sample_path_assumptions}(a) says the (unnormalized) weight decaying rate of a non-surviving particle is strictly less than that of a surviving particle. Consequently, the weight of a non-surviving particle converges to 0 exponentially fast. Assumption \ref{assumption:PTS_for_general_stochastic_bandit|sample_path_asymptotic_particle_behavior|sample_path_assumptions}(b) says that the process $G_t$ has some ergodicity property. It is similar to saying that $G_t$ is Harris recurrent, except $G_t$ is not Markov, because it excludes information about particles not in $S$. Note that knowing any row of $G_t$ determines all the other entries of $G_t$.  
 
\begin{assumption}[Boundedness of observation model]
	\label{assumption:PTS_for_general_stochastic_bandit|sample_path_asymptotic_particle_behavior|boundedness_of_observation_likelihood}
	Assume that the observation model $P_{\theta}(\cdot | a)$ satisfies: there exists constants $b_0, B_0 > 0$, such that for any $\theta, \theta' \in \Theta$, $b_0 \leq \frac{P_\theta(y|a)}{P_{\theta'}(y|a)}  \leq B_0$ for any $y \in \calY, a \in \calA$. 
\end{assumption}

The assumption can be easily verified for problems in which $\abs{\calY} < \infty$ and $\abs{\calA} < \infty$, for example, the Bernoulli bandit and max-Bernoulli bandit problems. 

Define a probability vector $\pi$ over $[N]$ by $\pi_i  = \lim_{t \rightarrow \infty}$ $\frac{1}{t+1} \sum_{\tau=0}^t w_{\tau, i}$. That is, $\pi_i$ is the limiting running average weight of particle $i$, if it exists. The following proposition shows the relationship between $\pi$ and the drift matrix $D$ and provides a necessary condition for surviving particles in a sample path.  

\begin{proposition}[Sample-path necessary surviving condition]
	\label{proposition:sample_path_necessary_survival_condition}
	Let $\StochasticBandit(\calA, \Theta, \calY, P_\theta(\cdot|a), R, \theta^*)$ be a given problem and $\calP_N \subset \Theta$ a given set of $N$ particles. Suppose $P_{\theta}(\cdot | a)$ satisfies Assumption \ref{assumption:PTS_for_general_stochastic_bandit|sample_path_asymptotic_particle_behavior|boundedness_of_observation_likelihood}. Consider running $\PTS(\calP_N)$ for the problem. 
	Let $D$ be the drift matrix. For a sample path of the algorithm under Assumption \ref{assumption:PTS_for_general_stochastic_bandit|sample_path_asymptotic_particle_behavior|sample_path_assumptions}, $\pi$ is well defined and satisfies 
	\begin{equation}
		\label{eq:PTS_for_general_stochastic_bandit|sample_path_asymptotic_particle_behavior|necessary_survival_condition}
		\arg \max (\pi D) = \supp(\pi) = S \, ,
	\end{equation}
	where $S$ is the set in Assumption \ref{assumption:PTS_for_general_stochastic_bandit|sample_path_asymptotic_particle_behavior|sample_path_assumptions}. 
\end{proposition}

The proposition says that, if a set of particles $S$ were to survive in a sample path, they must have a limiting average selection distribution $\pi$ that satisfies (\ref{eq:PTS_for_general_stochastic_bandit|sample_path_asymptotic_particle_behavior|necessary_survival_condition}).  The $j$th coordinate of $\pi D$, $(\pi D)_j$, is equal to $\angles{\pi, D_{\cdot j}}$, where $D_{\cdot j} = (D_{1j}, \cdots, D_{Nj})$ is the $j$th column of $D$, the drifts of particle $j$ when particles $1, 2, \cdots, N$ are chosen, which we recall can be interpreted as the fitness of particle $j$. Thus, $(\pi D)_j$ is the average fitness of particle $j,$ assuming distribution $\pi$ is used to select a random action $A(i).$ Therefore,  (\ref{eq:PTS_for_general_stochastic_bandit|sample_path_asymptotic_particle_behavior|necessary_survival_condition}) means that, with respect to distribution $\pi$, each surviving particle has the same average fitness, and the average fitness of each non-surviving particle is strictly smaller. This aligns with our observation in experiments: \emph{fit particles survive, unfit particles decay}. Note the following caveat: Proposition \ref{proposition:sample_path_necessary_survival_condition} provides a sample-path condition for surviving particles. The actual set of survivors may be random. Thus, there may be more than one $\pi$ that satisfies  \eqref{eq:PTS_for_general_stochastic_bandit|sample_path_asymptotic_particle_behavior|necessary_survival_condition}. 

Applying Proposition \ref{proposition:sample_path_necessary_survival_condition} to Bernoulli bandit with randomly generated particles in PTS, yields the following corollary that says that not many particles can survive. 

\begin{corollary}
	\label{corollary:PTS_for_K_arm_Bern_bandit|number_of_surviving_particles_is_no_more_than_K}
	Let $\calP_N$ be a set of $N$ points generated independently and uniformly at random from $[0,1]^K$. Consider running $\PTS(\calP_N)$ for a given Bernoulli bandit problem with $K$ arms and with $\theta^* \in [0,1]^K$. Suppose that any sample path satisfies Assumption \ref{assumption:PTS_for_general_stochastic_bandit|sample_path_asymptotic_particle_behavior|sample_path_assumptions}. Then with probability one, at most $K$ particles can survive, i.e. $\abs{\supp(\pi)} \leq K$. 
\end{corollary}

We suspect that something similar can be said about the fewness of survivors for other bandit problems in which the action space has a finite dimension $K$ (the number of actions may be much larger). But we don't have a proof.

Proofs of Proposition \ref{proposition:sample_path_necessary_survival_condition} and Corollary \ref{corollary:PTS_for_K_arm_Bern_bandit|number_of_surviving_particles_is_no_more_than_K} can be found in Appendix Section \ref{appendix:analysis_of_PTS}. For more evidence and intuition of the assumptions and conclusions of Proposition \ref{proposition:sample_path_necessary_survival_condition} and Corollary \ref{corollary:PTS_for_K_arm_Bern_bandit|number_of_surviving_particles_is_no_more_than_K}, see Appendix Section \ref{appendix:PTS_for_two_arm_Bernoulli_bandit}, where a thorough analysis of PTS for two-arm Bernoulli bandit is provided. 

\section{RPTS: Regenerative Particle Thompson Sampling} \label{sec:RPTS}

This section  proposes \emph{regenerative particle Thompson sampling} (RPTS) and demonstrates its performance by simulation. 
Recall that, in PTS, fit particles survive, unfit particles decay, and most particles eventually decay. When the weights of the decaying particles become so small that they become essentially inactive, continuing using these particles would be a waste of computational resource. A natural thing to do is to delete those decaying particles and use the saved computational resource to improve the algorithm. RPTS (Algorithm \ref{alg:RPTS}) is based on the following heuristic inspired by biological evolution: \emph{delete unfit decaying particles, regenerate new particles in the vicinity of the fit surviving particles.}

\begin{algorithm}[h]
	\caption{Regenerative particle Thompson sampling (RPTS)}
	\label{alg:RPTS}
	\textbf{Input}:  $\calA, \calY, \Theta \subset \sR^K, P_{\theta}(\cdot | a), R, \theta^*, \calP_N$ \\
	\textbf{Parameters}: $N$, $f_{del} \in (0,1)$, $w_{inact} \in (0,1)$, $w_{new} \in (0,1)$ \\
	\textbf{Initialization}: $w_0 \gets \left(\frac{1}{N}, \cdots, \frac{1}{N}\right)$
	\begin{algorithmic}[1]
		\For{$t = 1,2,\cdots$} 
		\State Generate $\theta_t$ from $\calP_N$ according to weights $w_{t-1}$
		\State Play $A_t \gets \arg \max_{a \in \calA} \E_{\theta_t} \left[R(Y) | A_t = a\right]$
		\State Observe $Y_t \sim P_{\theta^*}(\cdot | A_t)$
		\For{$i \in \left\{1,2,\cdots, N\right\}$}
		\State $\widetilde{w}_{t,i} = w_{t-1,i} \; P_{\theta^{(i)}}(Y_t|A_t)$ 
		\EndFor
		\State $w_t \gets \text{normalize} \; \widetilde{w}_t$
		\If {CONDITION($w_t, N, f_{del}, w_{inact}$) = True}
		\State $\calI_{del} \gets$ the indices of the lowest weighted $\ceil{f_{del}N}$ particles in $\calP_N$ 
		\State $\{\theta^{(i)}: i \in \calI_{del}\} \overset{\text{replace}}{\gets}$ RPTS-Exploration
		\State $w_{t,i} \gets \frac{w_{new}}{\ceil{f_{del}N}}$ for each $i \in \calI_{del}$
		\State normalize $w_t$
		\EndIf
		\EndFor
	\end{algorithmic}
\end{algorithm}

\begin{algorithm}[h]
	\begin{algorithmic}[0]
		\State CONDITION($w_t, N, f_{del}, w_{inact}$):
		\State $w'_t \gets$ sort $w_t$ in ascending order
		\State If $\sum_{i=1}^{\ceil{f_{del}N}} w'_{t,i} \leq w_{inact}$: Return True
		\State Else: Return False
	\end{algorithmic}
\end{algorithm}

\begin{algorithm}[h!]
	\begin{algorithmic}[0]
		\State RPTS-Exploration:
		\State $\mu_t \gets \E_{\theta \sim w_t}[\theta]$, $\Sigma_t \gets \E_{\theta \sim w_t}[(\theta-\mu_t)(\theta-\mu_t)^T]$
		\State Generate $\ceil{f_{del}N}$ particles $\overset{i.i.d.}{\sim} \calN(\mu_t, \frac{1}{K} \tr(\Sigma_t)I_K)$, project to $\Theta$
	\end{algorithmic}
\end{algorithm}

Steps 1-8 of RPTS are the same as PTS (Algorithm \ref{alg:PTS_stochastic_bandit}). The difference is that RPTS adds steps 9-14. Three new hyper-parameters are introduced: $f_{del}$, the fraction of particles to delete; $w_{inact}$, the weight threshold for deciding inactive particles; $w_{new}$, the new (aggregate) weight of regenerated particles. The CONDITION in Step 9 checks if $f_{del}$ fraction of the particles become inactive. If so, we find the lowest weighted $f_{del}$ fraction of the particles (Step 10), delete them, and regenerate the same number of particles through RPTS-Exploration (Step 11). In RPTS-Exploration, we first calculate the empirical mean $\mu_t$ and covariance matrix $\Sigma_t$ of all the particles based on their current weights $w_t$\footnote{According to the RPTS heuristic, one may expect to calculate $\mu_t$ and $\Sigma_t$ based on the weights of the surviving particles only, instead of all the particles. But because the surviving particles have a total weight of at least $1-w_{inact}$, close to 1, the difference is negligible.}, i.e. $\mu_t = \sum_{i=1}^N w_{t,i} \theta^{(i)}$ and $\Sigma_t = \sum_{i=1}^N w_{t,i} \left(\theta^{(i)}-\mu_t\right)\left(\theta^{(i)}-\mu_t\right)^T$, then generate the new particles according to a multi-variate Gaussian distribution. $I_K$ is the $K \times K$ identity matrix. We use $\frac{1}{K}\tr(\Sigma_t)I_K$ as the covariance matrix instead of $\Sigma_t$, in case $\Sigma_t$ is or close to singular. This particle regeneration strategy requires that the parameter space $\Theta$ is a subset of $\sR^K$. If a newly generated particle is outside of $\Theta$, we project it to $\Theta$ in any natural way.\footnote{Alternatively, we can reject it and regenerate until it is in $\Theta$.} Step 12 means that the newly generated $\ceil{f_{del}N}$ particles are assigned a total weight of $w_{new}$ and each of them has the same weight. 

Typical values of the three hyperparameters are $f_{del} = 0.8$, $w_{inact} = 0.001$ and $w_{new} = 0.01$. Section \ref{appendix:RPTS} in appendix elaborates on the choice of these values.

We run simulations\footnote{Code is available if the paper is accepted.}
to compare RPTS with PTS and TS. Selected results are shown in Figure \ref{fig:RPTS_simulation}. For the Bernoulli bandit problem, TS is implemented as a bench mark. For max-Bernoulli bandit, it is not clear how TS can be implemented. Each curve is obtained by averaging over 200 independent simulations. In each simulation of PTS or RPTS, the initial particles are generated uniformly at random from $[0,1]^K$.

\begin{figure}[h]
\centering
\subcaptionbox{Bernoulli bandit with $K=10$ \\ $\theta^* = [0.51, 0.52, \cdots, 0.60].$}{\includegraphics[width=0.48\columnwidth]{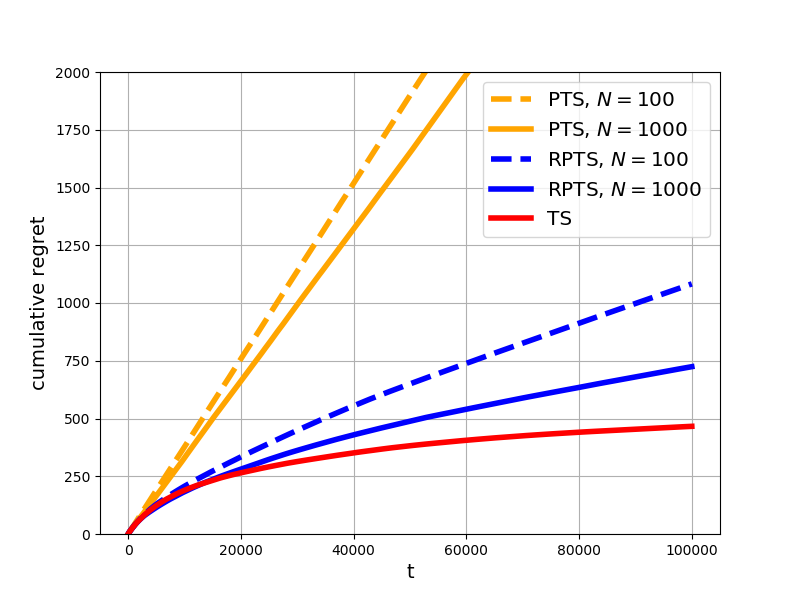}}%
\subcaptionbox{Bernoulli bandit with $K=100$ \\ $\theta^*$ consists of $N=100$ points uniformly spaced over [0.5,0.7].}{\includegraphics[width=0.48\columnwidth]{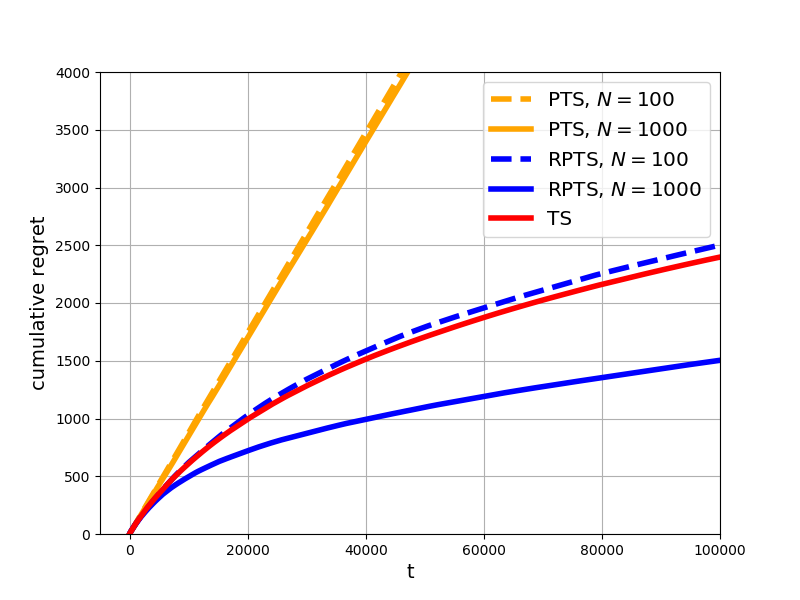}}%

\subcaptionbox{Max-Bernoullin bandit with $K=100$, $M=5$ \\ $\theta^*$ consists of $N=100$ points uniformly spaced over [0.3,0.8].}{\includegraphics[width=0.48\columnwidth]{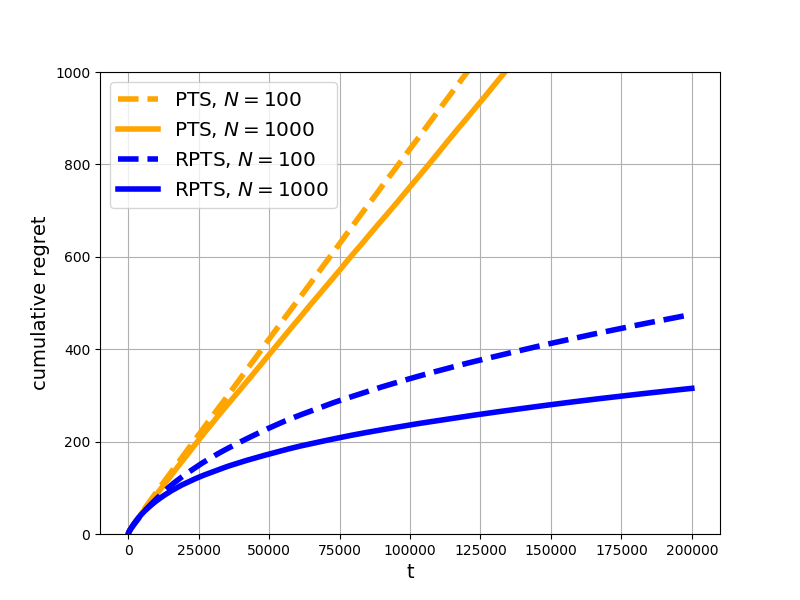}}%
\caption{Simulations}
\label{fig:RPTS_simulation}
\end{figure}

\section{Application to Network Slicing} \label{sec:application}

In this section, we describe an application of PTS and RPTS to 5G network slicing. Network slicing is the partition of a network infrastructure into logically independent networks across multiple technology domains, in order to support independent vertical services with heterogeneous requirements. A network slice is an end-to-end virtual network, formed by stitching resources across different domains. Although network slicing is a promising technology, there remain many challenges both on the system level and theory level, see \cite{NencioniGarroppoGonzalezHelvikProcissi_5G} for a detailed account. One main challenge is the complexity in the coordination and integration of resources at different domains, which necessitates a centralized control for resource allocation and cross-dodmain coordination for stitching the slice.
We propose a high-level model that captures the main features and challenges of the network slicing process and solve it using PTS and RPTS.

\subsection{Model}

On a high level, a mobile operator creates network slices across domains on-demand, which are then put into use and exhibits certain performance. The system observes each domain behaviors, e.g., latency, to make better decisions in the future. We formulate the problem as a contextual stochastic bandit problem by specifying the following elements: $(\calC, \calA, \calY, \Theta, \theta^*, P_\theta(\cdot | a, c), R)$. See Figure \ref{fig:application|models|general_network_slicing_model}. 

\begin{figure}[h]
\begin{center}
\centerline{\includegraphics[width=0.7\columnwidth]{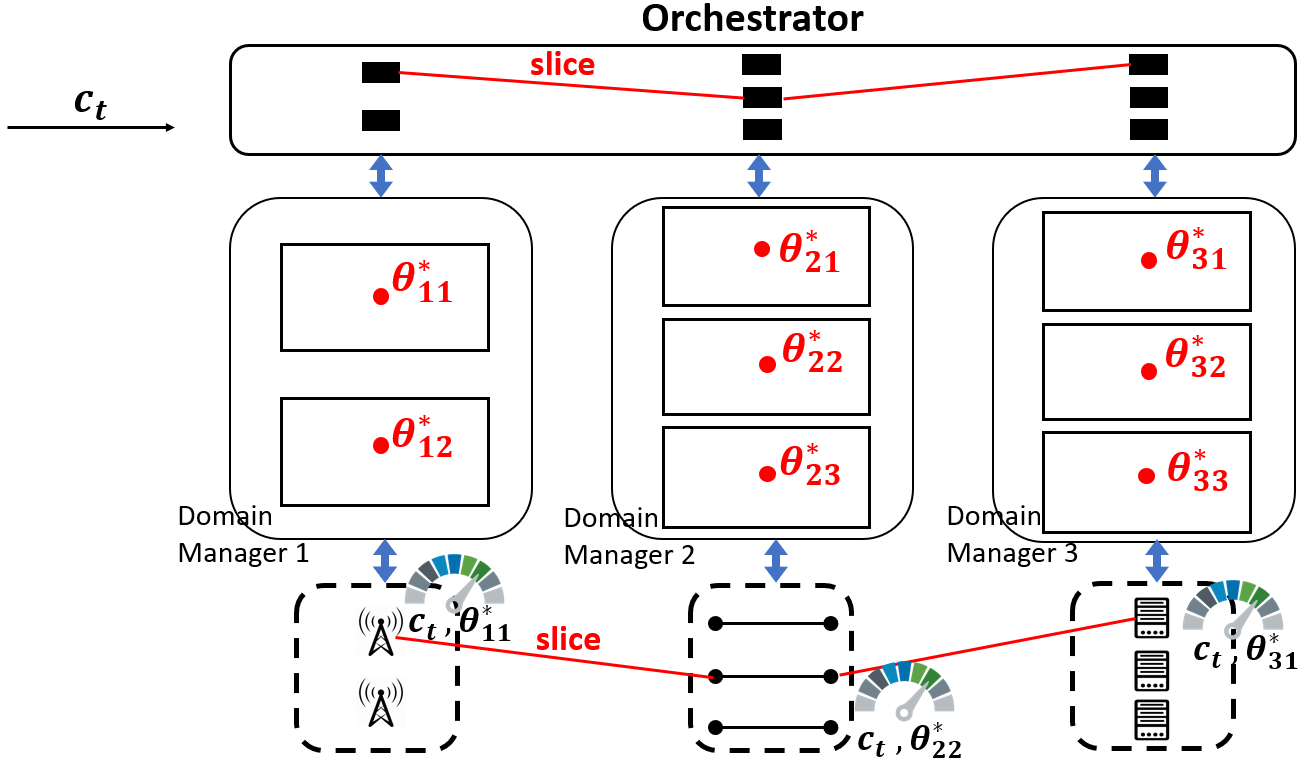}}
\caption{A network slicing model.}
\label{fig:application|models|general_network_slicing_model}
\end{center}
\vskip -0.3in
\end{figure}

\textbf{Context set $\calC$}. Let $\calC = [0,1]^2$. A context vector $c = (c_1, c_2)$ represents a slice request, characterizing the load and latency requirements for the intended service. Specifically, $c_1 \in [0,1]$ is the scaled offered load, relative to some maximum load that the mobile operator can support. For example, if the maximum supportable load is $20$Gbps and $c_1 = 0.5$, then the requested load is $20 \cdot c_1 = 10$Gbps. Let $c_2 \in [0,1]$ be the inverse end-to-end latency requirement, scaled by the minimum possible. For example, if the minimum latency the network can support is 1ms and $c_2 = 0.5$, then the latency required by the service provider is $\frac{1}{c_2} = 2$ms.  

\textbf{Action space $\calA$}. Let $\calA = [B_1] \times \cdots \times [B_D]$, where $D$ is the number of domains, $B_i$ is the number of resource blocks in domain $i$, and $[n]$ is short for $\{1,2,\cdots,n\}$. That is, an action $a = (a_1, \cdots, a_D)$ is a stitched chain of resource blocks, one from each domain, that form an end-to-end network slice. The resource blocks model the resources available in each domain, regardless of their specific types. Block $j$ in domain $i$ is denoted as $\Block_{ij}$. At time $t$, the mobile operator selects an action $A_t \in \calA$ through the central orchestrator. In Figure \ref{fig:application|models|general_network_slicing_model}, $D = 3$, $(B_1, B_2, B_3) = (2,3,3)$, and the action selected is $(1,2,1)$. In practice, $D$ and $B_i$'s are not large.  

\textbf{Parameter space $\Theta$ and parameter $\theta^*$.} The parameter space is $\Theta = \Theta_1 \times \cdots \times \Theta_D$, where $\Theta_i = \underbrace{[0,1]^2 \times \cdots \times [0,1]^2}_{B_i \; \text{such terms}}$ is the parameter space of domain $i$. Thus, the dimension of $\Theta$ is $ \sum_{i=1}^D 2B_i$. The system parameter is $\theta^* = \left(\theta^*_{ij} \right)_{i \in [D], j \in [B_i]}$, where $\theta^*_{ij} = (\theta^*_{ij1}, \theta^*_{ij2}) \in [0,1]^2$ reflects some intrinsic properties of $\Block_{ij}$.

\textbf{Observation space $\calY$}. Let $\calY = \calY_1 \times \cdots \times \calY_D$ be the observation space of the whole system, where $\calY_i = [0,\infty)$ for each $i$. Given that action $a = (a_1, \cdots, a_D)$ is taken, the resource blocks $(\Block_{1, a_1}, \cdots, \Block_{D, a_D})$ are selected. $Y_i \in \calY_i$ is the observed latency in domain $i$, exhibited by $\Block_{i, a_i}$. Assume that $Y_i$ is observable by domain manager $i$ for each $i$. $Y_t = (Y_{t,1}, \cdots, Y_{t, D}) \in \calY$ is the system performance observed in all $D$ domains at time $t$. 

\textbf{Observation Model $P_\theta(\cdot | a, c)$}. Given action $a = (a_1, \cdots, a_D)$ and context $c = (c_1, c_2)$, the observation $Y = (Y_1, \cdots, Y_D)$ is generated by the following distribution: $Y_i$'s are independent and each $Y_i$ follows an exponential distribution with $\E[Y_i] = c_1 \theta^*_{ij1} + \theta^*_{ij2}$, where $j = a_i$. An interpretation of this expression is that the expected latency $\E[Y_i]$ exhibited by domain $i$ is positively related to the offered load $c_1$ of the requested service, due to queueing effects. $\theta^*_{ij1}$ is the rate at which the latency scales with the offered load at $\Block_{ij}$, $\theta^*_{ij2}$ is the baseline latency at $\Block_{ij}$. 

\textbf{Reward function $R$}. The reward function $R: \calY \times C \rightarrow \sR$ is defined by $R((Y_1, Y_2, Y_3), (c_1, c_2)) = g_{c_2}(Y_1 + Y_2 + Y_3)$, where $g_d$ for $0 \leq d \leq 1$ is defined by 
	\begin{equation*}
		g_d(y) = \left\{\begin{array}{cccc}
			\frac{y}{d} & if & 0 \leq y \leq d \\
			0 & if & y > d \,  \\
		\end{array} \right. .
	\end{equation*}
	
This reward function is based on two ideas. First, the minimum latency requirement $c_2$ in the context serves as a Service Level Agreement (SLA) between the mobile operator and the service provider. If the actual end-to-end latency is larger than $c_2$, SLA is violated and the mobile operator gets a huge penalty (zero reward). Second, minimizing the latency as much as possible might be an overkill, which could be costly. The mobile operator would be content with an observed latency that just meets the target. 

\subsection{Algorithm}

\begin{algorithm}[h]
	\caption{PTS for contextual stochastic bandit (per-system particles)}
	\label{alg:PTS_for_contextual_stochastic_bandit}
	\textbf{Inputs}: $\calC, \calA, \calY, \Theta, \theta^*, P_{\theta}(\cdot | a,c), R, \calP_N \subset \Theta$\\
	\textbf{Initialization}: $w_0 \gets \left(\frac{1}{N}, \cdots, \frac{1}{N}  \right)$
	\begin{algorithmic}[1]
		\For{$t = 1,2,\cdots$} 
		\State Get $c_t$
		\State Generate $\theta_t$ from $\calP_N$ according to weights $w_{t-1}$
		\State Play $A_t \gets \arg \max_{a \in \calA} \E_{\theta_t} \left[R(Y) | A_t = a, c_t\right]$
		\State Observe $Y_t \sim P_{\theta^*}(\cdot | A_t, c_t)$
		\For{$k \in \left\{1,2,\cdots, N\right\}$}
		\State $\widetilde{w}_{t,k} = w_{t-1,k} \; P_{\theta^{(k)}}(Y_t|A_t, c_t)$ 
		\EndFor
		\State $w_t \gets \text{normalize} \; \widetilde{w}_t$
		\EndFor
	\end{algorithmic}
\end{algorithm}

PTS (Algorithm \ref{alg:PTS_stochastic_bandit}) can be easily updated to include contexts, shown below in Algorithm \ref{alg:PTS_for_contextual_stochastic_bandit}. RPTS (Algorithm \ref{alg:RPTS}) can be similarly updated to include contexts: just update steps 1-8 of Algorithm \ref{alg:RPTS} to steps in Algorithm \ref{alg:PTS_for_contextual_stochastic_bandit}.

\begin{figure}[h]
\begin{center}
\centerline{\includegraphics[width=0.6\columnwidth]{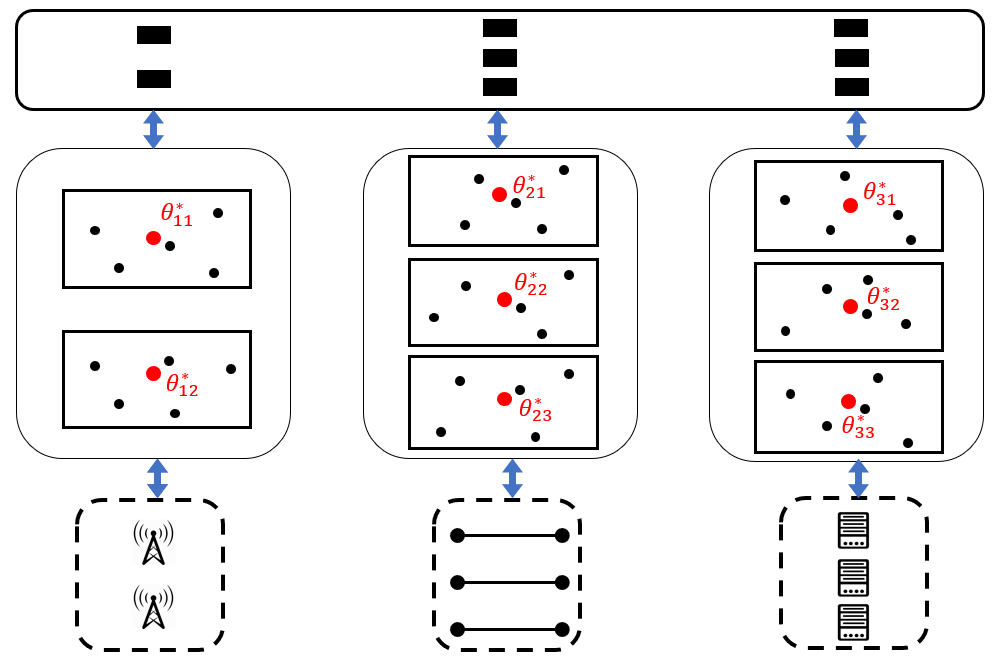}}
\caption{Per-block particles implementation.}
\label{fig:application|models|two_ways_of_generating_particles}
\end{center}
\vskip -0.2in
\end{figure}

In Algorithm \ref{alg:PTS_for_contextual_stochastic_bandit}, each particle in $\calP_N$ has the same dimension as $\theta^* \in \Theta$. However, due to the independence and availability of observations across the domains for this particular model, there is a more effective way to construct the particles and update their weights, called per-block particles, as follows (See Figure \ref{fig:application|models|two_ways_of_generating_particles} for an illustration). For each $\Block_{ij}$, we generate a set of $N$ particles $\calP_{ij} = \left\{\theta^{(1)}_{ij}, \cdots, \theta^{(N)}_{ij}\right\} \subset [0,1]^2$, which have weights $w_{t,ij} = \left(w_{t,ij,1}, \cdots, w_{t,ij,N} \right)$ at time $t$. In step 3 of Algorithm \ref{alg:PTS_for_contextual_stochastic_bandit}, we generate $\theta_t = \left\{\theta_{t,ij} \right\}_{i \in [D], j \in [B_i]}$ by  generating each $\theta_{t,ij}$ from $\calP_{ij}$ according to weights $w_{t, ij}$. Steps 6-8 of Algorithm \ref{alg:PTS_for_contextual_stochastic_bandit} then become:
	\begin{equation*}
		\begin{aligned}
			&\textbf{for} \; i \in \{1,2,\cdots, D\} \;  \textbf{do}: \\
			& \hspace{0.5cm} \textbf{for} \; k \in \{1, \cdots, N\} \; \textbf{do}: \\
			& \hspace{1cm}  \widetilde{w}_{t,i,A_{t,i},k} = w_{t-1, i, A_{t,i}, k} P_{\theta_{i, A_{t,i}}^{(k)}}(Y_{t,i} | A_{t,i}, c_t) \\
			& \hspace{0.5cm} w_{t,i,A_{t,i}} \leftarrow \text{normalize} \; \widetilde{w}_{t,i,A_{t,i}} \, 
		\end{aligned}
	\end{equation*} 
due to the independence of observations across domains. In essence, we  maintain a set of particles for each block, and in each time step, we only update the weights of the particles of the chosen block in each domain, while keeping unchanged the weights of the particles of the unused blocks. Per-block particle implementation stores the same number of parameter values in the system, $2N \sum_{i=1}^D B_i$, but the effective number of per-system particles is $N^{\sum_{i=1}^D B_i}$ (although these particles are not independent). 

For this model, the expectation in step 4 of Algorithm \ref{alg:PTS_for_contextual_stochastic_bandit} can be approximately calculated. See Appendix Section \ref{appendix:application}.

\subsection{Simulation}

Simulation setup:
$D = 3$ and $(B_1, B_2, B_3) = (3,3,3)$. In practice, $D$ and $B_i$'s are often small. Results are in Figure \ref{fig:application|simulation|model2_simulation}. Each curve is averaged over 100 independent simulations. In each simulation, the system parameter $\theta^*$ and the initial set of particles are randomly generated in the parameter space. Both PTS and RPTS work poorly with 10 per-block particles and is subject to much randomness. With 100 per-block particles, both algorithms are effective, although the improvement of PRTS compared to PTS is not obvious at the shown time scale. 

\begin{figure}[h]
\begin{center}
\centerline{\includegraphics[width=0.5\columnwidth]{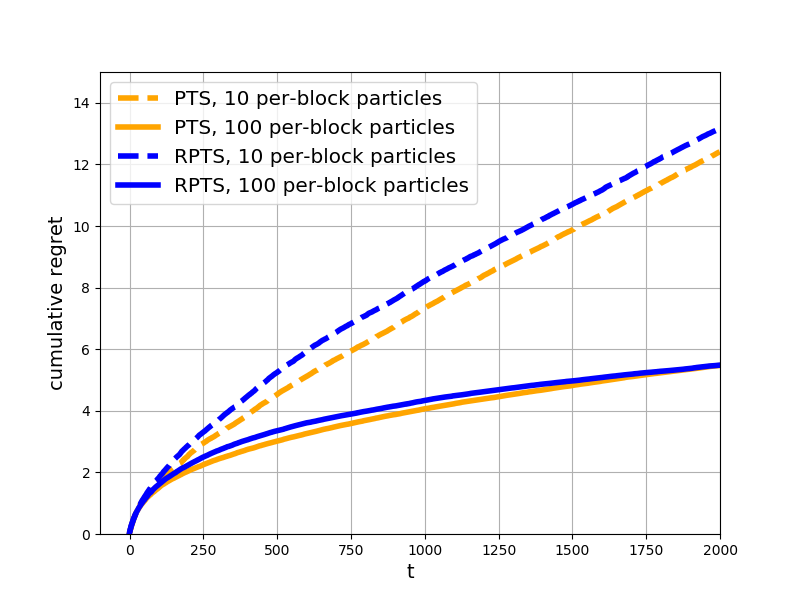}}
\caption{Simulation for network slicing.}
\label{fig:application|simulation|model2_simulation}
\end{center}
\vskip -0.3in
\end{figure}

\section{Conclusions and Future Work} \label{sec:conclusions}

This paper provides a practical variation of Thomson sampling. An analysis of PTS for general stochastic bandit problems is provided, by which we show that fit particles survive and unfit particles decay. We propose RPTS to improve PTS based on a simple heuristic that periodically deletes essentially inactive particles and regenerate new particles in the vicinity of survivors. We show empirically that RPTS significantly outperforms PTS in a set of representative bandit problems. Finally, we show an application of PTS and RPTS to network slicing and demonstrate through simulations that the algorithms are effective. 

Some directions for future work are as follows. First, the necessary survival condition in Proposition \ref{proposition:sample_path_necessary_survival_condition} may be further explored to provide insight on which particles can survive for some specific bandit problems. Second, while the particle regeneration strategy we used in PRTS is simple and effective, there may be other and more principle-guided  strategies that have some theoretical guarantees.

\bibliographystyle{plain}
\bibliography{references}

\appendix

\section{Proofs of Proposition \ref{proposition:sample_path_necessary_survival_condition} and Corollary \ref{corollary:PTS_for_K_arm_Bern_bandit|number_of_surviving_particles_is_no_more_than_K} } \label{appendix:analysis_of_PTS}

This section contains the proofs of Proposition \ref{proposition:sample_path_necessary_survival_condition} and Corollary \ref{corollary:PTS_for_K_arm_Bern_bandit|number_of_surviving_particles_is_no_more_than_K}. 

Let $L_{t,i} \triangleq \ln \widetilde{w}_{t,i} - \ln \widetilde{w}_{t-1,i}$. Assumption  \ref{assumption:PTS_for_general_stochastic_bandit|sample_path_asymptotic_particle_behavior|sample_path_assumptions} has two additional assumptions:
\begin{itemize}
	\item[(c)]  $\abs{\frac{1}{t} \sum_{\tau=1}^t \indicator_{\{I_\tau=i\}} - \frac{1}{t}\sum_{\tau=0}^{t-1} w_{\tau,i}} \rightarrow 0$ as $t \rightarrow \infty$ for any $i \in [N]$.
	\item[(d)] For any $i \in [N]$ that is used infinitely many times, $\frac{1}{M}\sum_{m=1}^M L_{t_i(m)} \rightarrow D_i$ as $M \rightarrow \infty$, where $t_i(m)$ is the $m$th time particle $i$ is chosen and $D_i$ is the $i$th row of the drift matrix $D$.  
\end{itemize}

In Assumption \ref{assumption:PTS_for_general_stochastic_bandit|sample_path_asymptotic_particle_behavior|sample_path_assumptions}(c), $\indicator_{\{I_\tau = i\}}$ is a Bernoulli random variable with mean $w_{\tau-1,i}$ for each $\tau$. Therefore it holds with probability one by the Azuma-Hoeffding inequality. Assumption \ref{assumption:PTS_for_general_stochastic_bandit|sample_path_asymptotic_particle_behavior|sample_path_assumptions}(d) holds with probability one by the definition of $D$ and the strong law of large numbers. 

The proof of Proposition \ref{proposition:sample_path_necessary_survival_condition} starts with the following lemma. All the lemmas in the rest of this proof deal with a sample path under Assumption \ref{assumption:PTS_for_general_stochastic_bandit|sample_path_asymptotic_particle_behavior|sample_path_assumptions}.

\begin{lemma}
	\label{lemma:PTS_for_general_stochastic_bandit|sample_path_asymptotic_particle_behavior|pi_is_well_defined}
	The probability vector $\pi$ is well defined. In addition, $\supp(\pi) = S$. That is, if $i \not\in S$, then $\pi_i = 0$; if $i \in S$, then $\pi_i > 0$. 
\end{lemma}

\begin{proof}
	For $i \not\in S$, 
	\begin{equation*}
		w_{t,i} = \frac{\widetilde{w}_{t,i}}{\sum_{j=1}^N \widetilde{w}_{t,j}} = \frac{e^{\ln \widetilde{w}_{t,i}}}{\sum_{j=1}^N e^{\ln \widetilde{w}_{t,j}}} \leq \frac{e^{\ln \widetilde{w}_{t,i}}}{e^{\ln \widetilde{w}_{t,j_0}}} 
	\end{equation*}
	for any $j_0 \in S$. By Assumption \ref{assumption:PTS_for_general_stochastic_bandit|sample_path_asymptotic_particle_behavior|sample_path_assumptions}(a), $w_{t,i} \rightarrow 0$. Hence $\pi_i = \lim_{t \rightarrow \infty} \frac{1}{t+1} \sum_{\tau=0}^t w_{t,i} = 0$.

	Next, define
	\begin{equation*}
		w'_{t,i} \triangleq \left\{\begin{array}{cccc}
			0 & if & i \notin S \\
			\frac{w_{t,i}}{\sum_{j \in S} w_{t,j}} & if & i \in S 
		\end{array} \right. \, . 
	\end{equation*}
	
	Fix $i \in S$. 
	\begin{equation*}
    \begin{aligned}
		w'_{t,i} - w_{t,i} &= w_{t,i}\left(\frac{1}{\sum_{j \in S} w_{t,j} } - 1 \right) = w_{t,i} \frac{\sum_{j \not\in S}  w_{t,j} }{\sum_{j \in S} w_{t,j} } =  w_{t,i} \frac{\sum_{j \not\in S}  w_{t,j} }{1 - \sum_{j \not\in S} w_{t,j}} \, .
	\end{aligned}
	\end{equation*}
	
	Since the set $[N] \backslash S$ is finite, $\sum_{j \not\in S} w_{t,j} \rightarrow 0$. It follows that $w'_{t,i} - w_{t,i} \rightarrow 0$. Hence 
	\begin{equation}
		\label{eq:PTS_for_general_stochastic_bandit|sample_path_asymptotic_particle_behavior|pi_is_well_defined_lemma|eq1}
		\frac{1}{t+1}\sum_{\tau=0}^t w'_{\tau,i} - \frac{1}{t+1}\sum_{\tau=0}^t w_{\tau,i} \rightarrow 0 \, . 
	\end{equation}
	
	Now, observe that $w'_{t,i}$ can be determined from $\{\ln \widetilde{w}_{t,j}\}_{j \in S}$ by $w'_{t,i} = \frac{e^{\ln \widetilde{w}_{t,i}}}{\sum_{j \in S} e^{\ln \widetilde{w}_{t,j}}}$. Therefore, $w'_{t,i}$ is a continuous and bounded function of $\{\ln \widetilde{w}_{t,j} \}_{j \in S}$, and hence of $G_t$. We write this as $w'_{t,i} = w'_i(G_t)$. According to Assumption \ref{assumption:PTS_for_general_stochastic_bandit|sample_path_asymptotic_particle_behavior|sample_path_assumptions}(b),  
	\begin{equation}
		\label{eq:PTS_for_general_stochastic_bandit|sample_path_asymptotic_particle_behavior|pi_is_well_defined_lemma|eq2}
		\frac{1}{t+1} \sum_{\tau=0}^t w'_{\tau,i} \rightarrow \E_{\mu_G}[w'_i]  \, . 
	\end{equation}
	
	Combining (\ref{eq:PTS_for_general_stochastic_bandit|sample_path_asymptotic_particle_behavior|pi_is_well_defined_lemma|eq1}) and $(\ref{eq:PTS_for_general_stochastic_bandit|sample_path_asymptotic_particle_behavior|pi_is_well_defined_lemma|eq2})$, we obtain $\pi_i = \E_{\mu_G}[w'_i]$. Since $w_i'$ is a positive function and $\mu_G$ is a distribution, we conclude that $\pi_i > 0$ for $i \in S$. 
	
	Finally, 
	\begin{equation*}
	\begin{aligned}
		\sum_{i \in [N]} \pi_i &= \sum_{i \in [N]} \lim_{t \rightarrow \infty} \frac{1}{t+1} \sum_{\tau=0}^t w_{\tau, i} \overset{(i)}{=} \lim_{t \rightarrow \infty} \sum_{i \in [N]} \frac{1}{t+1} \sum_{\tau=0}^t w_{\tau,i} = \lim_{t \rightarrow \infty} \frac{1}{t+1} \sum_{\tau=0}^t \sum_{i \in [N]} w_{\tau,i} = \lim_{t \rightarrow \infty} 1 = 1 \, ,
	\end{aligned}
	\end{equation*}
	where in step $(i)$ we switch the limit and summation because all summands are non-negative and $N$ is finite. Thus $\pi$ is well defined. 
\end{proof}

\begin{lemma}
	\label{lemma:PTS_for_general_stochastic_bandit|sample_path_asymptotic_particle_behavior|lemma2}
	$\frac{1}{t}\sum_{\tau=1}^t L_\tau \rightarrow \pi D$ as $t \rightarrow \infty$. 
\end{lemma}

\begin{proof}
	Let $M_i(t)$ be the number of times particle $i$ has been played up to time $t$. Let $\tau_i(m)$ be the $m$th time that particle $i$ is played. Then 
	\begin{equation*}
    \begin{aligned}
		\frac{1}{t}\sum_{\tau=1}^t L_\tau &= \frac{1}{t}\sum_{i=1}^N \sum_{m=1}^{M_i(t)} L_{\tau_i(m)} = \sum_{i=1}^N \frac{M_i(t)}{t} \frac{1}{M_i(t)} \sum_{m=1}^{M_i(t)} L_{\tau_i(m)} \, . 
	\end{aligned}
	\end{equation*}
	Since $M_i(t) = \sum_{\tau=1}^t \indicator_{\{I_\tau = i \}}$, by Assumption \ref{assumption:PTS_for_general_stochastic_bandit|sample_path_asymptotic_particle_behavior|sample_path_assumptions}(c) and the definition of $\pi_i$, $\frac{M_i(t)}{t} \rightarrow \pi_i$ for all $i \in [N]$. If particle $i$ is played infinitely many times in the sample path, then $\frac{1}{M_i(t)} \sum_{m=1}^{M_i(t)} L_{\tau_i(m)} \rightarrow D_i$ as $t \rightarrow \infty$ by Assumption \ref{assumption:PTS_for_general_stochastic_bandit|sample_path_asymptotic_particle_behavior|sample_path_assumptions}(d). If particle $i$ is played finitely many times, thus $M_i(t) \leq C$ for some constant $C$ for all $t$, then $\frac{M_i(t)}{t} \rightarrow 0$ and $\lim_{t \rightarrow \infty} \frac{1}{M_i(t)} \sum_{m=1}^{M_i(t)} L_{\tau_i(m)} < \infty$. Either case, we have 
	\begin{equation*}
		\frac{M_i(t)}{t} \frac{1}{M_i(t)} \sum_{m=1}^{M_i(t)} L_{\tau_i(m)} \rightarrow \pi_i D_i \quad \text{as} \quad t \rightarrow \infty \, . 
	\end{equation*}
	It follows that 
	\begin{equation*}
		\frac{1}{t} \sum_{\tau=1}^t L_\tau \rightarrow \sum_{i=1}^N \pi_i D_i = \pi D \quad \text{as} \quad t \rightarrow \infty \, . 
	\end{equation*}
\end{proof}

\begin{lemma}
	\label{lemma:PTS_for_general_stochastic_bandit|sample_path_asymptotic_particle_behavior|lemma3}
	If a real-valued sequence $\{x_t\}_{t \geq 1}$ satisfies
	\begin{enumerate}
		\item[(1)] $\{x_t\}$ has a limiting distribution $\mu$.
		\item[(2)] $\{x_t\}$ is $B$-Lipschitz: there exists some constant $B$ such that $\abs{x_t - x_s} \leq B \abs{t-s}$ for all $t, s \in \sN^+$. 
	\end{enumerate}
	Then $\lim_{t \rightarrow} \frac{1}{t} x_t = 0$. 
\end{lemma}

\begin{proof}
	We show $\limsup_{t \rightarrow \infty} \frac{1}{t}x_t \leq \delta$ for any $\delta > 0$. Suppose there exists $\delta > 0$ such that $\limsup_{t \rightarrow \infty} \frac{1}{t} x_t > \delta$. Condition $(1)$ implies that, there exists $c \in \sR$ such that
	\begin{equation}
		\label{eq:K_arm_bern_bandit|asymptotic_particle_weights|lemma2_condition_i_implication}
		\frac{1}{t} \sum_{\tau = 1}^t \indicator_{\{x_\tau \geq c \}} \leq \frac{\delta}{2B} \quad \text{for all} \; t \; \text{sufficiently large} \, . 
	\end{equation}
	Let $\{t_1, t_2, \cdots, t_n, \cdots \}$ be a sequence of positive integers such that $\lim_{n \rightarrow \infty} t_n = \infty$ and $\frac{1}{t_n} x_{t_n} \geq \delta$ for all $n$. Thus $x_{t_n} \geq \delta t_n$ for all $n$. Since $\{x_t\}$ is $B$-Lipschitz, for any $t \in [1, t_n]$, 
	\begin{equation*}
		x_t \geq x_{t_{n}} - B(t_n - t) \geq \delta t_n - B(t_n - t) = Bt - (B - \delta) t_n \, . 
	\end{equation*}
	It follows that, if $t \geq \frac{c}{B} + \left(1-\frac{\delta}{B}\right)t_n$, then $x_t \geq c$. Therefore, for $t_n > \frac{2c}{\delta}$, 
	\begin{equation*}
	\begin{aligned}
		\frac{1}{t_n} \sum_{\tau=1}^{t_n} \indicator_{\{\tau \geq c\}} &\geq  \frac{1}{t_n} \sum_{\tau=1}^{t_n} \indicator_{\left\{\tau \geq \frac{c}{L} + \left(1-\frac{\delta}{L}\right) t_n\right\}} = \frac{1}{t_n} \left[t_n - \left(\frac{c}{B}+\left(1-\frac{\delta}{B}\right) t_n \right)\right] = \frac{\delta}{B} - \frac{c}{B t_n} >\frac{\delta}{2B} \, ,
	\end{aligned}
	\end{equation*}
	which contradicts (\ref{eq:K_arm_bern_bandit|asymptotic_particle_weights|lemma2_condition_i_implication}). Therefore, $\limsup_t \frac{1}{t}x_t \leq \delta$ for any $\delta > 0$. Similarly, we can show that $\liminf_{t \rightarrow \infty} \frac{1}{t} x_t \geq -\delta$ for any $\delta > 0$. We conclude that $\lim_{t \rightarrow \infty} \frac{1}{t} x_t = 0$.
\end{proof}

\begin{lemma}
	\label{lemma:PTS_for_general_stochastic_bandit|sample_path_asymptotic_particle_behavior|lemma4}
	If $i,j \in S$, then $(\pi D)_i = (\pi D)_j$.
\end{lemma}

\begin{proof}
	Consider $i, j \in S$. Then 
	\begin{equation*}
		\begin{aligned}
			\frac{1}{t} \sum_{\tau=1}^t L_{\tau,i} - \frac{1}{t} \sum_{\tau=1}^t L_{\tau,j} &= \frac{1}{t} \sum_{\tau=1}^t \left(L_{\tau,i} - L_{\tau,j}\right) \\
			&=  \frac{1}{t}\sum_{\tau=1}^t \left[\left(\ln \widetilde{w}_{\tau,i} - \ln \widetilde{w}_{\tau-1,i}\right) - \left(\ln \widetilde{w}_{\tau,j} - \ln \widetilde{w}_{\tau-1,j} \right) \right] \\
			&= \frac{1}{t} \left[ \left(\ln \widetilde{w}_{t,i} - \ln \widetilde{w}_{0,i}\right) - \left( \ln \widetilde{w}_{t,j} - \ln \widetilde{w}_{0,j} \right) \right] \\
			&= \frac{1}{t} \left(\ln \widetilde{w}_{t,i} - \ln \widetilde{w}_{t,j} \right) = \frac{1}{t} G_t(i,j) \, . 
		\end{aligned}
	\end{equation*}
	The third equality above used $\ln \widetilde{w}_{0,i} = \ln \widetilde{w}_{0,j} = 0$ by initialization (although that is not important, as long as the difference is finite). By the dynamics of the weights $\{w_{t,i}\}$ and $\{w_{t,j} \}$, we have that
	\begin{equation*}
		G_{t+1}(i,j) = G_t(i,j) + \ln \frac{P_{\theta^{(i)}}(Y_{t+1}|A_{t+1})}{P_{\theta^{(j)}}(Y_{t+1}|A_{t+1})} \, . 
	\end{equation*}
	
	By Assumption \ref{assumption:PTS_for_general_stochastic_bandit|sample_path_asymptotic_particle_behavior|boundedness_of_observation_likelihood}, $\abs{G_{t+1}(i,j) - G_t(i,j)} \leq B$, where $B = \max\{\abs{\ln b_0}, \abs{\ln B_0}\}$. Thus $\{G_t(i,j)\}_{t \geq 1}$ is an $B$-Lipschitz sequence. Therefore 
	\begin{equation*}
    \begin{aligned}
		(\pi D)_i - (\pi D)_j &\overset{(i)}{=} \lim_{t \rightarrow \infty} \left(\frac{1}{t} \sum_{\tau=1}^t L_{\tau,i} - \frac{1}{t} \sum_{\tau=1}^t L_{\tau,j} \right) = \lim_{t \rightarrow \infty} \frac{1}{t} G_t(i,j) \overset{(ii)}{=} 0 \, ,
	\end{aligned}
	\end{equation*}
	where equality $(i)$ is due to Lemma \ref{lemma:PTS_for_general_stochastic_bandit|sample_path_asymptotic_particle_behavior|lemma2} and equality $(ii)$ equality is due to Lemma \ref{lemma:PTS_for_general_stochastic_bandit|sample_path_asymptotic_particle_behavior|lemma3} and Assumption \ref{assumption:PTS_for_general_stochastic_bandit|sample_path_asymptotic_particle_behavior|sample_path_assumptions}(b).
\end{proof}

\begin{lemma}
	\label{lemma:PTS_for_general_stochastic_bandit|sample_path_asymptotic_particle_behavior|lemma5}
	If $i \not\in S$ and $j \in S$, then $(\pi D)_i < (\pi D)_j$.  
\end{lemma}

\begin{proof}
	Similar to the proof of Lemma \ref{lemma:PTS_for_general_stochastic_bandit|sample_path_asymptotic_particle_behavior|lemma4}, we have
	\begin{equation*}
		\frac{1}{t}\sum_{\tau=1}^t L_{\tau,i} - \frac{1}{t}\sum_{\tau=1}^t L_{t,j} = \frac{1}{t} \left(\ln \widetilde{w}_{t,i} - \ln \widetilde{w}_{t,j} \right)
	\end{equation*}
	The LHS converges to $(\pi D)_i - (\pi D)_j$ as $t \rightarrow \infty$ by Lemma \ref{lemma:PTS_for_general_stochastic_bandit|sample_path_asymptotic_particle_behavior|pi_is_well_defined}. The RHS converges to a strictly negative value as $t \rightarrow \infty$ by Assumption \ref{assumption:PTS_for_general_stochastic_bandit|sample_path_asymptotic_particle_behavior|sample_path_assumptions}(a). Thus $(\pi D)_i < (\pi D)_j$.
\end{proof}

\begin{proof}[Proof of Proposition \ref{proposition:sample_path_necessary_survival_condition}]

Lemma \ref{lemma:PTS_for_general_stochastic_bandit|sample_path_asymptotic_particle_behavior|pi_is_well_defined} shows $\supp(\pi) = S$. Lemma \ref{lemma:PTS_for_general_stochastic_bandit|sample_path_asymptotic_particle_behavior|lemma4} and Lemma \ref{lemma:PTS_for_general_stochastic_bandit|sample_path_asymptotic_particle_behavior|lemma5} show $\arg \max (\pi D) = S$. Proposition \ref{proposition:sample_path_necessary_survival_condition} is thus proved. 

\end{proof}

\begin{proof}[Proof of Corollary \ref{corollary:PTS_for_K_arm_Bern_bandit|number_of_surviving_particles_is_no_more_than_K}]
	If $N \leq K$, then $\abs{\supp(\pi)} \leq N \leq K$ trivially. Let $N > K$. The observation model of a Bernoulli bandit problem satisfies Assumption \ref{assumption:PTS_for_general_stochastic_bandit|sample_path_asymptotic_particle_behavior|boundedness_of_observation_likelihood} trivially. By Proposition \ref{proposition:sample_path_necessary_survival_condition}, with probability one, for any sample path, the probability vector $\pi$ is well-defined and $\pi$ and $S$ satisfy $\arg \max (\pi D) = \supp(\pi) = S$, which implies the following constraints on $\pi$:
	
	\begin{equation}
		\label{eq:PTS_for_general_stochastic_bandit|PTS_for_K_arm_Bern_bandit|number_of_surviving_particles_is_no_more_than_K_proposition|proof|eq_a}
		\begin{aligned}
			&\; \pi_i = 0 \; \text{for} \; i \not\in S \, , \\
			&\; (\pi D)_i = (\pi D)_j  \; \text{for all} \; i, j \in S \, , \\
		\end{aligned}
	\end{equation}  
	where $S$ is the subset of $[N]$ in Assumption \ref{assumption:PTS_for_general_stochastic_bandit|sample_path_asymptotic_particle_behavior|sample_path_assumptions}. Suppose $|S| > K$. The remainder of the proof shows that, with probability one, any $\pi$ that satisfies (\ref{eq:PTS_for_general_stochastic_bandit|PTS_for_K_arm_Bern_bandit|number_of_surviving_particles_is_no_more_than_K_proposition|proof|eq_a}) is the all-zero vector (thus $\pi$ cannot be a probability vector). This leads to a contradiction with $\abs{S} > K$ and therefore we  conclude that $|S| \leq K$.
	
	We construct a matrix $\widetilde{D} \in \sR^{K \times N}$ and a probability (row) vector $\widetilde{\pi} \in [0,1]^K$ from $D$ and $\pi$, as follows. 
	
	Recall that, row $i_1$ and row $i_2$ of $D$ are the same if $A(i_1) = A(i_2)$. Since there are $K$ arms, there can be at most $K$ unique rows in $D$. Let $\widetilde{D}$ be $D$ reduced to its unique $K$ rows. That is, $\widetilde{D}_k =  \E[L_t | A_t = k]$ (which is independent of $t$) for $k \in [K]$.  
	
	For $k \in [K]$, let $\widetilde{\pi}_k = \sum_{i:i \in S, A(i) = k} \pi_i$. That is, $\widetilde{\pi}_k$ is the sum of the asymptotic weights of surviving particles with the optimal arm $k$. If no $i \in S$ satisfies $A(i) = k$, then $\widetilde{\pi}_k = 0$. It is easy to verify that $\widetilde{\pi}_1 + \cdots  + \widetilde{\pi}_K = 1$. 
	
	Now, observe that, 
	\begin{equation*}
		\begin{aligned}
			\pi D &= \sum_{i = 1}^N \pi_i D_i = \sum_{i \in S} \pi_i D_i = \sum_{k=1}^K \sum_{i: i \in S, A(i) = k} \pi_i D_i = \sum_{k=1}^K \sum_{i: i \in S, A(i) = k} \pi_i \widetilde{D}_k \\
			&=  \sum_{k=1}^K \left(\sum_{i: i \in S, A(i) = k} \pi_i\right) \widetilde{D}_k = \sum_{k=1}^K \widetilde{\pi}_k \widetilde{D}_k = \widetilde{\pi} \widetilde{D} \, . 
		\end{aligned}
	\end{equation*}
	
	Therefore, the constraints (\ref{eq:PTS_for_general_stochastic_bandit|PTS_for_K_arm_Bern_bandit|number_of_surviving_particles_is_no_more_than_K_proposition|proof|eq_a}) on $\pi$ imply the following constraints on $\widetilde{\pi}$:
	\begin{equation}
		\label{eq:PTS_for_general_stochastic_bandit|PTS_for_K_arm_Bern_bandit|number_of_surviving_particles_is_no_more_than_K_proposition|proof|eq_b}
		\begin{aligned}
			&\; (\widetilde{\pi} \widetilde{D})_i = (\widetilde{\pi} \widetilde{D})_j \; \text{for all} \; i, j \in S \, . 
		\end{aligned}
	\end{equation}
	
	Let $\widetilde{D}_i$ be the $i$th column of $\widetilde{D}$. Then $(\widetilde{\pi}\widetilde{D})_i = \angles{\widetilde{\pi}, \widetilde{D}_i}$. Constraints (\ref{eq:PTS_for_general_stochastic_bandit|PTS_for_K_arm_Bern_bandit|number_of_surviving_particles_is_no_more_than_K_proposition|proof|eq_b}) can thus be re-written as 
	\begin{equation}
		\label{eq:PTS_for_general_stochastic_bandit|PTS_for_K_arm_Bern_bandit|number_of_surviving_particles_is_no_more_than_K_proposition|proof|eq_c}
		\angles{\widetilde{\pi}, \widetilde{D}_i - \widetilde{D}_j} = 0 \; \text{for all} \; i, j \in S \, . 
	\end{equation}
	
	For a Bernouli bandit problem, the entries in $\widetilde{D} = [\widetilde{D}_{kj}]_{1\leq k \leq K, 1 \leq j \leq N}$ are in the form $\widetilde{D}_{kj} = -d(\theta^*_k || \theta^{(j)}_k)$, where $d(x||y) = x \ln \frac{x}{y} + (1-x) \ln \frac{1-x}{1-y}$ for $x,y \in [0,1]$ and $\theta^{(j)}_k$ is uniformly distributed in $[0,1]$ and is independent across $k \in  [K]$ and $j \in [N]$. Therefore, since $|S| > K$, with probability one, the set of vectors $\{\widetilde{D}_i - \widetilde{D}_j: i, j \in S \}$ spans $\sR^K$, in which case the only $\widetilde{\pi} \in \sR^K$ that satisfies (\ref{eq:PTS_for_general_stochastic_bandit|PTS_for_K_arm_Bern_bandit|number_of_surviving_particles_is_no_more_than_K_proposition|proof|eq_c}) is the all-zero vector. By construction of $\widetilde{\pi}$, with probability one, the only vector $\pi \in \sR^N$ that satisfies (\ref{eq:PTS_for_general_stochastic_bandit|PTS_for_K_arm_Bern_bandit|number_of_surviving_particles_is_no_more_than_K_proposition|proof|eq_a}) is the all-zero vector. 
\end{proof}

\section{Analysis of PTS for Two-Arm Bernoulli Bandit} \label{appendix:PTS_for_two_arm_Bernoulli_bandit}

This section considers perhaps the most simple bandit problem in more depth than Proposition \ref{proposition:sample_path_necessary_survival_condition}.   The results provide further intuition about PTS and about the assumptions and conclusions of Proposition 1 and its corollary.
Specifically, we analyze PTS for the two-arm Bernoulli bandit problem. 

The section is organized as follows. Subsection \ref{subsec:PTS_for_two_arm_Bernoulli_bandit|N_given_particles_weight_dynamics} provides a general analysis of the weight dynamics for $N$ given particles. Subsection \ref{subsec:PTS_for_two_arm_Bernoulli_bandit|two_given_particles} takes a closer look at the case of two given particles, including, in particular, the counter-reinforcing pair and the self-reinforcing pair. Subsection \ref{subsec:PTS_for_two_arm_Bernoulli_bandit|N_given_particles_asymptotic_behavior} discusses the asymptotic behavior of $N$ given particles. Subsection \ref{subsec:PTS_for_two_arm_Bernoulli_bandit|N_random_particles} discusses the performance of PTS for $N$ randomly generated particles, including two ways of generation: coordinate-wise and whole-particle. Subsection \ref{subsec:PTS_for_two_arm_Bernoulli_bandit|summary} summarizes the results in this section. Subsection \ref{subsec:PTS_for_two_arm_Bernoulli_bandit|useful_drift_implied_bounds} includes for reference two known bounds that are used in this section.

For a two-arm Bernoulli bandit problem, $\calA = \{1,2\}, \calY = \{0,1\}, \Theta = [0,1]^2, R(y) = y$.  PTS (Algorithm \ref{alg:PTS_stochastic_bandit}) is then reduced to Algorithm \ref{alg:PTS_for_two_arm_Bernoulli_bandit} below. 

\begin{algorithm}
	\caption{PTS for two-arm Bernoulli bandit}
	\label{alg:PTS_for_two_arm_Bernoulli_bandit}
	\textbf{Input}: $\theta^*, \calP_N$ \\
	\textbf{Initialization}: weights $w_0 \gets \left(\frac{1}{N}, \cdots, \frac{1}{N}\right)$, unnormalized weights $\widetilde{w}_0 \gets \left(1, \cdots, 1\right)$.
	\begin{algorithmic}[1]  
		\For{$t = 1,2,\cdots$} 
		\State Generate $\theta_t$ from $\calP_N$ according to weights $w_{t-1}$
		\State Play $A_t \gets \arg \max_{a \in \{1,2\}} \theta_{t,a}$
		\State Observe reward $R_t \sim \Bernoulli(\theta^*_{A_t})$
		\For{$i \in \left\{1,2,\cdots, N\right\}$}
		\State 		\begin{equation}
		\begin{aligned}	\label{eq:PTS_for_two_arm_Bern_bandit|posterior_weights_update}
			\widetilde{w}_{t,i} &= \widetilde{w}_{t-1,i} P_{\theta^{(i)}_{A_t}}(R_t) = \left\{\begin{array}{cccc}
				\widetilde{w}_{t-1,i}  \theta^{(i)}_{A_t} & if & R_t = 1 \\
				\widetilde{w}_{t-1,i}  (1-\theta^{(i)}_{A_t}) & if & R_t = 0 \\
			\end{array}\right. .
		\end{aligned}
		\end{equation}
		\EndFor
		\State $w_t \gets \text{normalize} \; \widetilde{w}_t$
		\EndFor
	\end{algorithmic}
\end{algorithm}

Notation: Let $w_{t,i}, \widetilde{w}_{t,i}, \bar{w}_{t,i}$ be the normalized, unnormalized, and running-average weight of particle $i \in [N$ at time $t$, respectively. Let $w_t = (w_{t,1}, \cdots, w_{t,N})$. let $I_t \in [N]$ be the index of the particle chosen at time $t$; $I_t \sim w_{t-1}$. Let $q_{t,i}$ be the fraction of time particle $i$ has been played up to time $t$, i.e., $q_{t,i} = \frac{1}{t} \sum_{\tau=1}^t \indicator_{\{I_t = i\}}$. Let $A_t \in \calA = \{1,2\}$ be the action/arm taken at time $t$. Let $A: [0,1]^2 \rightarrow \{1,2\}$ be the function mapping from a particle to the corresponding best action/arm, defined by $A(\theta) = \arg\max_{a \in \{1,2\}} \theta_a$. In the case $\theta_1 = \theta_2$, we let $A(\theta)$ equal to either $\theta_1$ or $\theta_2$ deterministically. With a slight abuse of notation, we sometimes abbreviate $A(\theta^{(i)})$ by $A(i)$. Thus $A_t = A(I_t)$. Let $r_t \in [0,1]$ be the usage frequency of arm 1 at time $t$, namely, the fraction of time that arm 1 has been pulled up to and including time $t$. It follows that $1-r_t$ is the usage frequency of arm 2 at time $t$. Let $d(x||y) \triangleq x \ln \frac{x}{y} + (1-x) \ln \frac{1-x}{1-y}$ denote the KL-divergence between two Bernoulli distributions parameterized by $x$ and $y$ respectively. Let $D_i(r) \triangleq rd(\theta^*_1 || \theta^{(i)}_1) + (1-r)d(\theta^*_2 || \theta^{(i)}_2)$ denote the convex combination of the KL divergences between $\theta^*$ and $\theta^{(i)}$ at the two arms, with weight $r$ on arm 1 and weight $1-r$ on arm 2, for some $r \in [0,1]$. For brevity, we shall call $D_i(r)$ the \emph{divergence of particle $i$ at $r$}. Let an instance of a two-arm Bernouli bandit problem with parameter $\theta^*$ be denoted as $\BernoulliBandit(K=2, \theta^*)$.

\subsection{N given particles, weight dynamics} \label{subsec:PTS_for_two_arm_Bernoulli_bandit|N_given_particles_weight_dynamics}

We start with some informal analysis to provide some high-level intuition. Consider the process in Algorithm \ref{alg:PTS_for_two_arm_Bernoulli_bandit}. Consider a given particle $\theta^{(i)} \in \calP_N$. By (\ref{eq:PTS_for_two_arm_Bern_bandit|posterior_weights_update}), the unnormalized weight of particle $i$ at time $t$ can be written as
\begin{equation*}
\begin{aligned}
	\widetilde{w}_{t,i} &= \prod_{\tau=1}^t P_{\theta^{(i)}_{A_\tau}}(R_\tau) = \exp \left(\sum_{\tau=1}^t \ln P_{\theta^{(i)}_{A_\tau}}(R_\tau) \right) = \exp \left(\sum_{\tau \in \calT_1} \ln P_{\theta^{(i)}_1}(R_\tau) + \sum_{\tau \in \calT_2} \ln P_{\theta^{(i)}_2}(R_\tau) \right) \, ,
\end{aligned}
\end{equation*}
where $\calT_a \triangleq \{\tau \in \{1,\cdots,t\}: A_\tau = a\}$ for $a = 1,2$, i.e., $\calT_a$ is the set of time instances up to time $t$ at which arm $a$ is played. By the definition of $r_t$, $\abs{\calT_1} = t r_t$ and $\abs{\calT_2} = t(1-r_t)$.  Suppose both $\abs{\calT_1}$ and $\abs{\calT_2}$ are non-zero and grow with $t$. For large $t$, we have
\begin{equation*}
	\begin{aligned}
		 \frac{1}{t} \ln \widetilde{w}_{t,i} &=  r_t \frac{1}{t r_t} \sum_{\tau \in \calT_1} \ln P_{ \theta^{(i)}_1}(R_\tau) + (1-r_r)\frac{1}{t(1-r_t)} \sum_{\tau \in \calT_2} \ln P_{\theta^{(i)}_2}(R_\tau) \\
		&\approx r_t \E_{\theta^*} \left[\ln P_{\theta^{(i)}_1}(R_1) \right] + (1-r_t) \E_{\theta^*} \left[\ln P_{\theta^{(i)}_2}(R_1) \right] \\
		&= r_t \left(-d(\theta^*_1 || \theta^{(i)}_1) - H(\theta^*_1)\right) + (1-r_t)\left(-d(\theta^*_2 || \theta^{(i)}_2) - H(\theta^*_2) \right) \\
		&= -D_i(r_t) - \left(r_t H(\theta^*_1) + (1-r_t) H(\theta^*_2) \right) \, .
	\end{aligned}
\end{equation*} 
The term $r_t H(\theta^*_1) + (1-r_t) H(\theta^*_2)$ doesn't depend on $i$. Therefore, for large $t$, $\widetilde{w}_{t,i} \appropto e^{-t D_i(r_t)}$. 
The above discussion can be made formal by the following proposition.  

\begin{proposition}
	\label{proposition:PTS_for_two_arm_Bern_bandit|weight_dynamics_of_N_particles|unnormalized_weight_dynamics}
	Given a problem $\text{BernoulliBandit}(K=2, \theta^*)$ and a particle set $\calP_N \subset [0,1]^2$. Consider the process of running $\text{PTS}(\calP_N)$ as in Algorithm \ref{alg:PTS_for_two_arm_Bernoulli_bandit}. 
	For any $i \in \{1,\cdots,N \}$ and $t \geq 1$, 
	\begin{equation}
		\label{eq:PTS_for_two_arm_Bern_bandit|weight_dynamics_of_N_particles|weight_dynamics_proposition|statement|unnormalized_weight}
		\frac{1}{t} \ln \widetilde{w}_{t,i} = -D_i(r_t) + \epsilon_{t,i} + C(r_t) \, , 
	\end{equation}
	where $C(r_t)$ is a given function on $r_t$ that does not depend on $i$, and $\{\epsilon_{t,i}\}_{t\geq 1}$ is a random sequence that converges to zero in probability.\footnote{It can be further shown that this convergence is almost sure by using the Borel-Cantelli lemma. We state the convergence in probability result here because it will be used later.} More specifically, for some positive constant $B_{\theta^{(i)}}$ depending on $\theta^{(i)}$, 
	\begin{equation}
		\label{eq:PTS_for_two_arm_Bern_bandit|weight_dynamics_of_N_particles|weight_dynamics_proposition|statement|error_term_in_exponent_converges_in_prob}
		P \left\{\abs{\epsilon_{t,i}} > \delta \right\} \leq 4t e^{-B_{\theta^{(i)}} \delta^2 t} 
	\end{equation}
	for any $\delta > 0$ and $t \geq 1$.
\end{proposition}

\begin{proof}
	Let $N_{t,a}$ be the number of times action $a$ has been played up to time $t$, $a \in \{1,2\}$. $N_{t,1} + N_{t,2} = t$. Consider the following alternative construction of the reward generation process. Before the game starts, we generate a value $Z_a(k)$ for each action $a \in \{1,2\}$ and each time $k=1,2,\cdots$ independently according to the distribution $\Bernoulli(\theta^*_a)$. At each step $t$, playing action $A_t = a$ yields reward $R_t = Z_a(N_{t,a})$. That is, step 4 of Algorithm \ref{alg:PTS_for_two_arm_Bernoulli_bandit} becomes $R_t = Z_{A_t}(N_{t, A_t})$. It is easy to see that the distributions of any given sample path seen by the algorithm in both constructions are identical. Therefore, we can equivalently work with the alternative construction whenever it is more convenient.

	We have
	\begin{equation*}
		\begin{aligned}
			\widetilde{w}_{t,i} &= \exp \left(\sum_{\tau=1}^t \ln P_{\theta^{(i)}_{A_\tau}} (R_\tau)  \right) = \exp \left(\sum_{a \in \{1,2\}} \sum_{\tau=1}^t \indicator_{\{A_\tau = a\}} \ln P_{\theta^{(i)}_a} (R_\tau) \right)  \\
			&= \exp \left(\sum_{a \in \{1,2\}} \sum_{\tau=1}^t \indicator_{\{A_\tau = a\}} \ln P_{\theta^{(i)}_a} (Z_a(N_{\tau,a})) \right)  = \exp \left(\sum_{a \in \{1,2\}} \sum_{k=1}^{N_{t,a}} \ln P_{\theta_a^{(i)}} (Z_a(k)) \right)  \\
		\end{aligned} 
	\end{equation*}
	for any time $t$ and particle $i \in \{1,\cdots, N\}$. The values in $\left\{\ln P_{\theta^{(i)}_1} (Z_1(k)) \right\}_{k=1}^{N_{t,1}}$ are i.i.d. random variables, each equals to $\ln \theta^{(i)}_1$ with probability $\theta^*_1$ or $\ln(1-\theta^{(i)}_1)$ with probability $1-\theta^*_1$, with mean $-d(\theta^*_1 || \theta^{(i)}_1) - H(\theta^*_1)$. Similarly, values in $\left\{\ln P_{\theta^{(i)}_2} (Z_2(k)) \right\}_{k=1}^{N_{t,2}}$ are i.i.d. random variables with mean $-d(\theta^*_2 || \theta^{(i)}_2) - H(\theta^*_2)$. It follows after some simple algebraic re-arrangements that 
	\begin{equation*}
		\begin{aligned}
		    \frac{1}{t} \ln \widetilde{w}_{t,i} &= \frac{1}{t} \left(\sum_{k=1}^{N_{t,1}} \ln P_{\theta^{(i)}_1}(Z_1(k)) + \sum_{k=1}^{N_{t,2}} \ln P_{\theta^{(i)}_2} (Z_2(k)) \right) \\
			&= -D_i(r_t) + \epsilon_{t,i} \underbrace{-r_t H(\theta^*_1) - (1-r_t) H(\theta^*_2)}_{\triangleq C(r_t)} \, ,
		\end{aligned}
	\end{equation*}
	where 
	\begin{equation*}
	\begin{aligned}
		\epsilon_{t,i} &= \frac{1}{t}  \underbrace{\left(\sum_{k=1}^{N_{t,1}} \ln P_{\theta^{(i)}_1}(Z_1(k)) - \left(-d(\theta^*_1 || \theta^{(i)}_1) - H(\theta^*_1)\right)\right)}_{\triangleq E_1(N_{t,1})} \\
		 &+ \frac{1}{t} \underbrace{\left(\sum_{k=1}^{N_{t,2}} \ln P_{\theta^{(i)}_2} (Z_2(k)) - \left(-d(\theta^*_2 || \theta^{(i)}_2) - H(\theta^*_2)\right)\right)}_{\triangleq E_2(N_{t,2})}  \, .
	\end{aligned}
	\end{equation*}
	$E_1(N_{t,1})$ is the sum of $N_{t,1}$ i.i.d. random variables, each has mean zero and is contained in an interval with length $\abs{\ln \theta^{(i)}_1 - \ln (1-\theta^{(i)}_1)}$. $N_{t,1}$ is a random variable that takes values in $\{1, \cdots, t\}$. Therefore, for any $\gamma > 0$, 
	\begin{equation}
		\label{eq:PTS_for_two_arm_Bern_bandit|weight_dynamics_of_N_particles|weight_dynamics_proposition|proof|control_E1_by_Azuma_Hoeffding}
		\begin{aligned}
			P\left\{\abs{E_1(N_{t,1})} > \gamma \right\} &= \sum_{n=1}^t P\left\{\abs{E_1(n)} > \gamma | N_{t,1} = n \right\} P\left\{N_{t,1} = n \right\} \\
			&\leq \sum_{n=1}^t P\left\{\abs{E_1(n)} > \gamma \right\} \\
			&\leq \sum_{n=1}^t 2 \exp \left(- \frac{2\gamma^2}{n\left(\ln \theta^{(i)}_1 - \ln (1-\theta^{(i)}_1) \right)^2} \right) \\
			&\leq \sum_{n=1}^t 2 \exp \left(- \frac{2\gamma^2}{t\left(\ln \theta^{(i)}_1 - \ln (1-\theta^{(i)}_1) \right)^2} \right) \\
			&= 2 t \exp \left(- \frac{2\gamma^2}{t\left(\ln \theta^{(i)}_1 - \ln (1-\theta^{(i)}_1) \right)^2} \right) \, . 
		\end{aligned}
	\end{equation}
	The second inequality above is due to the Azuma-Hoeffding inequality. Similarly, 
	\begin{equation}
		\label{eq:PTS_for_two_arm_Bern_bandit|weight_dynamics_of_N_particles|weight_dynamics_proposition|proof|control_E2_by_Azuma_Hoeffding}
		\begin{aligned}
			&\; P\left\{\abs{E_2(N_{t,2})} > \gamma \right\} \leq 2 t \exp \left(- \frac{2\gamma^2}{t\left(\ln \theta^{(i)}_2 - \ln (1-\theta^{(i)}_2) \right)^2} \right) \, . 
		\end{aligned}
	\end{equation}
	Using (\ref{eq:PTS_for_two_arm_Bern_bandit|weight_dynamics_of_N_particles|weight_dynamics_proposition|proof|control_E1_by_Azuma_Hoeffding}) and (\ref{eq:PTS_for_two_arm_Bern_bandit|weight_dynamics_of_N_particles|weight_dynamics_proposition|proof|control_E2_by_Azuma_Hoeffding}), we have 
	\begin{equation*}
		\begin{aligned}
			P\left\{\abs{\epsilon_{t,i}} \geq \delta \right\} \leq \sum_{a \in \{1,2\}}  P\left\{\abs{E_a(N_{t,a})} \geq \frac{t \delta}{2} \right\} \leq \sum_{a \in \{1,2\} } 2t \exp \left(-\frac{\delta^2 t}{2 \left(\ln \frac{\theta^{(i)}_a}{1-\theta^{(i)}_a} \right)^2} \right) \leq 4te^{-B_{\theta^{(i)}} \delta^2 t} \, ,
		\end{aligned}
	\end{equation*}
	where $B_{\theta^{(i)}} = \frac{1}{2}\min \left\{ \left(\ln \frac{\theta^{(i)}_1}{1-\theta^{(i)}_1} \right)^{-2},  \left(\ln \frac{\theta^{(i)}_2}{1-\theta^{(i)}_2} \right)^{-2}  \right\} $.
\end{proof}

Let us discuss the implication of Proposition \ref{proposition:PTS_for_two_arm_Bern_bandit|weight_dynamics_of_N_particles|unnormalized_weight_dynamics}. Since $C(r_t)$ does not depend on $i$, it follows from (\ref{eq:PTS_for_two_arm_Bern_bandit|weight_dynamics_of_N_particles|weight_dynamics_proposition|statement|unnormalized_weight}) that $\widetilde{w}_{t,i} \propto \exp\left(-t(D_i(r_t) + \epsilon_{t,i})\right)$. We make two observations here:
\begin{itemize}
	\item For large $t$, the term $\epsilon_{t,i}$ becomes insignificant.  The particle $i$ with the lowest $D_i(r_t)$ at time $t$ is more likely to have the largest normalized weight. In this sense, the divergence $D_i(r_t)$ reflects the fitness of particle $i$ for survival:  the smaller $D_i(r_t)$ is, the more fit particle $i$. However, we cannot simply say one particle is more fit than another without mentioning $r_t$, which is a random process. It is not clear at this point how $r_t$ evolves.   
	\item Obviously, $r_t$ is affected by the history of the particles' weights $\{\widetilde{w}_{\tau,i}: 1 \leq \tau \leq t-1, 1 \leq i \leq N\}$. 
\end{itemize} 

To investigate the interplay between the particles' weights $w_t$ (or $\widetilde{w}_{t}$) and their usage frequencies $(r_t, 1-r_t)$, we take a look at the simplest case: two given particles.

\subsection{Two given particles} \label{subsec:PTS_for_two_arm_Bernoulli_bandit|two_given_particles}

Before we discuss possible configurations of two given particles, we introduce a helpful graphical tool called the \emph{divergence diagram}. A divergence diagram example is drawn in Figure  \ref{fig:PTS_for_two_arm_Bern_bandit|two_fixed_particles|divergence_diagram_example}, with the divergence of a particle $i$, $D_i(r)$ for $0 \leq r \leq 1$, represented by a line segment. The right (respectively, left) endpoint of the line segment is highlighted by a dot if $A(\theta^{(i)}) = 1$ (respectively, if $A(\theta^{(i)}) = 2$), that is, arm 1 (respectively, arm 2) is the optimal arm if $\theta^{(i)}$ is the true parameter. Informally speaking, the closer the line segment is to the bottom, the more fit the corresponding particle is. A line segment that coincides with the bottom line segment represents $\theta^*$ itself, because the KL divergences on both arms are zero.  Note that, not every line segment in the diagram corresponds to a unique particle in $[0,1]^2$, because in general it is possible to have $d(x||y_1) = d(x||y_2)$ with $y_1 \not= y_2$.

\begin{figure}[h]
\vskip 0.2in
\begin{center}
\centerline{\includegraphics[width=0.5\columnwidth]{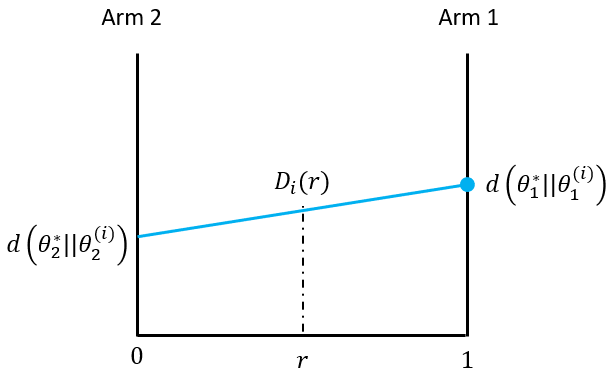}}
\caption{A divergence diagram example.}
\label{fig:PTS_for_two_arm_Bern_bandit|two_fixed_particles|divergence_diagram_example}
\end{center}
\vskip -0.2in
\end{figure}

\begin{figure}[h]
\vskip 0.2in
\begin{center}
\centerline{\includegraphics[width=0.7\columnwidth]{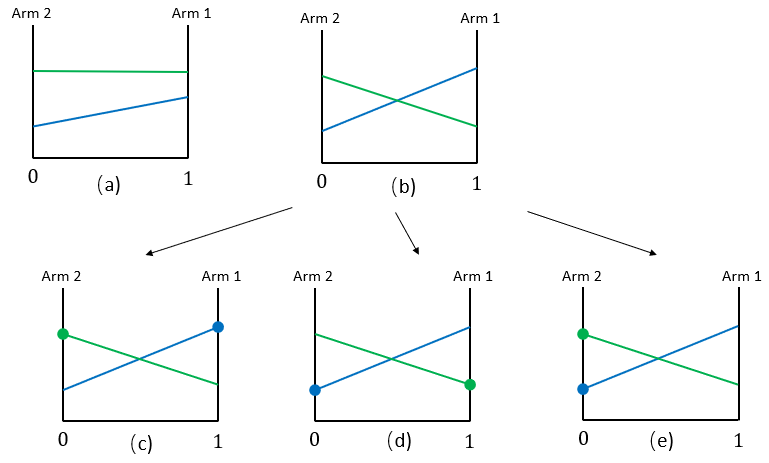}}
\caption{Possible two-particle configurations in the divergence diagram.}
\label{fig:PTS_for_two_arm_Bern_bandit|two_fixed_particles|all_possible_configurations}
\end{center}
\vskip -0.2in
\end{figure}

Consider the possible configurations of two particles in terms of their relative positions in the divergence diagram. See Figure \ref{fig:PTS_for_two_arm_Bern_bandit|two_fixed_particles|all_possible_configurations}.  
\begin{itemize}
	\item In case (a), The line segment of one particle is completely below the other particle. In this case, with probability one, the lower particle will gain all the weight. This is a trivial case. 
	\item In case (b), the line segments of two particles cross each other. This case can be further divided into three sub-cases, shown in (c), (d) and (e) respectively, depending on the optimal arm for each particle. In case (e), the optimal arm for both particles is the same. The problem essentially degenerates to a one-arm Bernoulli bandit problem, which is not so interesting. We will take a closer look at the remaining two cases: (c) counter-reinforcing pair and (d) self-reinforcing pair. 
\end{itemize}

\subsubsection{Counter-reinforcing pair} \label{subsection:PTS_for_two_arm_Bern_bandit|two_fixed_particles|CR_pair}

\begin{definition}(Counter-reinforcing pair)
	For a given $\BernoulliBandit(K=2, \theta^*)$ problem, we say that two particles $\{\theta^{(1)}, \theta^{(2)}\} \subset [0,1]^2$ form a \emph{counter-reinforcing pair (CR pair)} if they can be re-labeled such that the following conditions hold:
	\begin{equation}
	\label{eq:PTS_for_two_arm_Bern_bandit|two_fixed_particles|CR_pair|defining_conditions}
    \begin{aligned}
		d(\theta^*_1 || \theta^{(1)}_1) > d(\theta^*_1 || \theta^{(2)}_1), \; d(\theta^*_2 || \theta^{(1)}_2) < d(\theta^*_2 || \theta^{(2)}_2), A(1) = \{1\}, \;  A(2) = \{2\} \, . 
	\end{aligned}
	\end{equation}
\end{definition}

Note: The only way to re-label the two particles is to switch their labels. Without loss of generality, in the rest of this section, when we say $\{\theta^{(1)}, \theta^{(2)}\}$ form a CR pair, we mean that they have already been properly re-labeled to meet the conditions (\ref{eq:PTS_for_two_arm_Bern_bandit|two_fixed_particles|CR_pair|defining_conditions}).

\begin{figure}[h]
    \centering
    \subfloat[\centering Particle positions.]{{\includegraphics[width=0.4\columnwidth]{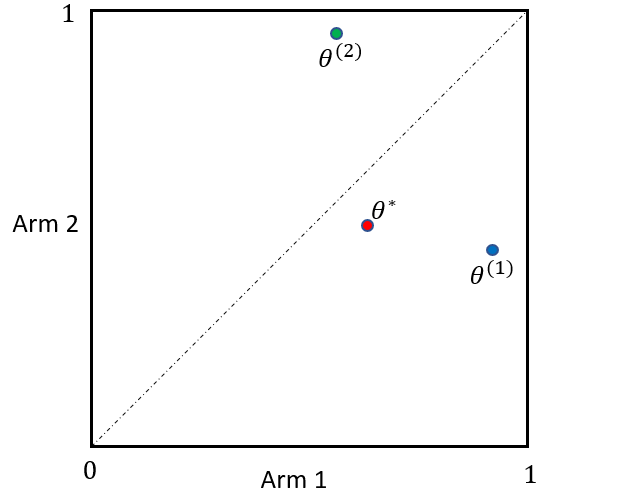} }}%
    \subfloat[\centering Divergences.]{{\includegraphics[width=0.55\columnwidth]{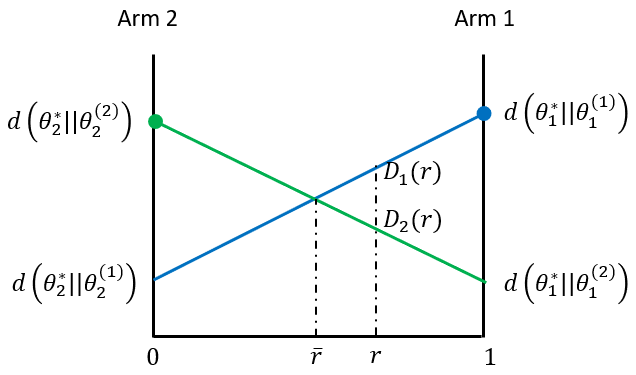} }}%
    \caption{A counter-reinforcing pair example.}%
    \label{fig:PTS_for_two_arm_Bern_bandit|two_fixed_particles|CR_pair|example}%
\end{figure}

A CR pair example is shown in Figure \ref{fig:PTS_for_two_arm_Bern_bandit|two_fixed_particles|CR_pair|example}. Figure \ref{fig:PTS_for_two_arm_Bern_bandit|two_fixed_particles|CR_pair|example}(a) depicts the positions of $\theta^*, \theta^{(1)}$ and $\theta^{(2)}$ in $[0,1]^2$. Figure \ref{fig:PTS_for_two_arm_Bern_bandit|two_fixed_particles|CR_pair|example}(b) depicts the divergences of the two particles. Let $\bar{r} \in (0,1)$ be such that $D_1(\bar{r}) = D_2(\bar{r})$, i.e., the point at which these two lines intersect. The definition of a CR pair guarantees that $\bar{r}$ exists and is unique. 

Consider a large time $t$. Suppose $r_t > \bar{r}$. Since $\widetilde{w}_{t,i} \appropto e^{-t D_i(r_t)}$ and $D_2(r_t) < D_1(r_t)$, we expect $w_{t,2}$ to be larger than $w_{t,1}$, thus particle 2 will be selected more often, which causes arm 2 to be pulled more often. But pulling arm 2 will make $r_t$ decrease. If $r_t$ decreases to a value less than $\bar{r}$, then by a similar argument we expect $w_{t,1}$ to become larger than $w_{t,2}$. Then particle 1 will be selected more often, which makes arm 1 to be pulled more often and $r_t$ to increase. Therefore, these two particles are \emph{counter-reinforcing} each other: selecting one particle will likely increase the weight of the other particle and vice versa. 

We expect to observe that $r_t$ cannot stay too far away either above or below $\bar{r}$. The drift of $r_t$ is always toward $\bar{r}$. However, we also observe through simulations that the weights of the two particles keep oscillating. The random oscillations are so strong that the drift does not make weights converge, that is, weights bounce around too much to converge, but are stochastically bounded. The above observations are formally stated in the following proposition. 

\begin{proposition}
	\label{proposition:PTS_for_two_arm_Bern_bandit|two_fixed_particles|CR_pair|almost_sure_convergence}
	Given a $\BernoulliBandit(K=2, \theta^*)$ problem and suppose a given particle set $\calP_2 = \{\theta^{(1)}, \theta^{(2)}\}$ form a CR pair for the problem. Consider the process of running $\PTS(\calP_2)$ as in Algorithm \ref{alg:PTS_for_two_arm_Bernoulli_bandit}. Let $\bar{r} \in (0,1)$ be the solution to $D_1(r) = D_2(r)$. Then, $r_t \rightarrow \bar{r}$ almost surely. Also, $q_t \rightarrow (\bar{r}, 1-\bar{r})$ and $\bar{w}_t \rightarrow (\bar{r},1-\bar{r})$ almost surely. 
\end{proposition}

The remainder of this section is dedicated to the proof of Proposition \ref{proposition:PTS_for_two_arm_Bern_bandit|two_fixed_particles|CR_pair|almost_sure_convergence}. The proof starts with constructing a sequence $\{X_t\}$, defined by $X_t \triangleq \ln \frac{\widetilde{w}_{t,1}}{\widetilde{w}_{t,2}} = \ln \frac{w_{t,1}}{w_{t,2}}$. Recall that, for $i = 1,2$, 
\begin{equation*}
\begin{aligned}
	\widetilde{w}_{t+1,i} &= \widetilde{w}_{t,i} P_{\theta^{(i)}_{A_{t+1}}}(R_{t+1}) = \left\{\begin{array}{cccc}
		\widetilde{w}_{t,i} \theta^{(i)}_{A_{t+1}} & if & R_{t+1} = 1 \\
		\widetilde{w}_{t,i} (1-\theta^{(i)}_{A_{t+1}}) & if & R_{t+1} = 0 \\
	\end{array} \right. .
\end{aligned}
\end{equation*}
By the conditions in (\ref{eq:PTS_for_two_arm_Bern_bandit|two_fixed_particles|CR_pair|defining_conditions}) that $A(1) = \{1\}$ and $A(2) = \{2\}$, $A_{t+1} = i$ iff particle $\theta^{(i)}$ is selected at time $t+1$, which occurs with probability $w_{t,i}$. So for $i = 1,2$, 
\begin{equation*}
	\widetilde{w}_{t+1,i} = \left\{\begin{array}{cccc}
		\widetilde{w}_{t,i} \theta^{(i)}_1 & w.p. & w_{t,1} \theta^*_1 \\
		\widetilde{w}_{t,i} (1-\theta^{(i)}_1) & w.p. & w_{t,1} (1-\theta^*_1) \\
		\widetilde{w}_{t,i} \theta^{(i)}_2 & w.p. & w_{t,2} \theta^*_2 \\
		\widetilde{w}_{t,i} (1-\theta^{(i)}_2) & w.p. & w_{t,2} (1-\theta^*_2) \\ 
	\end{array} \right. \, .  
\end{equation*}
Since $w_{t,1} + w_{t,2} = 1$, if we are given that $x = \ln \frac{\widetilde{w}_{t,1}}{\widetilde{w}_{t,2}} = \ln \frac{w_{t,1}}{w_{t,2}}$, then $w_{t,1} = \frac{e^x}{1+e^x}$ and $w_{t,2} = \frac{1}{1+e^x}$. It follows that
\begin{equation}
	\label{eq:PTS_for_two_arm_Bern_bandit|two_fixed_particles|CR_pair|Xt_expression}
	X_{t+1} = X_t + \left\{\begin{array}{cccc}
		\ln \frac{\theta^{(1)}_1}{\theta^{(2)}_1} & w.p. & \frac{e^{X_t}}{1+e^{X_t}} \theta^*_1 \\
		\ln \frac{(1-\theta^{(1)}_1)}{(1-\theta^{(2)}_1)} & w.p. & \frac{e^{X_t}}{1+e^{X_t}} (1-\theta^*_1) \\
		\ln \frac{\theta^{(1)}_2}{\theta^{(2)}_2} & w.p. & \frac{1}{1+e^{X_t}} \theta^*_2 \\
		\ln \frac{(1-\theta^{(1)}_2)}{(1-\theta^{(2)}_2)} & w.p. & \frac{1}{1+e^{X_t}} (1-\theta^*_2)  \\
	\end{array} \right. \, . 
\end{equation}
Note that $X_0 = 0$ since $w_{0,1} = w_{0,2} = \frac{1}{2}$. $\{X_t\}_{t \geq 0}$ is a time-homogeneous Markov process living in a state space of infinite cardinality. Note that (\ref{eq:PTS_for_two_arm_Bern_bandit|two_fixed_particles|CR_pair|Xt_expression}) is derived using only the conditions $A(1) = \{1\}$ and $A(2) = \{2\}$ in (\ref{eq:PTS_for_two_arm_Bern_bandit|two_fixed_particles|CR_pair|defining_conditions}), therefore it holds even if the two particles do not form a CR pair. The dynamics of $X_t$ in (\ref{eq:PTS_for_two_arm_Bern_bandit|two_fixed_particles|CR_pair|Xt_expression}) will be used again in the next section in the case of a self-reinforcing pair.

In the next lemma, we show that $\{X_t\}$ is stochastically bounded given the CR pair conditions. 

\begin{lemma}
	\label{lemma:PTS_for_two_arm_Bern_bandit|two_fixed_particles|CR_pair|Xt_is_stochastically_bounded}
	Consider the process described in Proposition \ref{proposition:PTS_for_two_arm_Bern_bandit|two_fixed_particles|CR_pair|almost_sure_convergence}. Let $X_t \triangleq \ln \frac{\widetilde{w}_{t,1}}{\widetilde{w}_{t,2}} = \ln \frac{w_{t,1}}{w_{t,2}}$. Then, for some constants $A_0$ and $B_0$ depending on $\theta^*$ and $\calP_2 = \{\theta^{(1)}, \theta^{(2)}\}$,
	\begin{equation*}
		P \left\{\abs{X_t} \geq x \right\} \leq A_0 e^{-B_0 x} \quad \forall t \geq 1 \; \text{and} \; x > 0 \, . 
	\end{equation*}
\end{lemma}

\begin{proof}
	The proof essentially relies on a drift implied bound in \cite{Hajek1982} (copied as Proposition \ref{proposition:appendix|Hajek_drift_bound} in Section \ref{section:appendix|Hajek_drift_bound} for reference).  We check the two conditions of Proposition \ref{proposition:appendix|Hajek_drift_bound} for $\{X_t\}$. 
	
	By (\ref{eq:PTS_for_two_arm_Bern_bandit|two_fixed_particles|CR_pair|Xt_expression}), the drift of the process $\{X_t\}$ at time $t$ is 
	\begin{equation*}
		\begin{aligned}
			&\; \E[X_{t+1} - X_t | X_t = x] \\
			=&\; \frac{e^x}{1+e^x} \theta^*_1 \ln \frac{\theta^{(1)}_1}{\theta^{(2)}_1} + \frac{e^x}{1+e^x} (1-\theta^*_1) \ln \frac{1-\theta^{(1)}_1}{1-\theta^{(2)}_1} + \frac{1}{1+e^x} \theta^*_2 \ln \frac{\theta^{(1)}_2}{\theta^{(2)}_2} + \frac{1}{1+e^x} (1-\theta^*_2) \ln \frac{1-\theta^{(1)}_2}{1-\theta^{(2)}_2} \\
			=& \; \left(\frac{e^x}{1+e^x}d(\theta^*_1 || \theta^{(2)}_1) + \frac{1}{1+e^x} d(\theta^*_2 || \theta^{(2)}_2) \right) - \left(\frac{e^x}{1+e^x}d(\theta^*_1 || \theta^{(1)}_1) + \frac{1}{1+e^x} d(\theta^*_2 || \theta^{(1)}_2) \right) \\
			=& \; D_2 \left(\frac{e^x}{1+e^x}\right) - D_1\left(\frac{e^x}{1+e^x}\right) \triangleq h(x) \, . 
		\end{aligned}
	\end{equation*}
	Let $f(r) \triangleq D_2(r) - D_1(r)$. Then $h(x) = f(\frac{e^x}{1+e^x})$. $f(r)$ is a linear function in $r$: $f(r) = \alpha r + \beta$, where
	\begin{equation}
		\label{eq:PTS_for_two_arm_Bern_bandit|two_fixed_particles|CR_pair|Xt_is_stochastically_bounded_lemma|definition_alpha_beta}
	\begin{aligned}
		\alpha &= \left(d(\theta^*_1 || \theta^{(2)}_1) - d(\theta^*_2 || \theta^{(2)}_2)\right) - \left(d(\theta^*_1 || \theta^{(1)}_1) - d(\theta^*_2 || \theta^{(1)}_2)\right) \, , \; \beta = d(\theta^*_2 || \theta^{(2)}_2) - d(\theta^*_2 || \theta^{(1)}_2) \, . 
	\end{aligned}
	\end{equation}
	Since the two particles form a CR pair, $\alpha < 0$ and $\beta > 0$. Let $\bar{r} = -\frac{\beta}{\alpha}$, which is the solution to $f(r) = 0$. 
	It can be verified that Condition C1 of Proposition \ref{proposition:appendix|Hajek_drift_bound} is satisfied with $a = \ln \frac{1+\bar{r}}{1-\bar{r}}$ and $\epsilon_0 = \frac{1}{2} \left(d(\theta^*_1 || \theta^{(1)}_1) - d(\theta^*_1 || \theta^{(2)}_1)\right)$. This corresponds to solving $\frac{e^a}{1+e^a} = \frac{\bar{r}+1}{2}$, so $h(a) = f(\frac{\bar{r}+1}{2}) = \frac{1}{2} (f(\bar{r}) + f(1)) = \frac{1}{2}f(1) = \epsilon_0$. Note that $a > 0$.
	
	To check Condition C2 of Proposition \ref{proposition:appendix|Hajek_drift_bound}, let $x^* \triangleq \max \left\{\abs{\ln \frac{\theta^{(1)}_1}{\theta^{(2)}_1}}, \abs{\ln \frac{(1-\theta^{(1)}_1)}{(1-\theta^{(2)}_1)}}, \abs{\ln \frac{\theta^{(1)}_2}{\theta^{(2)}_2}}, \abs{\ln \frac{(1-\theta^{(1)}_2)}{(1-\theta^{(2)}_2)}} \right\}$, and let random variable $Z = x^*$ with probability $1$. Then obviously $(\abs{X_{t+1} - X_t} | X_t) \prec Z$. Choose $\lambda = 1$ (any positive value works), then 
	\begin{equation}
		\label{eq:PTS_for_two_arm_Bern_bandit|two_fixed_particles|CR_pair|Xt_is_stochastically_bounded_lemma|constant_D}
	\begin{aligned}
		D &= \E[e^{\lambda Z}] = e^{x^*} = \max \left\{\frac{\theta^{(1)}_1}{\theta^{(2)}_1}, \frac{\theta^{(2)}_1}{\theta^{(1)}_1}, \frac{1-\theta^{(1)}_1}{1-\theta^{(2)}_1}, \frac{1-\theta^{(2)}_1}{1-\theta^{(1)}_1}, \frac{\theta^{(1)}_2}{\theta^{(2)}_2}, \frac{\theta^{(2)}_2}{\theta^{(1)}_2}, \frac{1-\theta^{(1)}_2}{1-\theta^{(2)}_2}, \frac{1-\theta^{(2)}_2}{1-\theta^{(1)}_2}  \right\} \, .
	\end{aligned}
	\end{equation}
	Note that $D > 1$. Condition C2 of Proposition \ref{proposition:appendix|Hajek_drift_bound} is satisfied. 
	
	Since $c \geq \frac{\E[e^{\lambda Z}] - (1+ \E[Z])}{\lambda^2} = D - 1 - x^*$, we can choose the following constants: $c = D$, $\eta = \min\left(1, \frac{\epsilon_0}{2c}\right)$, $\rho = 1- \frac{1}{2} \eta \epsilon_0$. 
	Note that $0 = X_0 \leq a$. Applying Proposition \ref{proposition:appendix|Hajek_drift_bound}, we have
	\begin{equation}
		\label{eq:Xt_is_stochastically_bounded_positive_part}
		P \left\{X_t \geq x \right\} \leq \frac{D}{1-\rho} e^{-\eta(x-a)} = A_1 e^{-B_1 x} \quad  \forall t, x > 0 ,
	\end{equation}
	where $A_1 = \frac{D}{1-\rho} e^{\eta a} = \frac{2D}{\eta \epsilon_0} e^{\eta a} = \frac{2D}{\eta \epsilon_0} \left(\frac{1+\bar{r}}{1-\bar{r}}\right)^\eta$ and $B_1 = \eta$. 
	
	Apply the same analysis to the sequence $\{-X_t\}_{t \geq 0}$ with the following constants: $a' = \ln \frac{2-\bar{r}}{\bar{r}}$, $\epsilon_0' = \frac{1}{2} \left(d(\theta^*_2 || \theta^{(2)}_2) - d(\theta^*_2 || \theta^{(1)}_2) \right)$, $\lambda = 1$, $D$ as in (\ref{eq:PTS_for_two_arm_Bern_bandit|two_fixed_particles|CR_pair|Xt_is_stochastically_bounded_lemma|constant_D}), $c = D$, $\eta' = \min\left(\lambda, \frac{\epsilon_0'}{2c}\right)$ and $\rho' = 1-\frac{1}{2} \eta' \epsilon_0'$, we get
	\begin{equation}
		\label{eq:Xt_is_stochastically_bounded_negative_part}
		P\{-X_t \geq x \} \leq \frac{D}{1-\rho'} e^{-\eta'(x-a')} = A_2 e^{-B_2 x} \; \forall t, x > 0 ,
	\end{equation}
	where $A_2 = \frac{D}{1-\rho'} e^{\eta' a'} = \frac{2D}{\eta' \epsilon_0'} e^{\eta' a'} = \frac{2D}{\eta' \epsilon_0'} \left(\frac{2-\bar{r}}{\bar{r}}\right)^{\eta'}$ and $B_2 = \eta'$. 
	
	Let $A_0 = 2\max\{A_1, A_2\}$ and $B_0 = \min\{B_1, B_2\}$ and combine (\ref{eq:Xt_is_stochastically_bounded_positive_part}) and (\ref{eq:Xt_is_stochastically_bounded_negative_part}), we get 
	\begin{equation*}
		P \left\{\abs{X_t} \geq x \right\} \leq A_0 e^{-B_0 x} \quad \forall t \; \text{and} \; x > 0 \, . 
	\end{equation*}
\end{proof}

We are now ready to prove Proposition \ref{proposition:PTS_for_two_arm_Bern_bandit|two_fixed_particles|CR_pair|almost_sure_convergence}. Roughly speaking, since $\ln \widetilde{w}_{t,i} \approx -t D_i(r_t)$, $X_t = \ln \frac{\widetilde{w}_{t,1}}{\widetilde{w}_{t,2}} \approx t(D_2(r_t) - D_1(r_t))$. The stochastic boundedness of $X_t$ then implies the stochastic boundedness of $t\abs{D_2(r_t) - D_1(r_t)}$. So for large $t$, $D_2(r_t) - D_1(r_t)$ is close to zero and hence $r_t$ is close to $\bar{r}$. We show that $r_t$ converges to $\bar{r}$ in probability, which combined with the Borel-Contelli lemma leads to convergence almost surely. The convergence of $q_t$ and $\bar{w}_t$ naturally follows.

\begin{proof}[Proof of Proposition \ref{proposition:PTS_for_two_arm_Bern_bandit|two_fixed_particles|CR_pair|almost_sure_convergence}]
	
	\sloppy
	Recall that $f(r) = D_2(r) - D_1(r) = \alpha r + \beta$ for $\alpha$ and $\beta$ given in (\ref{eq:PTS_for_two_arm_Bern_bandit|two_fixed_particles|CR_pair|Xt_is_stochastically_bounded_lemma|definition_alpha_beta}) and $f(\bar{r}) = 0$. So $\abs{f(r_t)} = \abs{f(r_t) - f(\bar{r})} = \abs{(\alpha r_t + \beta) - (\alpha \bar{r} + \beta)} = \abs{\alpha} \abs{r_t - \bar{r}}$. Therefore, for any $\delta > 0$,
	\begin{equation*}
		\begin{aligned}
			 P\left\{\abs{r_t - \bar{r}} \geq \delta \right\} &=  P\left\{\abs{f(r_t)} \geq \abs{\alpha} \delta \right\} \\ &\leq P\left\{\abs{f(r_t) + \epsilon_{t,1} - \epsilon_{t,2}}  \geq \frac{\abs{\alpha}\delta}{3} \right\} + P \left\{\abs{\epsilon_{t,1}} \geq \frac{\abs{\alpha}\delta}{3} \right\} + P \left\{\abs{\epsilon_{t,2}} \geq \frac{\abs{\alpha}\delta}{3} \right\} \, . 
		\end{aligned}
	\end{equation*} 

	But 
	\begin{equation*}
		\begin{aligned}
			f(r_t) + \epsilon_{t,1} - \epsilon_{t,2} &= D_2(r_t) - D_1(r_t) + \epsilon_{t,1} - \epsilon_{t,2} \\
			&= \left(-D_1(r_t) + \epsilon_{t,1} + C(r_t)\right) - \left(-D_2(r_t) + \epsilon_{t,2} + C(r_t)\right) \\
			&\overset{(i)}{=} \frac{1}{t} \ln \widetilde{w}_{t,1} - \frac{1}{t} \ln \widetilde{w}_{t,2} \\
			&= \frac{1}{t} \ln \frac{\widetilde{w}_{t,1}}{\widetilde{w}_{t,2}} = \frac{1}{t} X_t \, ,
		\end{aligned}
	\end{equation*} 
	where step $(i)$ is due to Proposition \ref{proposition:PTS_for_two_arm_Bern_bandit|weight_dynamics_of_N_particles|unnormalized_weight_dynamics}. Therefore, by Proposition \ref{proposition:PTS_for_two_arm_Bern_bandit|weight_dynamics_of_N_particles|unnormalized_weight_dynamics} and Lemma \ref{lemma:PTS_for_two_arm_Bern_bandit|two_fixed_particles|CR_pair|Xt_is_stochastically_bounded}, 
	\begin{equation*}
		\begin{aligned}
			P \left\{\abs{r_t - \bar{r}} \geq \delta \right\} &\leq P \left\{\abs{X_t} \geq \frac{\abs{\alpha}\delta t}{3} \right\} + P \left\{\abs{\epsilon_{t,1}} \geq \frac{\abs{\alpha}\delta}{3} \right\} + P \left\{\abs{\epsilon_{t,2}} \geq \frac{\abs{\alpha}\delta}{3} \right\} \\
			&\leq  A_0 e^{-\frac{B_0 \abs{\alpha} \delta t}{3}} + 4te^{-B_{\theta^{(1)}} \frac{\abs{\alpha}^2 \delta^2}{9} t} + 4te^{-B_{\theta^{(2)}} \frac{\abs{\alpha}^2 \delta^2}{9} t} \\
			&\leq  Ate^{-B\delta^2 t} \, ,
		\end{aligned}
	\end{equation*}
	where $A = 3\max\left\{A_0, 4 \right\}$ and $B =\min\left\{\frac{B_0 \abs{\alpha}}{3}, \frac{B_{\theta^{(1)}} \abs{\alpha}^2}{9}, \frac{B_{\theta^{(2)}} \abs{\alpha}^2}{9}\right\}$. It follows that
	\begin{equation*}
	\begin{aligned}
		\sum_{t=1}^\infty P\left\{\abs{r_t - \bar{r}} \geq \delta \right\} &\leq \sum_{t=1}^\infty Ate^{-B \delta^2 t} = Ae^{B\delta^2} \sum_{t=1}^\infty te^{-B\delta^2 (t-1)} = \frac{Ae^{B\delta^2}}{\left(1-e^{-B\delta^2}\right)^2} < \infty \, . 
	\end{aligned}
	\end{equation*}

	By the Borel-Cantelli Lemma, $P \left\{\abs{r_t - \bar{r}} \geq \delta \; i.o. \right\} = 0$ for any $\delta > 0$. It follows that $r_t \rightarrow \bar{r}$ almost surely as $t \rightarrow \infty$. Since arm 1 (resp. arm 2) is chosen iff particle 1 (resp. particle 2) is chosen, $q_t = (r_t, 1-r_t)$. So $q_t \rightarrow (\bar{r}, 1-\bar{r})$. Finally, since $I_t \sim w_{t-1} = (w_{t-1,1}, \cdots, w_{t-1,N})$, $\indicator_{\{I_t = i \}} \sim \Bernoulli(w_{t-1,i})$. For $i = 1,2$, by the Azuma-Hoeffding inequality, for any $\gamma > 0$, 
	\begin{equation*}
		\begin{aligned}
			\Pr\left\{\abs{q_{t,i} - \bar{w}_{t-1,i}} \geq \gamma\right\} &= \Pr \left\{\abs{\frac{1}{t}\sum_{\tau=1}^t \indicator_{\{I_t = i \}} - \frac{1}{t}\sum_{\tau=0}^{t-1} w_{\tau,i}} \geq \gamma \right\} \\
			&= \Pr \left\{\abs{\sum_{\tau=1}^t \left(\indicator_{\{I_t=i\}} - w_{t-1,i}\right)} \geq t\gamma \right\} \\
			&\leq 2\exp\left(-\frac{2(t\gamma)^2}{t}\right) \\
			&= 2e^{-2\gamma^2 t} \,  ,
		\end{aligned}
	\end{equation*}
	which is summable in $t$. Apply the Borel-Cantelli Lemma again, we get $\abs{q_t - \bar{w}_{t-1}} \rightarrow 0$ with probability one. So $\bar{w}_t \rightarrow (\bar{r}, 1-\bar{r})$. 
\end{proof}

\subsubsection{Self-reinforcing pair}

\begin{definition}(Self-reinforcing pair)
	For a given $\BernoulliBandit(K=2, \theta^*)$ problem, we say two particles $\theta^{(1)}, \theta^{(2)} \in [0,1]^2$ form a \emph{self-reinforcing pair} (SR pair) if they can be relabeled such that the following conditions hold:
	\begin{equation}
		\label{eq:PTS_for_two_arm_Bern_bandit|two_fixed_particles|SR_pair|defining_conditions}
	\begin{aligned}
		d(\theta^*_1 || \theta^{(1)}_1) < d(\theta^*_1 || \theta^{(2)}_1), \; d(\theta^*_2 || \theta^{(1)}_2) > d(\theta^*_2 || \theta^{(2)}_2), A(1) = \{1\}, \; A(2) = \{2\} \, . 
	\end{aligned}
	\end{equation}
\end{definition}

Without loss of generality, in this section when we say particles $\theta^{(1)}$ and $\theta^{(2)}$ are a SR pair, we assume they have already been properly labeled such that they satisfy (\ref{eq:PTS_for_two_arm_Bern_bandit|two_fixed_particles|SR_pair|defining_conditions}).

\begin{figure}[h]
    \centering
    \subfloat[\centering Particle positions.]{{\includegraphics[width=0.4\columnwidth]{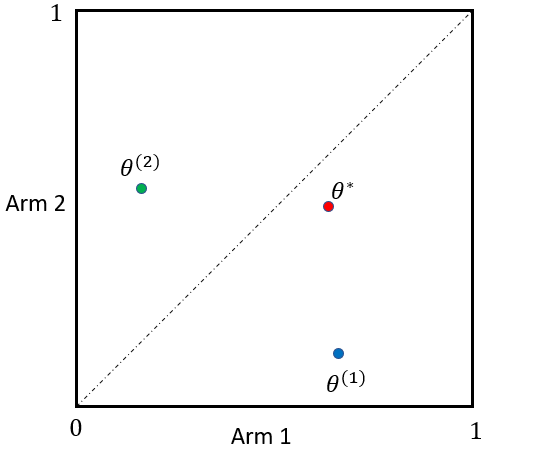}}}%
    \subfloat[\centering Divergences.]{{\includegraphics[width=0.55\columnwidth]{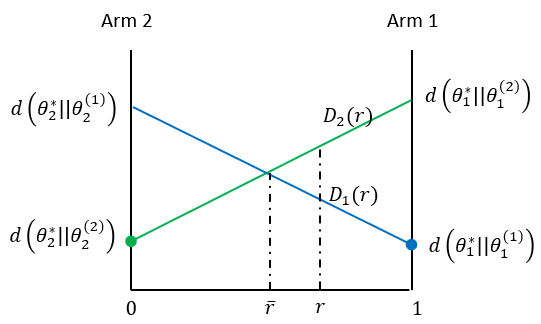}}}%
    \caption{A self-reinforcing pair example.}%
    \label{fig:PTS_for_two_arm_Bern_bandit|two_fixed_particles|SR_pair|example}%
\end{figure}

An SR pair example is drawn in Figure \ref{fig:PTS_for_two_arm_Bern_bandit|two_fixed_particles|SR_pair|example}. Consider a large time $t$. Since $\widetilde{w}_{t,i} \appropto e^{-t D_i(r_t)}$, if $r_t > \bar{r}$, with high probability particle 1 will be selected more often, which will cause $r_t$ to further increase. If $r_t < \bar{r}$, then with high probability particle 2 will be selected often, which will cause $r_t$ to further decrease. Therefore, each of the two particles is \emph{self-reinforcing}: selecting one particle will likely increase the weight of the particle itself which makes it to be selected more often. Each particle behaves like a black hole. We expect that, in the end, either particle 1 or particle 2 gain all the weight. Which of the two particles wins out in the end is random and is influenced by the initial condition. We state this observation more formally in the following proposition.  

\begin{proposition}
	\label{proposition:PTS_for_two_arm_Bern_bandit|two_fixed_particles|SR_pair|two_blackholes_convergence}
	Given a problem $\BernoulliBandit(K=2, \theta^*)$ and a particle set $\calP_2 = \{\theta^{(1)}, \theta^{(2)}\}$, suppose $\{\theta^{(1)}, \theta^{(2)}\}$ forms a SR pair for the problem. Consider the process of running $\PTS(\calP_2)$ as in Algorithm \ref{alg:PTS_for_two_arm_Bernoulli_bandit}. Let $X_t = \ln \frac{\widetilde{w}_{t,1}}{\widetilde{w}_{t,2}} = \ln \frac{w_{t,1}}{w_{t,2}}$ for $t \geq 0$. Then, with probability one, one of the following two cases happens:
	\begin{enumerate}
		\item $X_t \rightarrow \infty$, $q_t \rightarrow (1,0)$, $w_t \rightarrow (1,0)$ and $r_t \rightarrow 1$. 
		\item $X_t \rightarrow -\infty$, $q_t \rightarrow (0,1)$, $w_t \rightarrow (0,1)$ and $r_t \rightarrow 0$. 
	\end{enumerate}
\end{proposition}

The remainder of this section is dedicated to the proof of Proposition \ref{proposition:PTS_for_two_arm_Bern_bandit|two_fixed_particles|SR_pair|two_blackholes_convergence}. We first define the notion of stochastic asymptotic stability, which will be used for the proof.

\begin{definition}
	\label{def:PTS_for_two_arm_Bern_bandit|two_fixed_particles|SR_pair|stochastically_asymptotically_stable}
	Let $\{X_n\}_{n \geq 0}$ be a discrete time Markov process with state space $\sR$. 
	\begin{enumerate}
		\sloppy
		\item We say that $x \in \sR$ is \emph{stochastically asymptically stable} (SAS) for $\{X_n\}$ if for any $\epsilon > 0$, there exits $\delta > 0$ such that if $\abs{X_{n_0} - x} \leq \delta$ for some $n_0$, then $\Pr\left\{\abs{X_n - x} \leq \epsilon \; \forall n\geq n_0 | X_{n_0}\right\} \geq 1-\epsilon$ and $\Pr\left\{ \left\{\abs{X_n - x} \leq \epsilon \; \forall n \geq n_0 \right\} \backslash \left\{X_n \rightarrow x\right\} |X_{n_0} \right\} = 0$.
		\sloppy
		\item We say that $-\infty$ is SAS for $\{X_n\}$ if for any $L \in \sR$ and $\epsilon > 0$, there exists $L_0 \in \sR$ such that if $X_{n_0} \leq L_0$ for some $n_0$, then $\Pr\{X_n \leq L \; \forall n \geq n_0 |X_{n_0}\} \geq 1-\epsilon$ and ${\Pr\left\{ \left\{X_n \leq L \; \forall n \geq n_0\right\} \backslash \left\{X_n \rightarrow -\infty\right\} |X_{n_0} \right\} = 0}$.
		\sloppy
		\item We say that $+\infty$ is SAS for $\{X_n\}$ if for any $L \in \sR$ and $\epsilon > 0$, there exists $L_0 \in \sR$ such that if $X_{n_0} \geq L_0$ for some $n_0$, then $\Pr\{X_n \geq L \; \forall n\geq n_0 |X_{n_0}\} \geq 1-\epsilon$ and $\Pr\left\{\left\{X_n \geq L \; \forall n\geq n_0\right\} \backslash \left\{X_n \rightarrow \infty\right\} |X_{n_0} \right\} = 0$. 
	\end{enumerate}
	
	The second condition in the 1st (resp. 2nd or 3rd) definition above means that, given $X_{n_0}$, if $X_n$ is close to $x$ (resp. $-\infty$, $+\infty$) from $n_0$ onward, then $X_n$ converges to $x$ (resp. $-\infty$, $+\infty$). 
\end{definition}

Intuitively, a SAS point is like a black hole: if the process is close enough to the point, then with high probability it will be trapped around the point and eventually sucked to the point. 

We start the proof of Proposition \ref{proposition:PTS_for_two_arm_Bern_bandit|two_fixed_particles|SR_pair|two_blackholes_convergence} with the following lemma. 

\begin{lemma}
	\label{lemma:PTS_for_two_arm_Bern_bandit|two_fixed_particles|SR_pair|Xt_properties}
	The process $\{X_t\}$ described in Proposition  \ref{proposition:PTS_for_two_arm_Bern_bandit|two_fixed_particles|SR_pair|two_blackholes_convergence} is a Markov process. Moreover, it can be represented as: $X_{t+1} = X_t + U_{t+1}$, where the distribution of $U_{t+1}$ is determined by $X_t$ and it satisfies:
	\begin{enumerate}
		\item[(a)] $\abs{U_t} \leq C$ for all $t \geq 1$,
		\item[(b)] $\E[U_{t+1}|X_t=x] \leq -\mu_1$ whenever $x \leq C_1$,
		\item[(c)] $\E[U_{t+1}|X_t=x] \geq \mu_2$ whenever $x \geq C_2$,
	\end{enumerate}
	for some constants $\mu_1 > 0$, $\mu_2 > 0$, $C$, $C_1$ and $C_2$ that depend on $\theta^*$ and $\calP_2$. 
\end{lemma}

\begin{proof}
	By the recursive update formula for $\widetilde{w}_t$ in (\ref{eq:PTS_for_two_arm_Bern_bandit|posterior_weights_update}) and the conditions $A(1) = \{1\}$ and $A(2) = \{2\}$ in  (\ref{eq:PTS_for_two_arm_Bern_bandit|two_fixed_particles|SR_pair|defining_conditions}), we can obtain the same dynamics of $X_t$ as in (\ref{eq:PTS_for_two_arm_Bern_bandit|two_fixed_particles|CR_pair|Xt_expression}), such that that $X_{t+1} = X_t + U_{t+1}$, where $U_{t+1}$ is the increment of the process $\{X_t\}$ at time $t$, given by
	\begin{equation}
		\label{eq:self_reinforcing_pair_lemma1_step_distribution}
		U_{t+1} =  \left\{\begin{array}{cccc}
			\ln \frac{\theta^{(1)}_1}{\theta^{(2)}_1} & w.p. & \frac{e^{X_t}}{1+e^{X_t}} \theta^*_1 \\
			\ln \frac{(1-\theta^{(1)}_1)}{(1-\theta^{(2)}_1)} & w.p. & \frac{e^{X_t}}{1+e^{X_t}} (1-\theta^*_1) \\
			\ln \frac{\theta^{(1)}_2}{\theta^{(2)}_2} & w.p. & \frac{1}{1+e^{X_t}} \theta^*_2 \\
			\ln \frac{(1-\theta^{(1)}_2)}{(1-\theta^{(2)}_2)} & w.p. & \frac{1}{1+e^{X_t}} (1-\theta^*_2)  \\
		\end{array} \right. 
	\end{equation}
	for $t \geq 0$. Clearly, $\{X_t\}_{t \geq 0}$ is a Markov process and the distribution of $U_{t+1}$ is determined by $X_t$. Property (a) is easily satisfied by setting 		
	\begin{equation*}
    \begin{aligned}
		C \triangleq \max\left\{\abs{\ln \frac{\theta^{(1)}_1}{\theta^{(2)}_1}}, \abs{\ln \frac{(1-\theta^{(1)}_1)}{(1-\theta^{(2)}_1)}}, \abs{\ln \frac{\theta^{(1)}_2}{\theta^{(2)}_2}}, \abs{\ln \frac{(1-\theta^{(1)}_2)}{(1-\theta^{(2)}_2)}} \right\} \, . 
	\end{aligned}
	\end{equation*}
	
	Let $h(x) \triangleq \E[U_{t+1}|X_t=x]$. It can be shown that $h(x) = \alpha \frac{e^x}{1+e^x} + \beta$, where
	\begin{equation*}
	\begin{aligned}
	    \alpha = \left(d(\theta^*_1 || \theta^{(2)}_1) - d(\theta^*_1 || \theta^{(1)}_1)\right) + \left(d(\theta^*_2 || \theta^{(1)}_2) - d(\theta^*_2 || \theta^{(2)}_2)\right) \, ,
	\end{aligned}
	\end{equation*}
	and
	\begin{equation*}
	    \beta = \left(d(\theta^*_2 || \theta^{(2)}_2) - d(\theta^*_2 || \theta^{(1)}_2)\right) \, . 
	\end{equation*}

	By conditions (\ref{eq:PTS_for_two_arm_Bern_bandit|two_fixed_particles|SR_pair|defining_conditions}), $\alpha > 0$ and $\beta < 0$. Let $f(r) = \alpha r + \beta$, $0 \leq r \leq 1$. The graph of $f(r)$ is shown below:

\begin{figure}[h]
\begin{center}
\centerline{\includegraphics[width=0.6\columnwidth]{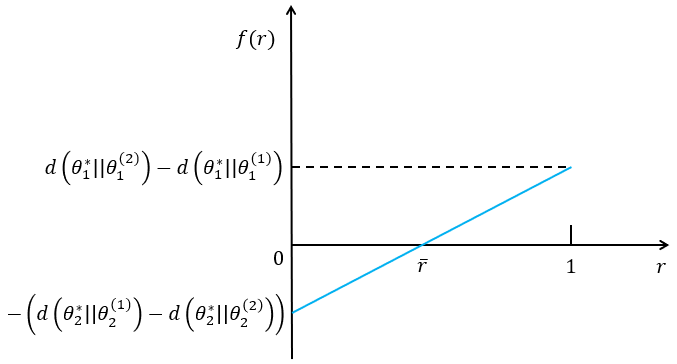}}
\end{center}
\vskip -0.2in
\end{figure}

	At $r = \bar{r} = -\frac{\beta}{\alpha}$, $f(r) = 0$. Let 
	\begin{equation*}
		\mu_1 = \frac{d(\theta^*_2 || \theta^{(1)}_2)-d(\theta^*_2||\theta^{(2)}_2)}{2} \quad \text{and} \quad \mu_2 = \frac{d(\theta^*_1 || \theta^{(2)}_1)-d(\theta^*_1||\theta^{(1)}_1)}{2} \, .
	\end{equation*}
	
	Then $f(r) \leq -\mu_1$ whenever $0 \leq r \leq \frac{\bar{r}}{2}$ and $f(r) \geq \mu_2$ whenever $\frac{\bar{r}+1}{2} \leq r \leq 1$. Let $\frac{e^{C_1}}{1+e^{C_1}} = \frac{\bar{r}}{2}$ and $\frac{e^{C_1}}{1+e^{C_2}} = \frac{\bar{r}+1}{2}$, we get 
	\begin{equation*}
		C_1 = \ln \frac{\bar{r}}{2-\bar{r}} = \ln \frac{-\beta}{2\alpha + \beta} \quad \text{and} \quad 
	    C_2 = \ln \frac{1+\bar{r}}{1-\bar{r}} = \ln \frac{\alpha-\beta}{\alpha+\beta} \, .
	\end{equation*}
	
	Since $h(x) = f(\frac{e^x}{1+e^x})$ and $h(x)$ is monotonely increasing in $x$, we have that $h(x) \leq -\mu_1$ whenever $x \leq C_1$ and $h(x) \geq \mu_2$ whenever $x \geq C_2$.
\end{proof}

\begin{lemma}
	\label{lemma:PTS_for_two_arm_Bern_bandit|two_fixed_particles|SR_pair|Xt_has_two_SAS_points}
	The process $\{X_t\}$ described in Proposition \ref{proposition:PTS_for_two_arm_Bern_bandit|two_fixed_particles|SR_pair|two_blackholes_convergence} has $+\infty$ and $-\infty$ as two SAS points. 
\end{lemma}

\begin{proof}
	First, we show that $-\infty$ is SAS for $\{X_t\}$. Consider any given $L \in \sR$ and $\epsilon > 0$. Without loss of generality, we can assume $L \leq C_1$ and choose $L_0 = L - \frac{C^2}{2\mu_1} \ln \frac{1}{\epsilon}$, where $C_1$ and $C$ are given in Lemma \ref{lemma:PTS_for_two_arm_Bern_bandit|two_fixed_particles|SR_pair|Xt_properties}.\footnote{If $L > C_1$, we can choose $L_0 = C_1 - \frac{C^2}{2\mu_1} \ln \frac{1}{\epsilon}$. Then by the same argument in this proof, we can show that $\Pr\{X_t \leq C_1 \; \forall t | X_0\} \geq 1-\epsilon$, which still implies $\Pr\{X_t \leq L \; \forall t |X_0 \} \geq 1-\epsilon$.} Define
	\begin{equation*}
		T \triangleq \min\left\{t > 0: X_t > L \right\}
	\end{equation*}
	to be the crossing time, the first time the process $\{X_t\}$ crosses above the threshold $L$. By convention, if $\{X_t > L\}$ never happens, $T=\infty$. Define a random sequence $\{\widetilde{X}_t\}_{t \geq 0}$ by $\widetilde{X}_0 = X_0$ and 
	\begin{equation*}
		\widetilde{X}_t = \left\{
		\begin{array}{cccc}
			X_t & if &  1\leq t \leq T \\
			\widetilde{X}_{t-1} - \mu_1 & if & t > T \, . \\
		\end{array}
		\right.
	\end{equation*}
	
	Let $\widetilde{U}_{t+1} = \widetilde{X}_{t+1}-\widetilde{X}_t$, then 
	\begin{equation*}
		\widetilde{U}_t = \left\{\begin{array}{cccc}
			U_t & if & 1 \leq t \leq T \\
			-\mu_1 & if & t > T
		\end{array} \right.
	\end{equation*}
	
	By Lemma \ref{lemma:PTS_for_two_arm_Bern_bandit|two_fixed_particles|SR_pair|Xt_properties} and the above construction, $\E[\widetilde{U}_{t+1}|\widetilde{X}_t] \leq -\mu_1 < 0$ and $\abs{\widetilde{U}_t} \leq C$ for all $t$. It immediately follows from LLN that $\widetilde{X}_t \rightarrow -\infty$ with probability one. Also, if $\widetilde{X}_0 \leq L_0$, then 
	\begin{equation*}
		\begin{aligned}
			\Pr\left\{\widetilde{X}_t \leq L \; \forall t \, \Big| \, \widetilde{X}_0 \right\} &= \Pr\left\{\max_{t \geq 0} \widetilde{X}_t \leq L \, \Big| \, \widetilde{X}_0 \right\}  \\
			&= \Pr \left\{\max_{t\geq 0} (\widetilde{X}_t - L_0) \leq L - L_0 \, \Big| \, \widetilde{X}_0 \right\} \\
			&=  \Pr \left\{\max_{t\geq 0} (\widetilde{X}_t - L_0) \leq \frac{C^2}{2\mu_1} \ln \frac{1}{\epsilon} \, \Big| \, \widetilde{X}_0 \right\} \\
			&\overset{(i)}{\geq} 1 - \exp\left\{-\frac{2\mu_1}{C^2} \frac{C^2}{2\mu_1} \ln \frac{1}{\epsilon} \right\} \\
			&= 1-\epsilon \, ,
		\end{aligned}
	\end{equation*}
	where inequality $(i)$ is due to Proposition \ref{proposition:appendix|bounded_steps_drift_bound} (see Appendix \ref{section:appendix|bounded_steps_drift_bound}).
	
	Note that, $\{X_t \leq L \; \forall t\} = \{\widetilde{X}_t \leq L \, \forall t\}$, and under such event, $\{X_t\}_{t\geq 0} = \{\widetilde{X}_t\}_{t\geq 0}$. It follows that 
	\begin{equation*}
		\Pr\left\{X_t \leq L \, \forall t \, \Big| \, X_0 \right\} = \Pr\left\{\widetilde{X}_t \leq L \; \forall t \, \Big| \, \widetilde{X}_0 \right\} \geq 1-\epsilon 
	\end{equation*} 
	and
	\begin{equation*}
		\begin{aligned}
			& \; \Pr \left\{\left\{X_t \leq L \, \forall t\right\} \backslash \left\{X_t \rightarrow -\infty\right\} \, \Big| \, X_0 \right\} \\
			=& \; \Pr \left\{\left\{X_t \leq L \, \forall t \right\} \cap \left\{X_t \not\rightarrow -\infty \right\} \, \Big| \, X_0 \right\} \\
			=& \; \Pr \left\{\left\{X_t \leq L \, \forall t\right\} \cap \left\{\widetilde{X}_t \not\rightarrow -\infty \right\} \, \Big| \, \widetilde{X}_0 \right\} \\
			\leq& \; \Pr \left\{ \widetilde{X}_t \not\rightarrow -\infty \, \Big| \, \widetilde{X}_0 \right\} = 0 \, . 
		\end{aligned}
	\end{equation*}
	We conclude that $-\infty$ is SAS for $\{X_t\}$.
	
	By a similar argument, using properties (a) and (c) of Lemma \ref{lemma:PTS_for_two_arm_Bern_bandit|two_fixed_particles|SR_pair|Xt_properties} and Corollary \ref{corollary:appendix|bounded_steps_drift_bound} (see Appendix \ref{section:appendix|bounded_steps_drift_bound}), we can show that $+\infty$ is SAS for $\{X_t\}$. 
\end{proof}

We are now ready to prove Proposition \ref{proposition:PTS_for_two_arm_Bern_bandit|two_fixed_particles|SR_pair|two_blackholes_convergence}.

\begin{proof}[Proof of Proposition \ref{proposition:PTS_for_two_arm_Bern_bandit|two_fixed_particles|SR_pair|two_blackholes_convergence}]
	Fix $\epsilon = 0.5$ (any positive $\epsilon$ will do) and some $L_1, R_1 \in \sR$ such that $L_1 \leq C_1 \leq C_2 \leq R_1$. By Lemma \ref{lemma:PTS_for_two_arm_Bern_bandit|two_fixed_particles|SR_pair|Xt_has_two_SAS_points}, there exists $L_2 < L_1$ and $R_2 > R_1$ such that 
	\begin{enumerate}
	    \item[(1)] If $X_{t_0} \leq L_2$ for some $t_0$, then $\Pr\left\{X_t \leq L_1 \; \forall t \geq t_0 \, \Big| \, X_{t_0} \right\} \geq 0.5$ and $X_t \leq L_1$ $\forall t \geq t_0$ implies $X_t \rightarrow -\infty$, and
	    \item[(2)] If $X_{t_0} \geq R_2$ for some $t_0$, then  $\Pr\left\{X_t \geq R_1 \; \forall t \geq t_0 \, \Big| \, X_{t_0} \right\} \geq 0.5$ and $X_t \geq R_1$ $\forall t \geq t_0$ implies $X_t \rightarrow \infty$.  
	\end{enumerate}

	For a better illustration, see the figure below:
	
\begin{figure}[h]
\begin{center}
\centerline{\includegraphics[width=0.7\columnwidth]{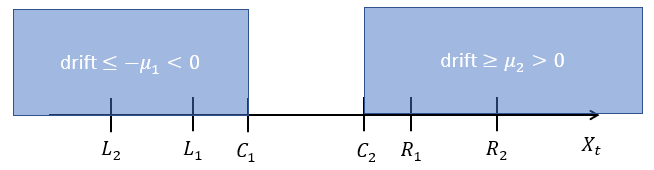}}
\end{center}
\vskip -0.2in
\end{figure}

	Two observations:
	\begin{itemize}
		\item If $X_{t_0}$ ever moves outside of the interval $(L_2, R_2)$ for some $t_0$, then with probability at least $0.5$, $X_t$ stays $\leq L_1$ or $\geq R_1$ for all $t \geq t_0$ and converges to $-\infty$ or $\infty$. 
		\item If $X_{t_0}$ is inside the interval $(L_2, R_2)$ for some $t_0$, then within a fixed $M$ number of the following steps, with a strictly positive probability $\delta$, $X_t$ will move outside of $[L_2, R_2]$. To see this, consider the following. Since the two particles form a SR pair, $\theta^{(1)}_1 \not= \theta^{(2)}_1$. We can assume without loss of generality that $\theta^{(1)}_1 > \theta^{(2)}_1$. By the form of the distribution of the step $U_{t+1}$ in (\ref{eq:self_reinforcing_pair_lemma1_step_distribution}), if $X_t \in (L_2, R_2)$, then within the next $M = \ceil{\frac{R_2 - L_2}{\ln \frac{\theta^{(1)}_1}{\theta^{(2)}_1}}}$ steps, with probability at least $\delta = \left(\frac{e^{L_2}}{1+e^{L_2}}\theta^*_1\right)^{M} > 0$, $X_t$ will become $\geq R_2$.
	\end{itemize}
	
	Consider the following:
	\begin{itemize}
		\item[(a)] Observe the process $\{X_t\}$ from $t = 0$. If $X_t$ always stays below $L_1$ or above $R_1$, then it will converge to $\infty$ or $-\infty$. 
		\item[(b)] If $X_t$ ever moves into the interval $(L_1, R_1)$, it is also in the interval $(L_2, R_2)$, then we start the following trial: \textit{observe whether $X_t$ will become $\leq L_2$ or $\geq R_2$ within the next $M$ steps, and if it does, observe whether it will stay $\leq L_1$ or $\geq R_1$ onward forever}. The trial fails if $X_t$ doesn't become $\leq L_2$ or $\geq R_2$ within the next $M$ steps, or it does, but after that it enters the interval $(L_1, R_1)$ at some time. By the above two observations, this trial is successful with probability at least $0.5 \delta > 0$. The failure of the trial, if it ever happens, can be detected in a finite number of steps. 
		\item[(c)] If the above trial fails, we start the next trial, same as the one in (b), which is also successful with probability at least $0.5 \delta$. Repeat this trial process whenever a trial fails. 
		\item[(d)] Since $0.5\delta > 0$, one trial will eventually be successful with probability one. 
	\end{itemize}
	
	We conclude that $X_t$ converges either to $-\infty$ or $\infty$ with probability one. In either case, the convergences of $q_t$, $w_t$ and $r_t$ are obvious. 
\end{proof}

\subsection{N given particles: asymptotic behavior} \label{subsec:PTS_for_two_arm_Bernoulli_bandit|N_given_particles_asymptotic_behavior}

We now turn to the case of $N$ given particles. The question is: which particles can survive? Let us start with a discussion of a representative example of a four-particle configuration in Figure \ref{fig:PTS_for_two_arm_Bern_bandit|N_fixed_particles|four_particles_example}. We discuss how the weights of the particles change based on our understanding of the case of two particles in the previous section.

\begin{figure}[h]
\begin{center}
\centerline{\includegraphics[width=0.5\columnwidth]{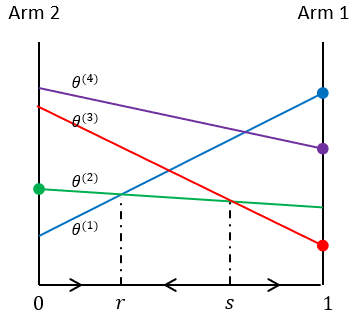}}
\caption{An example of four particles.}
\label{fig:PTS_for_two_arm_Bern_bandit|N_fixed_particles|four_particles_example}
\end{center}
\vskip -0.2in
\end{figure}

In the divergence diagram in Figure \ref{fig:PTS_for_two_arm_Bern_bandit|N_fixed_particles|four_particles_example}, we divide the bottom interval $[0,1]$ into three intervals, $[0,r]$, $[r,s]$ and $[s,1]$, based on the intersections of the line segments of particles 1, 2 and 3 (it will be soon clear why we ignore particle 4). Recall Proposition \ref{proposition:PTS_for_two_arm_Bern_bandit|weight_dynamics_of_N_particles|unnormalized_weight_dynamics} again, we have $\widetilde{w}_{t,i} \appropto e^{-tD_i(r_t)}$. For large $t$, if $r_t \in (0, r)$, particle 1 will tend to dominate, and $r_t$ will drift to the right; if $r_t \in (r,s)$, particle 2 will tend to dominate, and $r_t$ will drift to the left; if $r_t \in (s,1)$, particle 3 will tend to dominate, and $r_t$ will drift to the right. 

\begin{itemize}
	\item If $r_t$ stays around $r$ for a long time, then weights of particles 3 and 4 will eventually become negligible. The system essentially reduces to particles 1 and 2, which form a CR pair. By the discussion and results in Section \ref{subsection:PTS_for_two_arm_Bern_bandit|two_fixed_particles|CR_pair}, we expect that $\ln \frac{w_{t,1}}{w_{t,2}}$ oscillates but is stochastically bounded, $\ln \frac{w_{t,1}}{w_{t,3}} \rightarrow \infty$ and $\ln \frac{w_{t,1}}{w_{t,4}} \rightarrow \infty$. Also, we expect that $q_t \rightarrow (r,1-r,0,0)$, $\bar{w} \rightarrow (r, 1-r, 0,0)$ and $r_t \rightarrow r$.
	\item If $r_t$ stays close to $1$ for a long time, then weights of particles 1, 2 and 4 become negligible and the system essentially reduces to a single particle 3. Thus, when $r_t > s$, particle 3 is self-reinforcing. We expect that $q_t \rightarrow (0,0,1,0)$, $w_t \rightarrow (0,0,1,0)$ and $r_t \rightarrow 1$. 
\end{itemize}

Therefore, we expect that $r_t$ converges to either $r$ or $1$. In either case, we expect only two or one particle will survive in the end. 

We now state the ideas in the above discussion more formally for general $N$ fixed particles. Consider a two-arm Bernoulli bandit problem with parameter $\theta^*$ and a given set of $N$ particles $\calP_N$. Define $D^o(r) \triangleq \min_{i \in \{1, \cdots, N\}} D_i(r)$. Let $D^o$ be an abbreviation of the curve $\{D^o(r): r \in [0,1] \}$ and let $D_i$ be an abbreviation of the line segment $\{D_i(r): r \in [0,1]\}$.  Graphically, $D^o$ is the bottom piece-wise linear curve formed by the line segments of involved particles in the divergence diagram. We make the following assumptions about the particles. 

\begin{assumption}
	\label{assumption:PTS_for_two_arm_Bern_bandit|N_fixed_particles|no_more_than_two_particles_intersecting}
	Assume that $\theta^* \in [0,1]^2$ and $\calP_N \subset [0,1]^2$ satisfy: 
	\begin{enumerate}
		\item There do not exist two different particles $i, j$ such that $D_i = D_j$. 
		\item $\abs{\left\{i : D_i(r) = D^o(r) \right\}} \leq 2$ for all $r \in (0,1)$. 
	\end{enumerate}
	
\end{assumption}

The first assumption above means that each line segment in the divergence diagram represents one unique particle. The second assumption means that no point on the curve $D^o$ is shared by more than two particles, except possibly at the boundaries. Both assumptions hold with probability one if the $N$ particles are generated uniformly at random. For the rest of this section, we assume Assumption \ref{assumption:PTS_for_two_arm_Bern_bandit|N_fixed_particles|no_more_than_two_particles_intersecting} holds.\footnote{Even if Assumption \ref{assumption:PTS_for_two_arm_Bern_bandit|N_fixed_particles|no_more_than_two_particles_intersecting} do not hold, i.e., if two different particles have the same line segment or if more than two particles intersect at some point on $D^o$, we expect that Conjecture \ref{conjecture:PTS_for_two_arm_Bern_bandit|N_fixed_particles|rt_converges_to_a_point_in_script_R_with_prob_one} is still true, perhaps with some minor modifications of the related definitions. But since we don't have any rigorous results for these scenarios, and since those scenarios are not useful in practice, we deem it reasonable to proceed with Assumption \ref{assumption:PTS_for_two_arm_Bern_bandit|N_fixed_particles|no_more_than_two_particles_intersecting}.}

The breakpoints and their associated particles for $D^o$ are defined as follows. 

\begin{definition}
	A point $r \in [0,1]$ is a \emph{breakpoint} for $D^o$ if it is a boundary point (i.e., $0$ or $1$), or it is where two different particles intersect on $D^o$ (i.e., $D^o(r) = D_i(r) = D_j(r)$ for some $i \not= j$). Each breakpoint is associated with a set of one or two particles:
	\begin{itemize}
		\item If $r \in (0,1)$ is a breakpoint where $D^o(r) = D_i(r) = D_j(r)$ for some $i \not= j$, then its associated particles are $\{i,j\}$. 
		\item The breakpoint $0$ has one associated particle $i_0$, which is the particle such that there exists some $\epsilon > 0$ such that $D_{i_0}(\delta) < D_i(\delta)$ for all $i \not= i_0$ for all $\delta \in (0,\epsilon)$. 
		\item The breakpoint $1$ has one associated particle $i_1$, which is the particle such that there exists some $\epsilon > 0$ such that $D_{i_1}(1-\delta) < D_i(1-\delta)$ for all $i \not= i_1$ for all $\delta \in (0,\epsilon)$. 
	\end{itemize}
\end{definition}

\begin{definition}
	Let $\xi \in (0,1)$ be a non-breakpoint for $D^o$. The \emph{dominant particle} at $\xi$ for the process $\{r_t\}$ is a particle $i$ such that $D_i(\xi) = \min_{j \in [N]} D_j(\xi)$, i.e., $D_i(\xi) = D^o(\xi)$. If $\xi$ is contained in $(r,s)$, where $r,s$ are two neighbor breakpoints for $D^o$, we also say $i$ is the dominant particle for interval $(r,s)$ for the process $\{r_t\}$. 
\end{definition}

By Proposition \ref{proposition:PTS_for_two_arm_Bern_bandit|weight_dynamics_of_N_particles|unnormalized_weight_dynamics}, if $r_t$ stays around a non-breakpoint $\xi \in (0,1)$ for a long time, the weight of the corresponding dominant particle tends to increase exponentially. In that sense the particle dominates other particles. 

\begin{example}
	To illustrate the above definitions, see an example of six particles in the divergence diagram in Figure \ref{fig:PTS_for_two_arm_Bern_bandit|N_fixed_particles|six_particles_example}. 
	
	\begin{figure}[h]
\vskip 0.2in
\begin{center}
\centerline{\includegraphics[width=0.5\columnwidth]{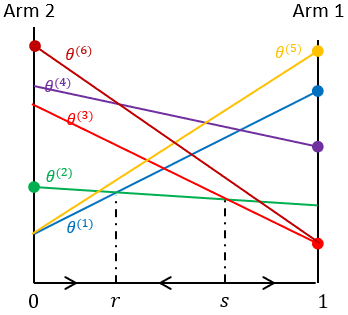}}
\caption{An example of six particles.}
\label{fig:PTS_for_two_arm_Bern_bandit|N_fixed_particles|six_particles_example}
\end{center}
\vskip -0.2in
\end{figure}
	
	In this example, the breakpoints are $\{0, r, s, 1\}$ and their associated particles are $0 \rightarrow \{1\}$, $r \rightarrow \{1,2\}$, $s \rightarrow \{2,3\}$ and $1 \rightarrow \{3\}$, respectively. The dominant particles for intervals $(0,r), (r,s), (s,1)$ are particles $1, 2, 3$, respectively. 
\end{example}

\begin{definition}
	\label{def:PTS_for_two_arm_Bern_bandit|N_fixed_particles|contraction_set_R}
	The contraction set for the $\{r_t\}$ process, denoted by $\calR$, is defined as follows. A value $r \in [0,1]$ is in $\calR$ if one of the following is true:
	\begin{enumerate}
		\item $r = 0$ and $A(i_0) = 2$, where $i_0$ is the associated particle for breakpoint $0$. 
		\item $r = 1$, and $A(i_1) = 1$, where $i_1$ is the associated particle for breakpoint $1$. 
		\item $r \in (0,1)$ is a breakpoint and particles $\{i,j\}$ form a CR pair, where ${i,j}$ are the associated particles for $r$. 
	\end{enumerate}
\end{definition}

For the example in Figure \ref{fig:PTS_for_two_arm_Bern_bandit|N_fixed_particles|six_particles_example}, $\calR = \{r, 1\}$.

\begin{remark}
	Note that once $\theta^*$ and $\calP_N$ are given, $\calR$ is determined, even before PTS runs. 
\end{remark}

\begin{conjecture}
	\label{conjecture:PTS_for_two_arm_Bern_bandit|N_fixed_particles|rt_converges_to_a_point_in_script_R_with_prob_one}
	Consider a given problem $\BernoulliBandit(K=2, \theta^*)$ and a particle set $\calP_N$ that satisfy Assumption \ref{assumption:PTS_for_two_arm_Bern_bandit|N_fixed_particles|no_more_than_two_particles_intersecting}. Consider the process of running $\PTS(\calP_N)$ as in Algorithm \ref{alg:PTS_for_two_arm_Bernoulli_bandit}. Let $\calR$ be the contraction set for the $\{r_t\}$ process. Then $\calR$ is non-empty and with probability one, $r_t \rightarrow r$ for some $r \in \calR$, and the one or two particles associated with the break point $r$ survive, while all other particles' weights converge to zero. 
\end{conjecture}

A proof for this conjecture might begin with analyzing a properly defined $N-1$ dimensional Markov process about the particles' weights (just like for the two-particle case we analyzed a one-dimensional Markov process). We don't have a proof for the conjecture, although its truthfulness is strongly indicated by discussion at the beginning of this section and empirical evidence.

The major take-away lesson of this section is that, with Assumption \ref{assumption:PTS_for_two_arm_Bern_bandit|N_fixed_particles|no_more_than_two_particles_intersecting}, no more than two particles can survive in the asymptotic regime, and the possible surviving particles can be found by drawing the divergence diagram, as discussed. Informally speaking, the line segments of the surviving particles should be low in the divergence diagram.

This is a special case of the sample-path necessary survival condition for general stochastic bandit problems in Section \ref{sec:analysis_of_PTS}.

\subsection{N Random particles} \label{subsec:PTS_for_two_arm_Bernoulli_bandit|N_random_particles}

Up to this point, we have been considering fixed given particles. In practice, particles are not given at the very beginning. One can use a pre-determined set of particles, or randomly generate some particles. In this section, we evaluate the performance of $\PTS$ with $N$ randomly generated particles. We will consider two different methods for particle generation. The following lemma is useful for the analysis of both cases. 

\begin{definition}
	We say that a particle $\theta \in [0,1]^2$ is \emph{action-optimal} for a given problem $\BernoulliBandit(K=2, \theta^*)$ if $A(\theta) = A(\theta^*)$. 
\end{definition}

In particular, if $\theta^*_1 = \theta^*_2$, then any $\theta \in [0,1]^2$ is action-optimal. 

\begin{lemma}
	\label{lemma:PTS_for_two_arm_Bern_bandit|N_random_particles|sufficient_condition_for_consistent_particle}
	Consider a given $\BernoulliBandit(K=2, \theta^*)$ problem and assume $\theta^*_1 \not= \theta^*_2$. There exist $\theta^*$-dependent positive constants $\bar{d}_1$ and $\bar{d}_2$ such that, if a particle $\theta \in [0,1]^2$ satisfies $d(\theta^*_1 || \theta_1) < \bar{d}_1$ and $d(\theta^*_2 || \theta_2) < \bar{d}_2$, then $\theta$ is action-optimal. In particular, $\bar{d}_1 = d\left(\theta^*_1 || \frac{\theta^*_1 + \theta^*_2}{2}\right)$ and $\bar{d}_2 = d\left(\theta^*_2 || \frac{\theta^*_1 + \theta^*_2}{2}\right)$ works.  
\end{lemma}

The lemma provides us with a useful divergence based sufficient condition under which a particle is action-optimal. 

\begin{proof}
	Without loss of generality, assume $\theta^*_1 > \theta^*_2$. It is clear that, if $\theta$ satisfies $\frac{\theta^*_1+\theta^*_2}{2} < \theta_1 \leq 1$ and $0 \leq \theta_2 < \frac{\theta^*_1+\theta^*_2}{2}$, then $A(\theta^*) = A(\theta)$. See the region highlighted by red in Figure \ref{fig:PTS_for_two_arm_Bern_bandit|N_random_particles|sufficient_condition_for_consistent_particle_lemma|consistent_region}. 
	
\begin{figure}[h]
\begin{center}
\centerline{\includegraphics[width=0.6\columnwidth]{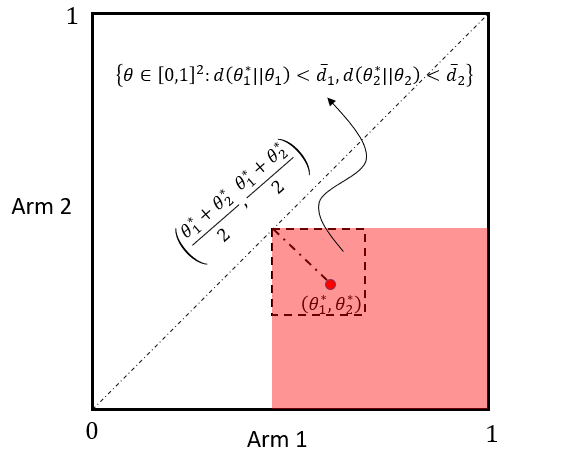}}
\caption{Any $\theta$ in the red region is consistent.}
\label{fig:PTS_for_two_arm_Bern_bandit|N_random_particles|sufficient_condition_for_consistent_particle_lemma|consistent_region}
\end{center}
\vskip -0.2in
\end{figure}

	The function $g(y) = d(x||y)$ for $x \in (0,1)$ is monotone decreasing for $y \in (0,x)$ and monotone increasing for $y \in (x, 1)$. Therefore a sufficient condition for $\frac{\theta^*_1+\theta^*_2}{2} < \theta_1 \leq 1$ is $d(\theta^*_1 || \theta_1) < d\left(\theta^*_1 || \frac{\theta^*_1 + \theta^*_2}{2}\right)$ and a sufficient condition for $0 \leq \theta_2 < \frac{\theta^*_1+\theta^*_2}{2}$ is $d(\theta^*_2 || \theta_2) < d\left(\theta^*_2 || \frac{\theta^*_1 + \theta^*_2}{2}\right)$. Let $\bar{d}_1 = d\left(\theta^*_1 || \frac{\theta^*_1 + \theta^*_2}{2}\right)$ and $\bar{d}_2 = d\left(\theta^*_2 || \frac{\theta^*_1 + \theta^*_2}{2}\right)$, the proof is done. 
\end{proof}

\subsubsection{Coordinate-wise random generation}

\emph{Method 1} (coordinate-wise random generation): Generate two sets $A$ and $B$, each contains $\sqrt{N}$ values generated independently uniformly at random from $[0,1]$. Let $\calP_N = A \times B = \left\{(a,b): a \in A, b \in B\right\}$.

\begin{figure}[h]
    \centering
    \subfloat[\centering Particles positions.]{{\includegraphics[width=0.35\columnwidth]{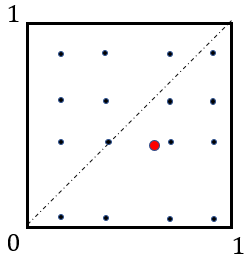}}}%
    \subfloat[\centering Divergence diagram.]{{\includegraphics[width=0.42\columnwidth]{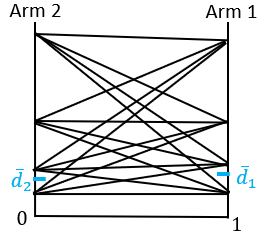}}}%
    \caption{An example of 16 particles produced by coordinate-wise random generation.}%
    \label{fig:PTS_for_two_arm_Bern_bandit|N_random_particles|coordinate_wise_generation|example_16_particles}%
\end{figure}

An example of 16 particles produced by Method 1 is shown in Figure \ref{fig:PTS_for_two_arm_Bern_bandit|N_random_particles|coordinate_wise_generation|example_16_particles}. The particles form a grid in the $[0,1]^2$ square (Fig. \ref{fig:PTS_for_two_arm_Bern_bandit|N_random_particles|coordinate_wise_generation|example_16_particles}). The line segments of the particles form a complete bipartite graph in the divergence diagram (Fig. \ref{fig:PTS_for_two_arm_Bern_bandit|N_random_particles|coordinate_wise_generation|example_16_particles}). By the discussion in Section \ref{subsec:PTS_for_two_arm_Bernoulli_bandit|N_given_particles_asymptotic_behavior}, the weight of the particle represented by the lowest line segment will converge to one with probability one. Call this the bottom particle. For particles generated by Method 1, the bottom particle always exists and is unique. The running average regret of PTS will converge to zero if and only if the bottom particle is action-optimal. If $N$ is large, we expect that with high probability, the KL divergences of the bottom particle at the two arms will be below $\bar{d}_1$ and $\bar{d}_2$ respectively and hence the bottom particle is action-optimal. 

\begin{definition}
	For a given stochastic bandit problem, we say that an algorithm is \textit{consistent} for a given sample path if the running average regret converges to zero. 
\end{definition}

In particular, for a given $\BernoulliBandit(K=2, \theta^*)$ problem, the running average regret is $\frac{1}{T}\sum_{t=1}^T \left( \max_{a \in \{1,2\}}\theta^*_a - \theta^*_{A_t}\right)$. Therefore, PTS is consistent for a given sample path if $w_{t,i} \rightarrow 1$ and $\abs{\frac{1}{T}\sum_{t=1}^T w_{t,i} - \frac{1}{T}\sum_{t=1}^T \indicator_{\{I_t = i\}} } \rightarrow 0$ for some action-optimal particle $i$.

\begin{proposition}
	\label{proposition:PTS_for_two_arm_Bern_bandit|N_random_particles|coordinate_wise_generation|consistency_whp_for_large_N}
	Let $\calP_N$ be a set of $N$ particles generated by Method 1. Consider the process of running $\PTS(\calP_N)$ for a given problem $\BernoulliBandit(K=2, \theta^*)$ as in Algorithm \ref{alg:PTS_for_two_arm_Bernoulli_bandit}. Let $E$ denote the event that the algorithm is consistent. Assume Conjecture \ref{conjecture:PTS_for_two_arm_Bern_bandit|N_fixed_particles|rt_converges_to_a_point_in_script_R_with_prob_one} is true. Then, for $N$ sufficiently large, 
	\begin{equation*}
		\Pr\left\{E\right\} \geq 1 - 2e^{-\frac{\abs{\theta^*_1-\theta^*_2} \sqrt{N}}{2}} \, . 
	\end{equation*}
\end{proposition}

The above result says that with coordinate-wise random particle generation, PTS is consistent with high probability. Observe that, if $\abs{\theta^*_1 - \theta^*_2}$ is large, it is more likely for the algorithm to be consistent, or in other words, it is easier for the algorithm to identify the optimal arm. That makes sense. 

\begin{proof}
	Let $A, B \subset [0,1]$ be the two random sets of $\sqrt{N}$ values generated by Method 1. Let $a_0 = \min_{a \in A} d(\theta^*_1 || a)$ and $b_0 = \min_{b \in B} d(\theta^*_2 || b)$ and let particle $i_0 \in [N]$ be the one with $\theta^{(i_0)} = (\theta^{(i_0)}_1, \theta^{(i_0)}_2) = (a_0, b_0)$. Particle $i_0$ is the bottom particle in our previous discussion. With probability one, $a_0, b_0$ and $i_0$ are unique. By construction, the contraction set $\calR$ of the $\{r_t\}$ process contains only one point, either 0 or 1, depending on the optimal arm for particle $i_0$. By Conjecture \ref{conjecture:PTS_for_two_arm_Bern_bandit|N_fixed_particles|rt_converges_to_a_point_in_script_R_with_prob_one}, the algorithm is consistent if and only if particle $i_0$ is action-optimal. We show that particle $i_0$ is action-optimal w.h.p. 
	
	If $\theta^*_1 = \theta^*_2$, any algorithm is consistent, there is nothing to prove. Without loss of generality, assume $\theta^*_1 > \theta^*_2$. Let $X$ and $Y$ be two independent uniform random variables in $[0,1]$. Let $p_1 \triangleq \Pr\left\{d(\theta^*_1 || X) \leq \bar{d}_1 \right\}$ and $p_2 \triangleq \Pr\left\{d(\theta^*_2 || Y) \leq \bar{d}_2 \right\}$ for $\bar{d}_1 = d\left(\theta^*_1 || \frac{\theta^*_1+\theta^*_2}{2}\right), \bar{d}_2 = d\left(\theta^*_2 || \frac{\theta^*_1+\theta^*_2}{2}\right)$ as in Lemma \ref{lemma:PTS_for_two_arm_Bern_bandit|N_random_particles|sufficient_condition_for_consistent_particle}. Since a sufficient condition for $d(\theta^*_1 || X) \leq \bar{d}_1$ is $X \in \left(\frac{\theta^*_1+\theta^*_2}{2}, \theta^*_1\right)$ and a sufficient condition for $d(\theta^*_2 || Y) \leq \bar{d}_2$ is $Y \in \left(\theta^*_2, \frac{\theta^*_1+\theta^*_2}{2}\right)$, we have  
	\begin{equation*}
		p_1 \geq \Pr \left\{X \in \left(\frac{\theta^*_1 + \theta^*_2}{2}, \theta^*_1 \right) \right\} = \frac{\theta^*_1 - \theta^*_2}{2} 
	\end{equation*}
	and
	\begin{equation*}
		p_2 \geq \Pr \left\{Y \in \left(\theta^*_2, \frac{\theta^*_1+\theta^*_2}{2}\right) \right\} = \frac{\theta^*_1 - \theta^*_2}{2} \, . 
	\end{equation*}
	
	It follows that 
	\begin{equation*}
		\begin{aligned}
			\Pr\{E\} &\geq \Pr \left\{d(\theta^*_1 || \theta^{(i_0)}_1) \leq \bar{d}_1 \; \text{and} \; d(\theta^*_2 || \theta^{(i_0)}_2) \leq \bar{d}_2 \right\} \\
			&= 1 - \Pr \left\{d(\theta^*_1||\theta^{(i_0)}_1) > \bar{d}_1 \; \text{or} \; d(\theta^*_2 || \theta^{(i_0)}_2) > \bar{d}_2 \right\} \\
			&\geq 1 - \Pr \left\{d(\theta^*_1||\theta^{(i_0)}_1) > \bar{d}_1 \right\} - \Pr \left\{ d(\theta^*_2 || \theta^{(i_0)}_2) > \bar{d}_2 \right\}  \\
			&= 1 - \Pr\left\{d(\theta^*_1 || a) > \bar{d}_1 \; \forall a \in A \right\}  - \Pr\left\{d(\theta^*_2 || b) > \bar{d}_2 \; \forall b \in B \right\} \\ 
			&= 1 - (1-p_1)^{\sqrt{N}} - (1-p_2)^{\sqrt{N}} \\
			&\geq 1-2\left(1 - \frac{\theta^*_1 - \theta^*_2}{2}\right)^{\sqrt{N}} \\
			&\geq 1-2e^{-\frac{(\theta^*_1-\theta^*_2)\sqrt{N}}{2}} \, . 
		\end{aligned}
	\end{equation*}
\end{proof}

Despite the nice performance guarantee of PTS for two-arm Bernoulli bandit, coordinate-wise random particle generation has two major limitations. First, for problems in which the parameter space does not have a product topology, it is not clear how particles can be generated coordinate-wise. Second, the method does not scale well for problems with a high dimensional parameter space. For example, for the $K$-arm Bernoulli bandit problem, even if we only generate two values on each coordinate, we have $2^K$ particles, which brings concerns on computational cost. 

\subsubsection{Whole-particle random generation}

\emph{Method 2} (whole-particle random generation): Let $\calP_N$ be a set of $N$ particles generated independently and uniformly at random from $[0,1]^2$. 

Let us discuss the performance of $\PTS(\calP_N)$ on a high-level when $\calP_N$ is generated by Method 2. Suppose $\theta^*$ is given, and so are $\bar{d}_1$ and $\bar{d}_2$ in Lemma \ref{lemma:PTS_for_two_arm_Bern_bandit|N_random_particles|sufficient_condition_for_consistent_particle}. If $N$ is large enough, w.h.p. we expect that the line segment of at least one particle is low and flat enough such that its two ends are below $\bar{d}_1$ and $\bar{d}_2$ respectively, which makes the particle action-optimal. Let us call it particle 1. Without loss of generality, suppose $a(1) = 1$. See Figure \ref{fig:PTS_for_two_arm_Bern_bandit|N_random_particles|whole_particle_generation|example_particles} for an illustration. 

\begin{figure}[h]
    \centering
    \subfloat[\centering Particles positions.]{{\includegraphics[width=0.35\columnwidth]{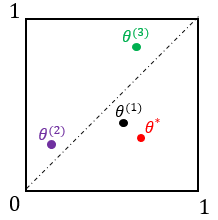}}}%
    \subfloat[\centering Divergence diagram.]{{\includegraphics[width=0.42\columnwidth]{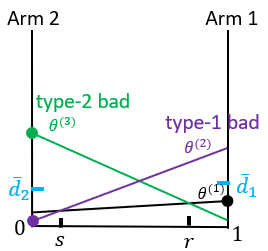}}}%
    \caption{How things could go wrong with whole-particle random generation.}%
    \label{fig:PTS_for_two_arm_Bern_bandit|N_random_particles|whole_particle_generation|example_particles}%
\end{figure}

However, unlike coordinate-wise random generation, here the existence of particle 1 does not guarantee that algorithm is consistent. Things could go wrong in two ways.
\begin{itemize}
	\item There could be a non-action-optimal particle that is close to $\theta^*$ on arm 2, but far from $\theta^*$ on arm 1. Call this the type-1 bad particle, exemplified by particle 2 in Fig \ref{fig:PTS_for_two_arm_Bern_bandit|N_random_particles|whole_particle_generation|example_particles}. Particles 1 and 2 form an SR pair, producing an interval $(0,s)$ in which the process ${r_t}$ would drift to the wrong side. 
	\item There could also be a non-action-optimal particle that is close to $\theta^*$ on arm 1, but far from $\theta^*$ on arm 2. Let us call this the type-2 bad particle, which is exemplified by particle 3 in Fig \ref{fig:PTS_for_two_arm_Bern_bandit|N_random_particles|whole_particle_generation|example_particles}. Particles 1 and 3 form a CR pair. If $r_t$ moves to anywhere in $(s,1)$, it will drift toward $r$ and stay around $1$, not converging to 1. 
\end{itemize}

In other words, for the particle configuration in Fig \ref{fig:PTS_for_two_arm_Bern_bandit|N_random_particles|whole_particle_generation|example_particles}, the process $\{r_t\}$ has contraction set $\calR = \{0, r\}$. Since $\calR$ doesn't contain $1$, PTS cannot be consistent. 

No matter how large $N$ is, the probability that there exist at least one type-1 bad particle and one type-2 bad particle like 2 and 3 in Fig \ref{fig:PTS_for_two_arm_Bern_bandit|N_random_particles|whole_particle_generation|example_particles} is non-zero. However, a bad particle of either type cannot be too flat in the divergence diagram. For example, the right end of the line segment of a type-1 bad particle cannot be below $\bar{d}_1$. Therefore, even with the existence of bad particles, a sufficiently good particle creates an interval in $[0,1]$ (e.g. $(s,t)$ in Fig \ref{fig:PTS_for_two_arm_Bern_bandit|N_random_particles|whole_particle_generation|example_particles}) in which $r_t$ always drifts to the right direction. For large $N$, we expect to have at least one good particle. And as $N$ increases, the line segment of that good particle becomes lower and flatter, making, making the aforementioned interval expand to $(0,1)$. We formally state these ideas as follows.

\begin{proposition}
	\label{proposition:PTS_for_two_arm_Bern_bandit|N_random_particles|whole_particle_generation|PTS_is_PAC}
	Consider a given $\BernoulliBandit(K=2, \theta^*)$ problem and let $\calP_N$ be a random set of $N$ particles generated by Method 2. Let $\calR$ be the contraction set for process $\{r_t\}$ defined in Definition \ref{def:PTS_for_two_arm_Bern_bandit|N_fixed_particles|contraction_set_R}. Then for sufficiently large $N$, with probability at least $1-e^{-N^{1/3}}$, the following statements are true:
	\begin{enumerate}
		\item[(a)] Any $r \in \calR$ satisfies either $r \leq s_0$ or $r \geq r_0$ for some $s_0, r_0 \in [0,1]$ satisfying $s_0 \leq C_1 N^{-\frac{1}{3}}$ and $r_0 \geq 1- C_2 N^{-\frac{1}{3}}$, where $C_1, C_2$ are some $\theta^*$-dependent constants.
		\item[(b)] For any $\xi \in (s_0,r_0)$, the corresponding dominant particle is action-optimal.
	\end{enumerate}
\end{proposition}

An illustration of Proposition \ref{proposition:PTS_for_two_arm_Bern_bandit|N_random_particles|whole_particle_generation|PTS_is_PAC} is shown in Figure \ref{fig:PTS_for_two_arm_Bern_bandit|N_random_particles|whole_particle_generation|PTS_is_PAC_illustration}.

\begin{figure}[h]
\begin{center}
\centerline{\includegraphics[width=0.4\columnwidth]{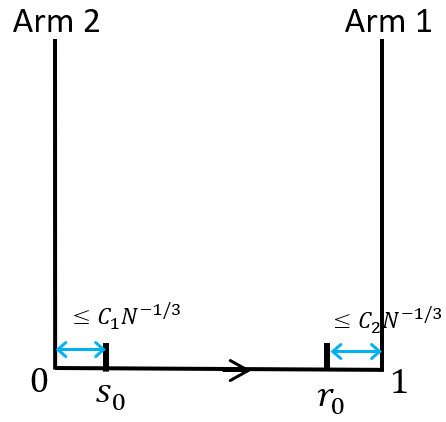}}
\caption{An illustration of Proposition \ref{proposition:PTS_for_two_arm_Bern_bandit|N_random_particles|whole_particle_generation|PTS_is_PAC}.}
\label{fig:PTS_for_two_arm_Bern_bandit|N_random_particles|whole_particle_generation|PTS_is_PAC_illustration}
\end{center}
\vskip -0.2in
\end{figure}

Before we prove this result, let us discuss its implication. Suppose without loss of generality that arm 1 is the optimal arm, i.e., $\theta^*_1 > \theta^*_2$. Let $E_1 \triangleq \left\{\lim_{t \rightarrow \infty} \overline{reg}_t \geq \left(1-\frac{C_1}{\sqrt[3]{N}}\right) \abs{\theta^*_1 - \theta^*_2} \right\}$, a bad event in which the running average regret is large. Let $E_2 \triangleq \left\{\lim_{t \rightarrow \infty} \overline{reg}_t \leq \frac{C_2}{\sqrt[3]{N}} \abs{\theta^*_1 - \theta^*_2} \right\}$, a good event where the running average is small, i.e., the algorithm is almost consistent. According to Proposition \ref{proposition:PTS_for_two_arm_Bern_bandit|N_random_particles|whole_particle_generation|PTS_is_PAC} and Conjecture \ref{conjecture:PTS_for_two_arm_Bern_bandit|N_fixed_particles|rt_converges_to_a_point_in_script_R_with_prob_one}, with high probability $r_t$ eventually converges to some $r \in [0,1]$, with either $r \leq s_0$ or $r \geq r_0$, and the former implies $E_1$ and the latter implies $E_2$. Thus 
\begin{equation}
	\label{eq:PTS_for_two_arm_Bern_bandit|N_random_particles|whole_particle_generation|PTS_is_almost_PAC}
	\Pr \left\{E_1 \cup E_2 \right\} \geq 1 - e^{-\sqrt[3]{N}} \, . 
\end{equation}
Without event $E_1$, (\ref{eq:PTS_for_two_arm_Bern_bandit|N_random_particles|whole_particle_generation|PTS_is_almost_PAC}) means that PTS is probably approximately consistent (PAC). But because we cannot exclude the possibility of $E_1$, we cannot say that PTS is PAC. However, as $N$ increases, the interval $(0, s_0]$ shrinks, we expect that the probability that $r_t$ is trapped somewhere in $[0, s_0]$ becomes smaller. That is, we expect that $\Pr\{E_1\} \rightarrow 0$ as $N\rightarrow \infty$, although we do not have a proof. If that is indeed true, then Proposition \ref{proposition:PTS_for_two_arm_Bern_bandit|N_random_particles|whole_particle_generation|PTS_is_PAC} implies that, with whole-particle random generation, PTS is PAC.

We now prove Proposition \ref{proposition:PTS_for_two_arm_Bern_bandit|N_random_particles|whole_particle_generation|PTS_is_PAC}, starting with the following lemma. 

\begin{lemma}
	\label{lemma:PTS_for_two_arm_Bern_bandit|N_random_particles|whole_particle_generation|PTS_is_PAC_proposition|lemma1}
	Let $\theta^* \in [0,1]^2$ be given. Let $\bar{d}_1$ and $\bar{d}_2$ be the constants in Lemma \ref{lemma:PTS_for_two_arm_Bern_bandit|N_random_particles|sufficient_condition_for_consistent_particle}. In the divergence diagram, let $L_1$ be the line with end points $0$ and $\bar{d}_1$ and let $L_2$ be the line with end points $1$ and $\bar{d}_2$. See Fig. \ref{fig:PTS_for_two_arm_Bern_bandit|N_random_particles|whole_particle_generation|PTS_is_PAC_proposition|lemma1_illustration}. Let $\delta_0$ be the height at which $L_1$ and $L_2$ intersects. For any $\delta \in [0,\delta_0)$, let $L = \{L(r) = \delta: 0 \leq r \leq 1 \}$ be the horizontal line of height $\delta$. Let $s_0$ be such that $L(s_0) = L_1(s_0)$ and let $r_0$ be such that $L(r_0) = L_2(r_0)$. Then $s_0 < r_0$. The following are true:
	\begin{enumerate}
		\item[(a)] If there exists a particle $i$ that satisfies $D_i(r) \leq L(r) = \delta$ for any $r \in (s_0, r_0)$ (i.e., $D_i$ intersects with the red rectangle in Fig. \ref{fig:PTS_for_two_arm_Bern_bandit|N_random_particles|whole_particle_generation|PTS_is_PAC_proposition|lemma1_illustration}), then particle $i$ must be action-optimal. 
		\item[(b)] If there exists a particle $j$ such that $D_j$ is entirely below $L$, then any $r \in \calR$ must satisfy $r \leq s_0$ or $r \geq r_0$. 
	\end{enumerate}
	\begin{figure}[h]
		\centering
		\includegraphics[scale=0.6]{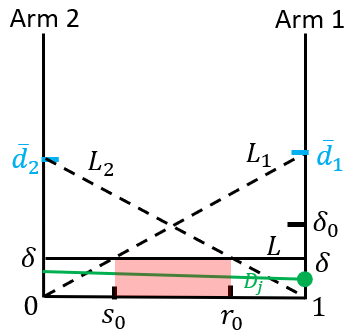}
		\caption{An illustration of Lemma \ref{lemma:PTS_for_two_arm_Bern_bandit|N_random_particles|whole_particle_generation|PTS_is_PAC_proposition|lemma1}.}
		\label{fig:PTS_for_two_arm_Bern_bandit|N_random_particles|whole_particle_generation|PTS_is_PAC_proposition|lemma1_illustration}
	\end{figure}
\end{lemma}

\begin{proof}
	The proof is geometric. See Figure \ref{fig:PTS_for_two_arm_Bern_bandit|N_random_particles|whole_particle_generation|PTS_is_PAC_proposition|lemma1_illustration}. It is obvious that $s_0 < r_0$. 
	
	We show part (a) by showing that its contraposition is true. Consider a particle $i$ associated with a line $D_i$ in the diagram. Suppose particle $i$ is not action-optimal. Then by Lemma \ref{lemma:PTS_for_two_arm_Bern_bandit|N_random_particles|sufficient_condition_for_consistent_particle}, either $D_i(0) \geq \bar{d}_2$ or $D_i(1) \geq \bar{d}_1$. Without loss of generality, assume $D_i(1) \geq \bar{d}_1$. Then $D_i$ must be entirely above $L_1$. Therefore $D_i$ cannot intersect the red rectangle in Fig. \ref{fig:PTS_for_two_arm_Bern_bandit|N_random_particles|whole_particle_generation|PTS_is_PAC_proposition|lemma1_illustration}. 
	
	Next, we show part (b). Suppose particle $j$ has $D_j$ entirely below $L$. Obviously particle $j$ is action-optimal. For any $\xi \in (s_0, r_0)$, its dominant particle must be either particle $j$ itself or below particle $j$ at $\xi$. In the latter case, the dominant particle must be action-optimal according to part (a). Thus, the dominant particle for any $\xi \in (s_0, r_0)$ must be action-optimal. Therefore if $r_t$ is in $(s_0,r_0)$, it always drift to the optimal arm side. $\calR$ does not contain any points in $(s_0, r_0)$.
	
\end{proof}

\begin{lemma}
	\label{lemma:PTS_for_two_arm_Bern_bandit|N_random_particles|whole_particle_generation|PTS_is_PAC_proposition|lemma2}
	Let $U$ be a random variable uniformly distributed in $[0,1]$. Then for any $\epsilon \in (0,1)$, for any value $x \in [0,1]$ fixed and given, $\Pr\left\{d(x||U) \leq \epsilon \right\} \geq \frac{\epsilon}{2}$. 
\end{lemma}

\begin{proof}
	By Theorem 1 in \cite{DragomirScholzSunde2000}, $d(x||u) \leq   \frac{x^2}{u} + \frac{(1-x)^2}{1-u}-1$. Therefore, if $u$ satisfies 
	\begin{equation*}
		u \geq \frac{1}{1+\epsilon}x \quad \text{and} \quad 1-u \geq \frac{1}{1+\epsilon}(1-x) \, , 
	\end{equation*}
	then $d(x||u) \leq (1+\epsilon)x + (1+\epsilon)(1-x) -1 = \epsilon$. It follows that 
	\begin{equation*}
	\begin{aligned}
		\Pr\left\{d(x||U) \leq \epsilon \right\} \geq \Pr\left\{\frac{1}{1+\epsilon}x \leq U \leq 1-\frac{1}{1+\epsilon}(1-x) \right\} = 1-\frac{1-x}{1+\epsilon}-\frac{x}{1+\epsilon} = \frac{\epsilon}{1+\epsilon} \geq \frac{\epsilon}{2} \, . 
	\end{aligned}
	\end{equation*}
\end{proof}

\begin{proof}[Proof of Proposition \ref{proposition:PTS_for_two_arm_Bern_bandit|N_random_particles|whole_particle_generation|PTS_is_PAC}]
	Consider a fixed large $N$. Let $\delta(N) = 2N^{-\frac{1}{3}}$. Without loss of generality, suppose $N$ is large enough such that $\delta(N) < \delta_0$ as in Lemma \ref{lemma:PTS_for_two_arm_Bern_bandit|N_random_particles|whole_particle_generation|PTS_is_PAC_proposition|lemma1}. Let $L(N), s_0(N), r_0(N)$ be defined for $\delta(N)$ as $L, s_0, r_0$ are defined for $\delta$ in Lemma \ref{lemma:PTS_for_two_arm_Bern_bandit|N_random_particles|whole_particle_generation|PTS_is_PAC_proposition|lemma1}. If a particle $i$ satisfies that $D_i$ is entirely below the line $L(N)$, we say that particle $i$ is good. Let $E$ be the event that there exists at least one good particle in $\calP_N$. It follows that 
	\begin{equation*}
		\begin{aligned}
			\Pr\{E\} &= 1 - (1 - \Pr\{\text{particle 1 is good}\})^N \\
			&= 1 - \left(1 - \Pr\left\{d(\theta^*_1 || \theta^{(1)}_1) \leq \delta(N)\right\} \cdot  \Pr\left\{d(\theta^*_2 || \theta^{(1)}_2) \leq \delta(N)\right\} \right)^N \\
			&\overset{(i)}{\geq} 1-\left(1 - N^{-1/3} N^{-1/3} \right)^N \\
			&\geq 1-e^{-N^{-2/3}N} = 1-e^{-N^{\frac{1}{3}}} \, ,
		\end{aligned}
	\end{equation*}
	where $(i)$ is due to $\Pr\left\{d(\theta^*_i||\theta^{(1)}_i) \leq \delta(N)\right\} \geq N^{-\frac{1}{3}}$ by Lemma \ref{lemma:PTS_for_two_arm_Bern_bandit|N_random_particles|whole_particle_generation|PTS_is_PAC_proposition|lemma2} for $i=1,2$.

	Suppose event $E$ is true. Let $i_0$ be one good particle. Then by Lemma \ref{lemma:PTS_for_two_arm_Bern_bandit|N_random_particles|whole_particle_generation|PTS_is_PAC_proposition|lemma1} part (b), any $r \in \calR$ must satisfy $r \leq s_0(N)$ or $r \geq r_0(N)$. Simple geometry shows that $s_0(N) = \frac{\delta(N)}{\bar{d}_1} = \frac{2}{\bar{d}_1} N^{-\frac{1}{3}}$ and $r_0(N) = 1- \frac{\delta(N)}{\bar{d}_2} = 1-\frac{2}{\bar{d}_2} N^{-\frac{1}{3}}$. Let $C_1 = \frac{2}{\bar{d}_1}$ and $C_2 = \frac{2}{\bar{d}_2}$, part (a) of Proposition \ref{proposition:PTS_for_two_arm_Bern_bandit|N_random_particles|whole_particle_generation|PTS_is_PAC} is proved. 
	
	Consider any $\xi \in (s_0, r_0)$, let the corresponding dominant particle be $j$. Then $D_j(\xi) \leq D_{i_0}(\xi)$. By Lemma \ref{lemma:PTS_for_two_arm_Bern_bandit|N_random_particles|whole_particle_generation|PTS_is_PAC_proposition|lemma1} part (a), particle $j$ must be action-optimal. Part (b) of Proposition \ref{proposition:PTS_for_two_arm_Bern_bandit|N_random_particles|whole_particle_generation|PTS_is_PAC} is proved. 
\end{proof}

\subsection{Summary} \label{subsec:PTS_for_two_arm_Bernoulli_bandit|summary}

In this section we analyzed $\PTS$ for the two-arm Bernoulli bandit problem. Our key findings are the following.
\begin{itemize}
	\item \emph{Fit particles survive, unfit particles decay}, in the sense described in Proposition \ref{proposition:PTS_for_two_arm_Bern_bandit|weight_dynamics_of_N_particles|unnormalized_weight_dynamics} and Conjecture \ref{conjecture:PTS_for_two_arm_Bern_bandit|N_fixed_particles|rt_converges_to_a_point_in_script_R_with_prob_one}. The fitness of a particle $i$ is measured in terms of its closeness to $\theta^*$ by the divergence $D_i(r_t)$, a convex combination of the KL divergences on the two arms. Unfortunately we cannot directly compare the fitness of particles because $D_i(r_t)$ depends on the random process $r_t$. It is possible that the weights of the surviving particles oscillates forever due to the counter-reinforcing effect. Also, the weights of the decaying particles decay exponentially fast. 
	\item \emph{The set of surviving particles is random.} This is mainly due to the self-reinforcing effect. One way to find out the possible sets of surviving particles is by drawing the divergence diagram described in Section \ref{subsec:PTS_for_two_arm_Bernoulli_bandit|N_given_particles_asymptotic_behavior}. 
	\item \emph{Most particles decay}. Under Assumption \ref{assumption:PTS_for_two_arm_Bern_bandit|N_fixed_particles|no_more_than_two_particles_intersecting}, we expect that all except at most two particles decay eventually. 
	\item Roughly speaking, with randomly generated particles, \emph{PTS is consistent or near-consistent with high probability}. See Proposition \ref{proposition:PTS_for_two_arm_Bern_bandit|N_random_particles|coordinate_wise_generation|consistency_whp_for_large_N} and Proposition \ref{proposition:PTS_for_two_arm_Bern_bandit|N_random_particles|whole_particle_generation|PTS_is_PAC}. 
\end{itemize}

We believe these findings and some related concepts can be extended to other and more general kinds of stochastic bandit problems. For example, for the $K$-arm Bernoulli bandit problem with $K \geq 3$, we expect to observe counter-reinforcing sets (not just pairs) of particles in PTS, in which the particles reinforce each other in some way. Proposition \ref{proposition:sample_path_necessary_survival_condition} provides a generalized method to identify surviving particles, including counter-reinforcing particles, for general stochastic bandit problems and for any finite number of particles. 

\subsection{Useful Drift Implied Bounds} \label{subsec:PTS_for_two_arm_Bernoulli_bandit|useful_drift_implied_bounds}

This section includes for reference two useful drft implied bounds.

\subsubsection{One drift implied bound with stochastic dominance} \label{section:appendix|Hajek_drift_bound}

The following result (Proposition \ref{proposition:appendix|Hajek_drift_bound}) is taken out from \cite{Hajek1982} for convenience of reference. Let $X_0, X_1, \cdots$ be a sequence of random variables. The drift at time $t$ is defined as $\E[X_{t+1} - X_t | \calF_t]$, where $\calF_t = \sigma(X_0, \cdots, X_t)$. Consider the following two conditions:

\textbf{Condition C1}: 
\begin{equation}
	\label{eq:appendix|Hajek_drift_bound|condition_C1}
	\E\left[(X_{t+1} - X_t) \indicator_{\{X_t \geq a\}} | \calF_t \right] \leq -\epsilon_0 \quad t \geq 0
\end{equation}
for some constants $-\infty \leq a < \infty$ and $\epsilon_0 > 0$. That is, the drift at time $t$ is strictly negative whenever $X_t \geq a$. 

\textbf{Condition C2}: There exists a random variable $Z$ with $\E[e^{\lambda Z}] = D$ for some constants $\lambda > 0$ and $D > 0$ such that $(\abs{X_{t+1} - X_t} | \calF_t) \prec Z$. That is, given $\calF_t$, $\abs{X_{t+1} - X_t}$ is stochastically dominated by a random variable with exponential tail. 

Let $c, \eta, \rho$ be constants such that
\begin{equation*}
	\begin{aligned}
		c &\geq \frac{\E[e^{\lambda Z}] - (1+\lambda \E[Z])}{\lambda^2} \, , \\
		0 &< \eta \leq \lambda \, , \\
		\eta &< \epsilon_0 / c \, , \\
		\rho &= 1 - \epsilon_0 \eta + c \eta^2 \, .
	\end{aligned}
\end{equation*}
Then $\rho < 1$.

\begin{proposition}[Theorem 2.3 in \cite{Hajek1982}]
	\label{proposition:PTS_for_two_arm_Bern_bandit|appendix|Hajek_drift_bound}
	\label{proposition:appendix|Hajek_drift_bound}
	Conditions C1 and C2 imply that
	\begin{equation*}
		P\left\{X_t \geq b | X_0 \right\} \leq \rho^t e^{\eta(Y_0 - b)} + \frac{1-\rho^t}{1-\rho} D e^{-\eta(b-a)} \, . 
	\end{equation*}
	In particular, if $X_0 \leq a$, then 
	\begin{equation*}
		P\left\{X_t \geq b | X_0\right\} \leq \frac{D}{1-\rho} e^{-\eta(b-a)} \, . 
	\end{equation*}
\end{proposition}

\subsubsection{Another drift implied bound with bounded steps} \label{section:appendix|bounded_steps_drift_bound}

Two lemmas are stated first.

\begin{lemma}[Hoeffding's Lemma]
	\label{lemma:appendix|bounded_steps_drift_bound|hoeffding_lemma}
	Suppose $Y$ is a random variable such that $\Pr\left\{Y \in [a,b] \right\} = 1$, then $\E\left[e^{\theta (Y - \E[Y])} \right] \leq \frac{\theta^2 (b-a)^2}{8}$. 
\end{lemma}

\begin{lemma}
	\label{lemma:appendix|bounded_steps_drift_bound|max_markov_type_inequality_for_supermartingale}
	\label{lemma:max-markov-type-inequality-for_supermartingale}
	Suppose $(M_k: k \geq 0)$ is a non-negative supermartingale. Then for any $n \geq 0$ and $\gamma > 0$, $\Pr \left\{\max_{0 \leq k \leq n} M_k > \gamma \right\} \leq \frac{\E[M_0]}{\gamma}$.
\end{lemma}

A proof of Lemma \ref{lemma:appendix|bounded_steps_drift_bound|max_markov_type_inequality_for_supermartingale} can be found in Section 3.4 (Page 69) of \cite{HajekECE567CommunicationNetworkAnalysisNotes}.

\begin{proposition}
	\label{proposition:appendix|bounded_steps_drift_bound}
	Consider a randon sequence $(U_n: n \geq 1)$ and define $\calF = \varnothing$ and $\calF_k = \sigma(U_1, \cdots, U_k)$. Suppose $\E[U_{k+1}|\calF_k] \leq -\mu < 0$ for $k \geq 0$ and $\Pr\left\{\abs{U_k} \leq C\right\} = 1$ for $k \geq 1$ for some constancts $\mu, C > 0$. Let $X_n \triangleq U_1 + \cdots + U_n$ for $n \geq 1$ and $X_0 = 0$. Let $G_n \triangleq \max_{0 \leq k \leq n} X_k$ and $G \triangleq \max_{k \geq 0} X_k$. Then for any $b > 0$, $\Pr\left\{G > b \right\} \leq e^{-\frac{2\mu b}{C^2}}$ \, . 
\end{proposition}

\begin{proof}
	By Hoeffding's lemma (Lemma \ref{lemma:appendix|bounded_steps_drift_bound|hoeffding_lemma}), 
	\begin{equation*}
		\E \left[e^{\theta(U_k - \E[U_k|\calF_{k-1}])} | \calF_{k-1} \right] \leq e^{\frac{\theta^2 (2C)^2}{8}} = e^{\frac{\theta^2 C^2}{2}} \, .
	\end{equation*}
	Therefore, for all $k \geq 1$, 
	\begin{equation*}
		\E\left[e^{\theta U_k} | \calF_{k-1} \right] \leq e^{\theta \E\left[U_k | \calF_{k-1} \right]} e^{\frac{\theta^2 C^2}{2}} \leq e^{-\theta \mu + \frac{\theta^2 C^2}{2}}\, . 
	\end{equation*}
	$-\theta \mu + \theta^2C^2/2$ is quadratic in $\theta$ and is less than or equal to zero for all $\theta \in [0, 2\mu/C^2]$. Let $\theta^* = 2\mu/C^2$. Then $\E\left[e^{\theta^* U_k} | \calF_{k-1} \right] \leq 1$ for all $k \geq 1$. Next, define $M_0 = 1$ and $M_k = e^{\theta^* X_k}$ for $k \geq 1$. $(M_k: k \geq 0)$ is a supermartingale because
	\begin{equation*}
		\begin{aligned}
			\E[M_{k+1} | \calF_k] = \E\left[e^{\theta^* (U_1 + \cdots + U_{k+1})} | \calF_k\right] = e^{\theta^*(U_1 + \cdots + U_k)} \E\left[e^{\theta^* U_{k+1}} | \calF_k\right] = M_k \E\left[e^{\theta^* U_k} | \calF_{k-1} \right] \leq M_k \, . 
		\end{aligned}
	\end{equation*}
	It follows that, for any $n \geq 0$ and $b > 0$, 
	\begin{equation*}
		\begin{aligned}
			\Pr\left\{G_n > b \right\} &= \Pr \left\{\max_{0 \leq k \leq n} X_k > b \right\} = \Pr \left\{\max_{0 \leq k \leq n} e^{\theta^* X_k} > e^{\theta^* b} \right\} = \Pr \left\{\max_{0 \leq k \leq n} M_k > e^{\theta^* b} \right\} \\ &\overset{(i)}{\leq} \frac{\E[M_0]}{e^{\theta^* b}} = e^{-\theta^* b} \, .
		\end{aligned}
	\end{equation*}
	Step $(i)$ is due to Lemma \ref{lemma:appendix|bounded_steps_drift_bound|max_markov_type_inequality_for_supermartingale}. Finally, since $G_n$ is non-decreasing in $n$ and $G_n \rightarrow G$ for each sample path, $\indicator_{\{G_n > b\}}$ is non-negative and is non-decreasing in $n$ and $\indicator_{\{G_n > b\}} \rightarrow \indicator_{\{G > b\}}$ for each sample path. So by the monotone convergence theorem 
	\begin{equation*}
		\begin{aligned}
			\Pr\left\{G >b \right\} = \E\left[\indicator_{\{G > b\}} \right] = \lim_{n \rightarrow \infty} \E\left[\indicator_{\{G_n > b\}} \right] = \lim_{n \rightarrow \infty} \Pr\left\{G_n > b \right\} \leq e^{-\theta^* b} = e^{-\frac{2\mu b}{C^2}} \, .
		\end{aligned}
	\end{equation*}
\end{proof}

\begin{corollary}
	\label{corollary:appendix|bounded_steps_drift_bound}
	Consider a randon sequence $(U_n: n \geq 1)$ and define $\calF = \varnothing$ and $\calF_k = \sigma(U_1, \cdots, U_k)$. Suppose $\E[U_{k+1}|\calF_k] \geq \mu > 0$ for $k \geq 0$ and $\Pr\left\{\abs{U_k} \leq C\right\} = 1$ for $k \geq 1$ for some constancts $\mu, C > 0$. Let $X_n \triangleq U_1 + \cdots + U_n$ for $n \geq 1$ and $X_0 = 0$. Let $G_n \triangleq \min_{0 \leq k \leq n} X_k$ and $G \triangleq \min_{k \geq 0} X_k$. Then for any $b > 0$, $\Pr\{G < -b\} \leq e^{-\frac{2\mu b}{C^2}}$.
\end{corollary}

\begin{proof}
	Apply Proposition \ref{proposition:appendix|bounded_steps_drift_bound} to the sequence $\{-X_n\}$. 
\end{proof}

\section{Regenerative particle Thompson sampling: choice of hyper-parameters and more simulations} \label{appendix:RPTS}

\begin{figure}[h]
\centering
\subcaptionbox{Bernoulli bandit, $K=10$ \\ $\theta^* = [0.05, 0.10, \cdots, 0.50].$}{\includegraphics[width=0.5\textwidth]{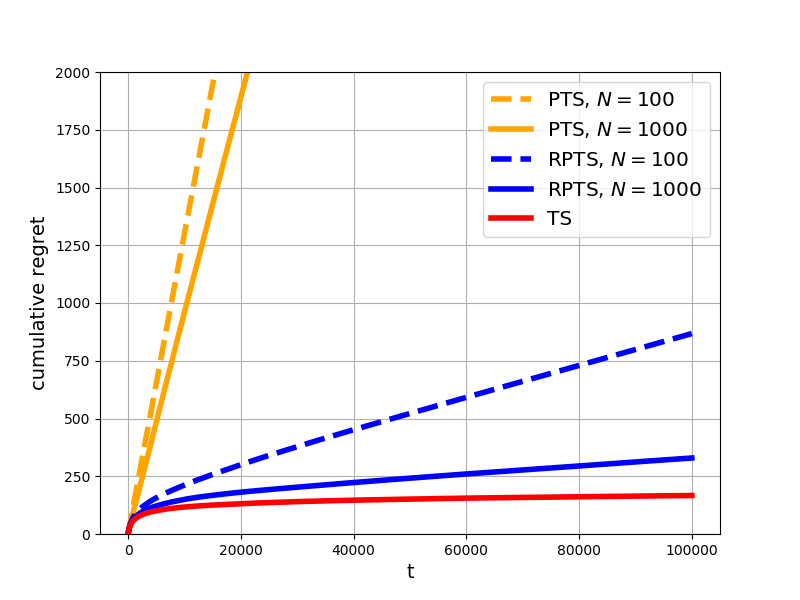}}%
\hfill
\subcaptionbox{Bernoulli bandit, $K=100$ \\ $\theta^*$ consists of $N=100$ points uniformly spaced over [0.3,0.8].}{\includegraphics[width=0.5\textwidth]{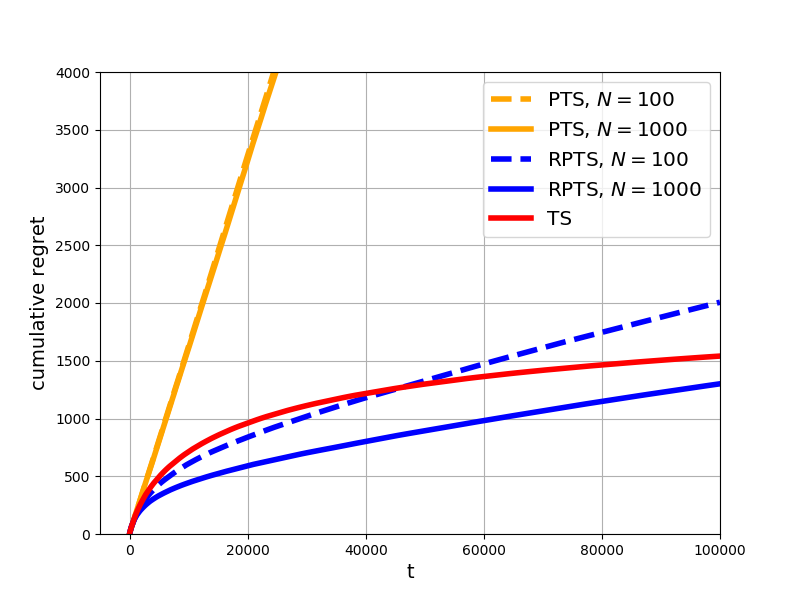}}%
\hfill
\subcaptionbox{Max-Bernoulli bandit, $K=10$, $M=3$ \\ $\theta^* = [0.51, 0.52, \cdots, 0.60].$ }{\includegraphics[width=0.5\textwidth]{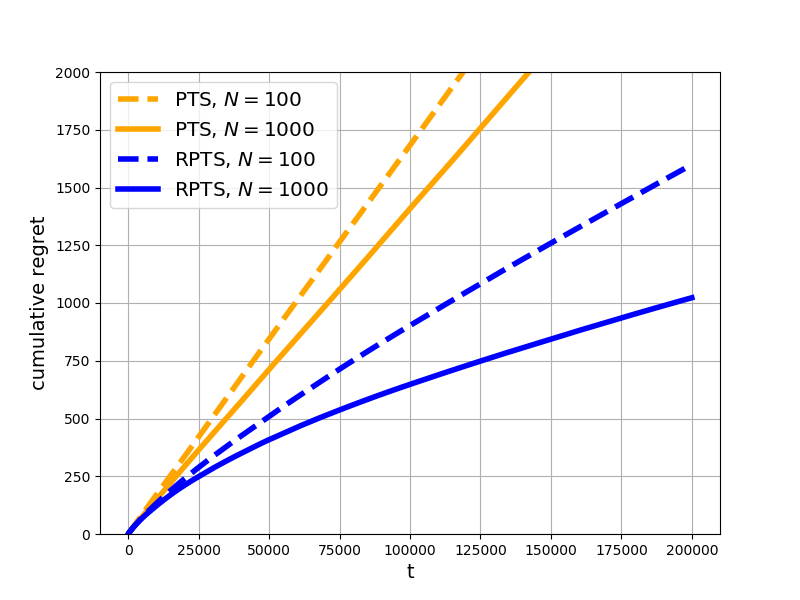}}%
\hfill
\subcaptionbox{Max-Bernoulli bandit, $K=10$, $M=3$ \\ $\theta^* = [0.05, 0.10, \cdots, 0.50].$}{\includegraphics[width=0.5\textwidth]{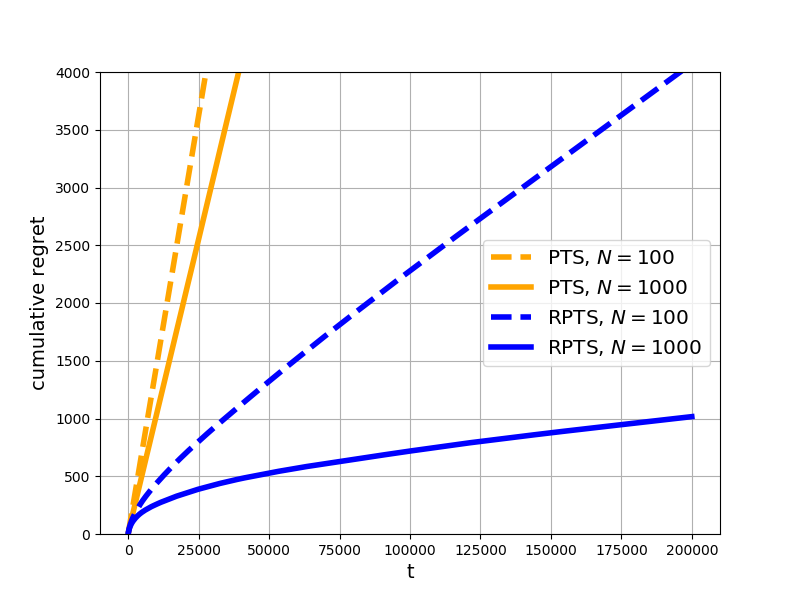}}%
\hfill
\subcaptionbox{Linear bandit, $K=10$, $\sigma_W^2 = 0.1$, $\theta^* = [0.2, \cdots, 0.2]$.}{\includegraphics[width=0.5\textwidth]{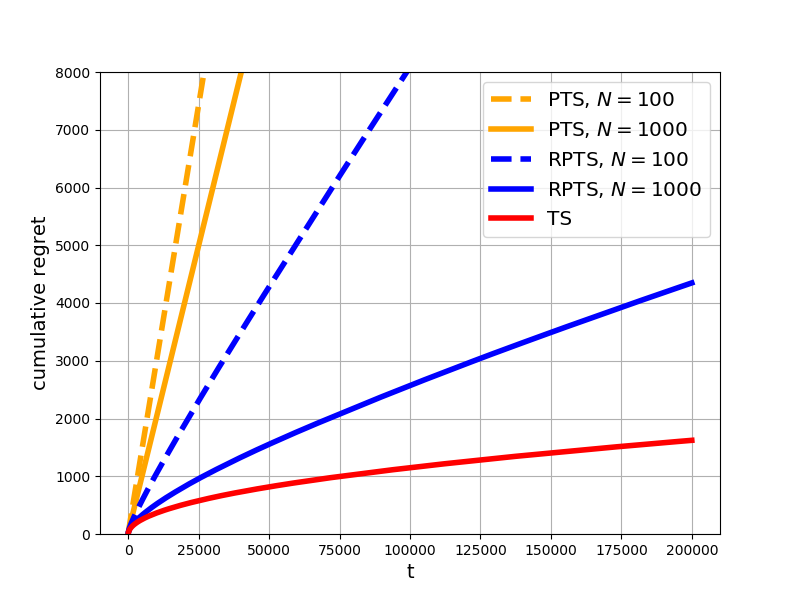}}%
\hfill
\subcaptionbox{Linear bandit, $K=100$, $\sigma_W^2 = 0.1$, $\theta^* = [0.08, \cdots, 0.08]$. }{\includegraphics[width=0.5\textwidth]{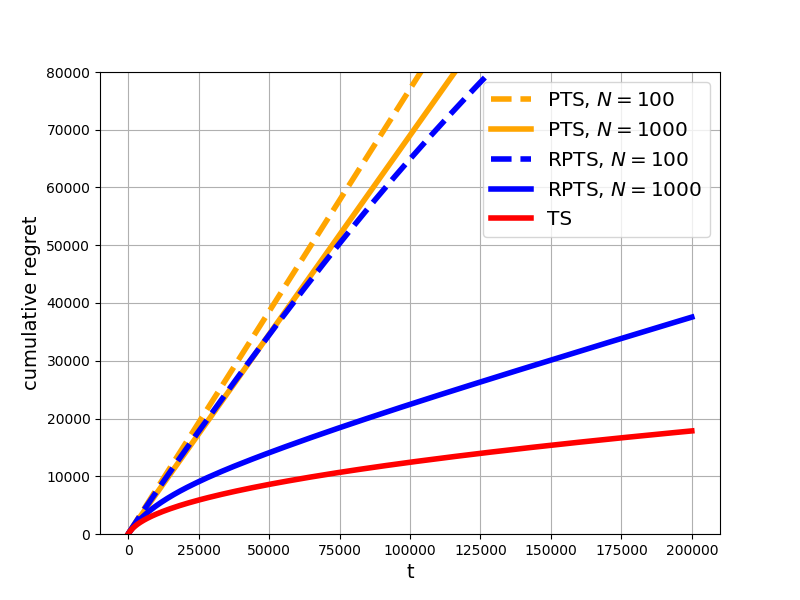}}%
\caption{More simulations.}
\label{fig:RPTS_more_simulation}
\end{figure}

The recommended numerical values of the three hyper-parameters for RPTS (Algorithm \ref{alg:RPTS}) are $f_{del} = 0.8$, $w_{inact} = 0.001$, and $w_{new} = 0.01$. The behavior of the algorithm is relatively insensitive to these values, but further tuning may be beneficial in a given application. In this section we comment on how these values influence the performance of the algorithm.

\begin{itemize}
	\item Analysis for Bernoulli bandits (Section \ref{appendix:PTS_for_two_arm_Bernoulli_bandit}) and empirical evidence for other bandit models indicate that with high probability all but a few particles eventually decay in PTS. Hence it may be attempting to make $f_{del}$ very large. However, since the set of decaying particles is random, it may happen that some fit particles end up decaying. Also, a not-so-bad particle may have an oscillating weight due to counter-reinforcing effects and thus may have low weight at times. Making $f_{del}$ not too large gives those unfortunate fit and not-so-bad particles a chance to survive. We have tried $f_{del} = 0.8$ and $f_{del} = 0.5$ and both work fine. 
	\item The value of $w_{inact}$ should be small, but if it is too small, it may take a long time for the CONDITION in Step 9 to become true, especially when the particles become concentrated in a small subset of the parameter space. 
	\item The value of $w_{new}$ should be small, but strictly larger than $w_{inact}$. There are three aspects of consideration here. First, it is desirable that the weight re-balancing in Step 13 due to normalization has minimal effect on the weights of the surviving particles. We discovered through experiments that it is good for heavy weight particles to remain heavily weighted. Therefore $w_{new}$ should be small. Second, $w_{new}$ should be larger than $w_{inact}$, because otherwise, the newly generated particles in a step will be immediately deleted in the next step. Third, the purpose of setting the value of $w_{new}$ is to give some initial weights to the new particles so that they can participate in the weight updating in the subsequent steps. If a new particle is fit, its weight will boost up exponentially fast; if a new particle is unfit, it will decay exponentially fast. Therefore, the initial weights assigned to these new particles should not significantly affect their chance of survival and their long-term weight dynamics. Thus, as long as $w_{new}$ is fairly small and larger than $w_{inact}$, the choice of its actual value may not make much difference qualitatively. 
\end{itemize}

More simulations are shown in Figure \ref{fig:RPTS_more_simulation}.

For the linear bandit problem, TS can also be exactly implemented by a Kalman filter. The initial set of particles of PTS and RPTS for linear bandits are generated uniformly at random from the unit ball in $\sR^K$. That is based on the assumption that we already know that $\theta^*$ is in the unit ball before running the algorithm. In practice, such knowledge may not be available and a common practice is to use a distribution that spreads out wide enough so that it should cover $\theta^*$. For the purpose of demonstrating the performance of PTS and RPTS here, our practice should be acceptable.

\section{Approximation of expected reward for the network slicing model} \label{appendix:application}

In Section \ref{sec:application}, in step 4 of Algorithm \ref{alg:PTS_for_contextual_stochastic_bandit}, the expected reward $\E_{\theta_t}[R(Y) | A_t = a, c_t]$ becomes $\E_{\theta_t}[g_{c_{t,2}}(Y_t) | a]$ for the network slicing model, where $Y_t = Y_{t,1} + Y_{t,2} + Y_{t,3}$. Since $Y_{t,1}, Y_{t,2}, Y_{t,3}$ are coupled through the non-linear function $g_d$, it is not clear if the expectation can be exactly calculated by a closed-form expression. We propose the following approximation. Given a random variable $Y = Y_1 + Y_2 + Y_3$, where $Y_i$ is an exponentially distributed random variable with mean $\mu_i$ and $Y_i$'s are independent.  Suppose we approximate $Y$ by a Gaussian random variable $\widetilde{Y}$ with mean $\mu = \mu_1 + \mu_2 + \mu_3$ and variance $\sigma^2 = \mu_1^2 + \mu_2^2 + \mu_3^2$. Then 
\begin{equation*}
	\begin{aligned}
		\E[g_d(Y)] &\approx \E[g_d(\widetilde{Y})] \\
		&= \int_0^d \! \frac{y}{d} \frac{1}{\sqrt{2 \pi \sigma^2}} e^{-\frac{(y - \mu)^2}{2\sigma^2}} \, \dd y \\ 
		&= \int_{-\mu}^{d-\mu} \! \frac{1}{d} (z+\mu) \frac{1}{\sqrt{2\pi \sigma^2}} e^{-\frac{z^2}{2\sigma^2}} \, \dd z \quad \left(\text{with} \; z = y-\mu\right) \\
		&= \frac{1}{d \sqrt{2\pi \sigma^2}} \int_{-\mu}^{d-\mu} \! z e^{-\frac{z^2}{2\sigma^2}} \, \dd z + \frac{\mu}{d} \int_{-\mu}^{d-\mu} \! \frac{1}{\sqrt{2\pi \sigma^2}} e^{-\frac{z^2}{2\sigma^2}} \, \dd z \\
		&= \frac{\sigma}{d \sqrt{2\pi}} \left(e^{-\frac{\mu^2}{2\sigma^2}} - e^{-\frac{(d-\mu)^2}{2\sigma^2}}\right) + \frac{\mu}{d} \left( \Phi\left(\frac{d-\mu}{\sigma}\right) - \Phi\left(-\frac{\mu}{\sigma} \right) \right) \, ,
	\end{aligned}
\end{equation*}
where $\Phi(x) \triangleq \sP(N \leq x)$ for a standard Gaussian random variable $N$. Then 
\begin{equation}
	\label{eq:application|model2_expected_reward_approximation2}
	\begin{aligned}
	\E_{\theta_t} \left[g_{c_{t,1}}(Y_t)|a \right] \approx \frac{\sigma_t}{c_{t,2} \sqrt{2\pi}} \left(e^{-\frac{\mu_t^2}{2\sigma_t^2}} - e^{-\frac{(c_{t,2}-\mu_t)^2}{2\sigma_t^2}}\right) + \frac{\mu_t}{c_{t,2}} \left( \Phi\left(\frac{c_{t,2}-\mu_t}{\sigma_t}\right) - \Phi\left(-\frac{\mu_t}{\sigma_t} \right) \right) \, , 
	\end{aligned}
\end{equation}
where $\mu_t= \mu_{t,1} + \mu_{t,2} + \mu_{t,3}$ and $\sigma^2_t = \mu_{t,1}^2 + \mu_{t,2}^2 + \mu_{t,3}^2$ and $\mu_{t,i} = c_{t,1} \theta_{t, i, a_i, 1} + \theta_{t, i, a_i, 2}$ for $i = 1,2,3$. Step 4 of Algorithm \ref{alg:PTS_for_contextual_stochastic_bandit} can then be approximately solved by looping over all possible $a \in [B_1] \times [B_2] \times [B_3]$ and find the one that maximizes (\ref{eq:application|model2_expected_reward_approximation2}).

\end{document}